\documentclass{article}
\pdfoutput=1
\usepackage{nips15submit_e}

\usepackage[page,header]{appendix}

\usepackage{times}
\usepackage[numbers]{natbib}
\usepackage{amsmath}
\usepackage{amssymb}
\usepackage{amsthm}

\usepackage{thmtools}
\newtheorem{theorem}{Theorem}

\usepackage{hyperref}

\usepackage{bm}
\usepackage{relsize}
\usepackage{graphicx}
\usepackage{subfigure}
\usepackage{wrapfig}
\usepackage{booktabs}
\usepackage{multirow}
\usepackage{multicol}
\usepackage{algorithm}
\usepackage{algorithmic}

\usepackage[skip=1pt]{caption}
\usepackage{ragged2e}
\usepackage{rotating}
\usepackage{colortbl}
\usepackage{colortbl}
\definecolor{mColor1}{rgb}{0.95,0.95,0.95}

\newcommand{\beginsupplement}{%
        \renewcommand{\thetable}{S\arabic{table}}%
        \renewcommand{\thefigure}{S\arabic{figure}}%
        \renewcommand{\thesection}{S\arabic{section}}%
         \renewcommand{\thealgorithm}{S\arabic{algorithm}}%
         
     }

\newcommand\Ba{\bm{a}}
\newcommand\Bb{\bm{b}}

\newcommand\Bd{\bm{d}}
\newcommand\Be{\bm{e}}

\newcommand\Bh{\bm{h}}

\newcommand\Bu{\bm{u}}
\newcommand\Bv{\bm{v}}

\newcommand\BA{\bm{A}}
\newcommand\BB{\bm{B}}
\newcommand\BC{\bm{C}}

\newcommand\BE{\bm{E}}

\newcommand\BH{\bm{H}}
\newcommand\BI{\bm{I}}

\newcommand\BM{\bm{M}}

\newcommand\BS{\bm{S}}

\newcommand\BU{\bm{U}}

\newcommand\BW{\bm{W}}

\newcommand\Bla{\bm{\lambda}}
\newcommand\Bep{\bm{\epsilon}}

\newcommand\Bmu{\bm{\mu}}

\newcommand\Bth{\bm{\theta}}
\newcommand\Bxi{\bm{\xi}}

\newcommand\BPs{\bm{\Psi}}
\newcommand\BSi{\bm{\Sigma}}

\newcommand\BTh{\bm{\Theta}}
\newcommand\BOn{\bm{1}}
\newcommand\BZe{\bm{0}}

\newcommand\bbR{\mathbb{R}}

\newcommand\Fcal{\mathcal{F}}
\newcommand\Ecal{\mathcal{E}}
\newcommand\Ncal{\mathcal{N}}

\newcommand\Qcal{\mathcal{Q}}
\newcommand\diag{\mathop{\mathrm{diag}\,}}

\newcommand\EXP{\mathbf{\mathrm{E}}}

\newcommand\PP{\mathbf{\mathrm{P}}}

\newcommand\COV{\mathbf{\mathrm{Cov}}}

\newcommand\TR{\mathbf{\mathrm{Tr}}}

\newcommand\oo{\mathrm{old}}
\newcommand\nn{\mathrm{new}}

\title{Rectified Factor Networks}

\author{
Djork-Arn\'e Clevert,  Andreas Mayr, Thomas Unterthiner and Sepp Hochreiter\\
Institute of Bioinformatics, 
Johannes Kepler University, Linz, Austria\\
\texttt{\{okko,mayr,unterthiner,hochreit\}@bioinf.jku.at}}

\nipsfinalcopy 

\begin{document}

\maketitle

\begin{abstract}
We propose rectified factor networks (RFNs) to 
efficiently construct 
very sparse, non-linear, high-dimensional representations
of the input.
RFN models identify rare and small events in the input,
have a low interference between code units,
have a small reconstruction error, and
explain the data covariance structure.
RFN learning is a generalized alternating minimization
algorithm derived from the posterior regularization method 
which enforces
non-negative and normalized posterior means.
We proof convergence and correctness 
of the RFN learning algorithm.

On benchmarks, RFNs are compared to
other unsupervised methods like
autoencoders, RBMs, factor analysis,
ICA, and PCA.
In contrast to previous sparse coding methods,
RFNs yield sparser codes,
capture the data's covariance structure more precisely, 
and have a significantly smaller reconstruction error.
We test RFNs as pretraining technique for deep networks
on different vision datasets, where RFNs were superior to
RBMs and autoencoders. 
On gene expression data from two
pharmaceutical drug discovery studies, RFNs 
detected small and rare gene modules that
revealed highly relevant new biological insights which 
were so far missed by other unsupervised methods.

\end{abstract}

\section{Introduction}

The success of deep learning
is to a large part
based on advanced and efficient input representations
\cite{Hinton:06,Bengio:07_short, Schmidhuber:15, LeCun:15}.
These representations are sparse and hierarchical.
Sparse representations of the input are in general obtained by
rectified linear units
(ReLU) 
\cite{Nair:10_short,Glorot:11_short}
and dropout 
\cite{Srivastava:14}. 
The key advantage of sparse representations is that dependencies
between coding units are easy to model
and to interpret.
Most importantly, distinct concepts are much less likely to
interfere in sparse representations.
Using sparse representations, similarities of samples often break down to
co-occurrences of features in these samples.
In bioinformatics sparse codes excelled in biclustering of gene expression data
\cite{Hochreiter:10s} and in finding DNA sharing patterns between
humans and Neanderthals \cite{Hochreiter:13}.

Representations learned by ReLUs
are not only sparse but also {\em non-negative}.
Non-negative representations do not code
the degree of absence of events or objects in the input.
As the vast majority of events is
supposed to be absent, to code for their degree of absence would
introduce a high level of random fluctuations.
We also aim for {\em non-linear} input representations to stack models
for constructing {\em hierarchical representations}.
Finally, the representations
are supposed to have a {\em large number of coding units}
to allow coding of rare and small events in the input.
Rare events are only observed in few samples like
seldom side effects in drug design,
rare genotypes in genetics, or small customer groups in e-commerce.
Small events affect only few input components like
pathways with few genes in biology,
few relevant mutations in oncology,
or a pattern of few products in e-commerce.
In summary, our goal is to construct input representations that
(1) are sparse, (2) are non-negative, (3) are non-linear, (4) use many
code units, and (5) model structures in the input data
(see next paragraph).

Current unsupervised 
deep learning approaches like autoencoders or restricted
Boltzmann machines (RBMs) do not model specific structures in the data.
On the other hand, generative models explain structures in the data
but their codes cannot be enforced to be
sparse and non-negative.
The input representation of a generative model is
its posterior's mean, median, or mode, 
which depends on the data. 
Therefore sparseness and non-negativity cannot
be guaranteed independent of the data.
For example, generative models with rectified priors,
like rectified factor analysis,
have zero posterior probability for negative values, therefore their
means are positive and not sparse \cite{Frey:99,Harva:07}. 
Sparse priors do not guarantee sparse posteriors as seen
in the experiments with factor analysis with Laplacian and Jeffrey's
prior on the factors (see Tab.~\ref{tab:compare}).
To address the data dependence of the code, we employ the
{\em posterior regularization method} \cite{Ganchev:10}.
This method separates model characteristics from data dependent
characteristics that are enforced by constraints on the model's posterior.

We aim at representations that are feasible for many code units and
massive datasets,
therefore the computational complexity of generating a code 
is essential in our approach.
For non-Gaussian priors, the computation of the
posterior mean of a new input requires either to numerically
solve an integral or to iteratively update variational parameters
\cite{Palmer:06_short}.
In contrast, for Gaussian priors the posterior mean is the product between
the input and a matrix that is independent of the input.
Still the posterior regularization method
leads to a quadratic (in the number of coding units) constrained
optimization problem in each E-step (see Eq.~\eqref{eq:normP} below).
To speed up computation,
we do not solve the quadratic problem
but perform a gradient step.
To allow for stochastic gradients and
fast GPU implementations, also the M-step is a gradient step.
These E-step and M-step modifications of the
posterior regularization method result in a
{\em generalized alternating minimization} (GAM) 
algorithm \cite{Ganchev:10}.
We will show that the GAM algorithm used for RFN learning (i)
converges and (ii) is correct. 
Correctness means that the
RFN codes are non-negative, sparse, have a low reconstruction
error, and explain the covariance structure of the data.

\section{Rectified Factor Network}
\label{sec:RFN}

Our goal is to construct representations of the input that
(1) are sparse, (2) are non-negative, (3) are non-linear, (4) use many
code units, and (5) model structures in the input.
Structures in the input are identified by a generative model, where the
model assumptions determine which input structures to explain by the model.
We want to model the covariance structure of the input, therefore we
choose maximum likelihood factor analysis as model.
The constraints on the input representation are enforced by the
{\em posterior regularization method} \cite{Ganchev:10}.
{\em Non-negative constraints} lead to sparse
and non-linear codes, while
{\em normalization constraints} scale
the signal part of each hidden (code) unit.
Normalizing constraints avoid that generative models
explain away {\em rare and small signals} by noise.
Explaining away becomes a serious problem for models with many coding
units since their capacities are not utilized.
Normalizing ensures that all hidden units are used but
at the cost of coding also random and spurious signals.
Spurious and true signals must be separated in a subsequent step
either by supervised techniques, by evaluating coding units via
additional data, or by domain experts.

A generative model with hidden units $\Bh$ and data $\Bv$
is defined by its prior $p(\Bh)$ and its likelihood
$p(\Bv \mid \Bh)$.
The full model distribution $p(\Bh , \Bv)= p(\Bv \mid \Bh) p(\Bh)$
can be expressed by the model's posterior
$p(\Bh \mid \Bv)$ and 
its evidence (marginal likelihood) 
$p(\Bv)$: $p(\Bh , \Bv)= p(\Bh \mid \Bv) p(\Bv)$.
The representation of input $\Bv$ is the posterior's mean, median, or mode.
The posterior regularization method introduces a
{\em variational distribution} $Q(\Bh \mid \Bv) \in \Qcal$
from a family $\Qcal$,
which approximates the posterior $p(\Bh \mid \Bv)$.
We choose $\Qcal$ to constrain the posterior means to be non-negative
and normalized.
The full model distribution
$p(\Bh , \Bv)$ contains all model assumptions and, thereby, defines
which structures of the data are modeled.
$Q(\Bh \mid \Bv)$ contains data dependent constraints on the
posterior, therefore on the code.

For data $\{\Bv\}  =  \{\Bv_1,\ldots,\Bv_n\}$,
the posterior regularization method maximizes the
objective $\Fcal$ \cite{Ganchev:10}:
\begin{align}
\label{eq:objective1a}
\Fcal \ &= \ \frac{1}{n} \ \sum_{i=1}^{n}\log p(\Bv_i) \ - \
\frac{1}{n} \ \sum_{i=1}^{n}
D_{\mathrm{KL}}(Q(\Bh_i \mid \Bv_i) \parallel p(\Bh_i \mid \Bv_i) )\\\nonumber
&= \ \frac{1}{n} \ \sum_{i=1}^{n}
\int Q(\Bh_i \mid \Bv_i) \ \log p(\Bv_i \mid \Bh_i) \ d\Bh_i 
- \  \frac{1}{n} \
\sum_{i=1}^{n} D_{\mathrm{KL}}(Q(\Bh_i \mid \Bv_i) \parallel p(\Bh_i) )  \ ,
\end{align}
where $D_{\mathrm{KL}}$ is the Kullback-Leibler distance.
Maximizing $\Fcal$ achieves two goals simultaneously:
(1) extracting desired structures and information from the data
as imposed by the generative model and
(2) ensuring desired code
properties via $Q \in \Qcal$.

\begin{wrapfigure}{r}{0.35\textwidth}
\vspace*{-5pt}
\begin{center}
\includegraphics[width=0.35\textwidth]{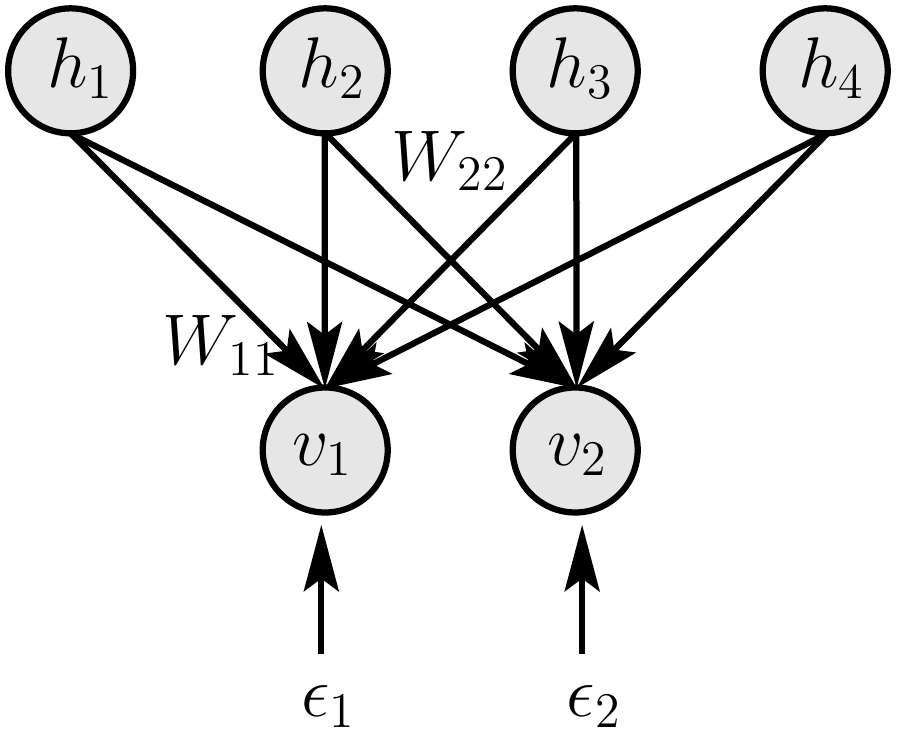}
\end{center}
\caption{Factor analysis model: hidden units (factors) $\Bh$,
  visible units $\Bv$, weight matrix
  $\BW$, noise $\Bep$. \label{fig:rfa}}
 \vspace*{-25pt}
\end{wrapfigure}{

The factor analysis model 
$\Bv =  \BW \Bh  + \Bep$
extracts the {\em covariance structure} of the
data.
The prior $\Bh \sim \Ncal\left(\BZe,\BI\right)$
of the hidden units (factors) $\Bh \in \bbR^l$
and the noise 
$\Bep \sim \Ncal\left(\BZe,\BPs \right)$ 
of visible units (observations) $\Bv \in \bbR^m$
are independent.
The model parameters are the weight (loading) 
matrix $\BW \in \bbR^{m \times l}$ and
the noise covariance matrix $\BPs \in \bbR^{m \times m}$.
We assume
diagonal $\BPs$ to explain correlations between
input components by the hidden units and not by correlated noise.
The factor analysis model is depicted in Fig.~\ref{fig:rfa}.
Given the mean-centered data $\{\Bv\}  =  \{\Bv_1,\ldots,\Bv_n\}$, the
posterior $p(\Bh_i \mid \Bv_i)$ is Gaussian with
mean vector $(\Bmu_p)_i$ and
covariance matrix $\BSi_p$:
\begin{align}
\nonumber
(\Bmu_p)_i \ &= \  \left(\BI \
  + \ \BW^T \BPs^{-1}\BW \right)^{-1}
 \BW^T \BPs^{-1} \ \Bv_i \ , \\ \label{eq:posterior}
\BSi_p \ &=  \ \left(  \BI \ + \ \BW^T
\BPs^{-1}\BW \right)^{-1} \ .
\end{align}

A {\em rectified factor network} (RFN) consists of 
a single or stacked factor analysis
model(s) with constraints on the posterior.
To incorporate the posterior constraints into the factor analysis
model, we use the posterior regularization method that maximizes the
objective $\Fcal$ given in Eq.~\eqref{eq:objective1a} \cite{Ganchev:10}.
Like the expectation-maximization (EM) algorithm, the
posterior regularization method alternates between an E-step and an
M-step.
Minimizing the first $D_{\mathrm{KL}}$ of Eq.~\eqref{eq:objective1a}
with respect to $Q$ leads to a constrained optimization problem.
For Gaussian distributions, the solution with $(\Bmu_p)_i$ and
$\BSi_p$ from Eq.~\eqref{eq:posterior} is
$Q(\Bh_i \mid \Bv_i) \sim \Ncal\left(\Bmu_i,\BSi \right)$ with
$\BSi=\BSi_p$ 
and the quadratic problem:
\begin{align}
\label{eq:normP}
\min_{\Bmu_i}  {\mbox{\ ~} } \frac{1}{ n}  \sum_{i=1}^{n}
(\Bmu_i  -   (\Bmu_p)_i)^T \ \BSi_p^{-1} \ (\Bmu_i  -  (\Bmu_p)_i)\ , \quad 
\mbox{ s.t. }  {\mbox{\ ~} } \forall_i:  \Bmu_i \ \geq \ \BZe \ , \ 
\forall_j:  \frac{1}{n}   \sum_{i=1}^{n} \mu_{ij}^2 \ = \ 1 \ ,
\end{align}
where ``$\geq$'' is component-wise.
This is a constraint non-convex quadratic optimization problem
in the number of hidden units which 
is too complex to be solved in each EM iteration.
Therefore, we perform a step of the
{\em gradient projection algorithm} \cite{Bertsekas:76,Kelley:99}, which
performs first a gradient step and then projects the result to
the feasible set.
We start by a step of the {\em projected Newton method}, then we try the
{\em gradient projection algorithm}, thereafter
the {\em scaled gradient projection algorithm} with reduced matrix \cite{Bertsekas:82}
(see also \cite{Kelley:99}).
If these methods fail to decrease the objective in Eq.~\eqref{eq:normP}, we use
the {\em generalized reduced method} \cite{Abadie:69}.
It solves each equality constraint
for one variable and inserts it into
the objective while ensuring convex constraints.
Alternatively, we use Rosen's gradient projection
method \cite{Rosen:61} or its improvement \cite{Haug:79}.
These methods guarantee a decrease of the E-step objective.

Since the projection $\PP$
by Eq.~\eqref{eq:Proj2} 
is very fast, the projected Newton and projected gradient 
update is very fast, too.
A projected Newton step requires $O(nl)$ steps
(see Eq.~\eqref{eq:NewtonUpdate} and $\PP$ defined in
Theorem~\ref{th:rectNorm}), a projected
gradient step requires $O(\min\{nlm,nl^2\})$ steps, and a scaled gradient
projection step requires $O(nl^3)$ steps.
The RFN complexity per iteration is $O(n(m^2+l^2))$ (see Alg.~\ref{alg:RFN}).
In contrast, a quadratic program solver typically requires for the
$(nl)$ variables (the means of the hidden units for all samples)
$O(n^4l^4)$ steps to find the minimum \cite{Ben-Tal:01}.
We exemplify these values on our benchmark datasets
MNIST ($n=50$k, $l=1024, m=784$)
and CIFAR ($n=50$k, $l=2048, m=1024$).
The speedup with projected Newton or projected gradient in contrast
to a quadratic solver is $O(n^3l^2)=O(n^4l^4)/O(nl^2)$, which gives
{\bf speedup ratios of $\bm{1.3 \cdot 10^{20}}$ for MNIST
and $\bm{5.2 \cdot 10^{20}}$ for CIFAR.}
These speedup ratios show that efficient
E-step updates are essential for RFN learning.
Furthermore, on our computers, RAM restrictions limited quadratic program
solvers to problems with $nl\leq 20$k.

The M-step decreases the {\em expected reconstruction error} 
\begin{align}
\label{eq:recError}
\Ecal \ &= \  - \  \frac{1}{n} \ \sum_{i=1}^{n} \int_{\bbR^l} Q(\Bh_i \mid \Bv_i) \
\log \left( p(\Bv_i \mid \Bh_i) \right) \  d\Bh_i
\\\nonumber
&= \ \frac{1}{2} \Big(
m \ \log \left( 2 \pi \right) \  + \  \log
 \left| \BPs \right| \ + \   \TR \left(\BPs^{-1} \BC \right) 
- \
2 \  \TR  \left( \BPs^{-1} \BW \BU^T \right)
 \ + \ \TR  \left(\BW^T \BPs^{-1} \BW \BS \right)\Big) \ .
\end{align}
from Eq.~\eqref{eq:objective1a} with respect to the model parameters $\BW$ and $\BPs$.
Definitions 
of $\BC$, $\BU$ and $\BS$ are given in Alg.~\ref{alg:RFN}.
The M-step performs a gradient step in the Newton direction,
since we want to allow stochastic gradients,
fast GPU implementation, and dropout
regularization.
The Newton step is derived in the supplementary which gives further
details, too.
Also in the E-step, RFN learning performs a gradient step using 
projected Newton or gradient projection methods.
These projection methods require the Euclidean projection $\PP$
of the posterior means
$\{(\Bmu_p)_i\}$ onto the {\em non-convex} feasible set:
\begin{align}
\label{eq:optmu2}
\min_{\Bmu_i}  {\mbox{\ ~} } &\frac{1}{n} \ \sum_{i=1}^{n} \left( \Bmu_i  \ - \
  (\Bmu_p)_i \right)^T \left( \Bmu_i  \ - \
  (\Bmu_p)_i \right) \  , \ \qquad
\mbox{ s.t. }  {\mbox{\ ~} } \Bmu_i \ \geq \ \BZe  \ , \
\frac{1}{n} \ \sum_{i=1}^{n} \mu_{ij}^2 \ = \ 1 \ .
\end{align}
The
following Theorem~\ref{th:rectNorm} gives the Euclidean projection $\PP$
as solution to Eq.~\eqref{eq:optmu2}.
\begin{theorem}[Euclidean Projection]
\label{th:rectNorm}
If at least one $(\mu_p)_{ij}$ is positive for $1 \leq j \leq l$,
then the solution to optimization problem  Eq.~\eqref{eq:optmu2} is
\begin{align}
\label{eq:Proj2}
\mu_{ij} \ &= \ \left[\PP ((\Bmu_p)_i) \right]_j \ = \
\frac{\hat{\mu}_{ij}}{\sqrt{\frac{1}{n} \ \sum_{i=1}^{n} \hat{\mu}_{ij}^2}} \ \ , \ \quad 
\hat{\mu}_{ij} \ = \  \left\{
\begin{array}{lcl}
0 & \mathrm{for} & (\mu_p)_{ij} \ \leq \ 0 \\
(\mu_p)_{ij} & \mathrm{for} &  (\mu_p)_{ij} \ > \ 0
\end{array} \right.\ .
\end{align}
If all $(\mu_p)_{ij}$ are non-positive for $1 \leq j \leq l$,
then the optimization problem Eq.~\eqref{eq:optmu2} has the solution
$ \ \mu_{ij} \ = \ \sqrt{n} \ $ for $ \ j \ = \  \arg \max_{\hat{j}} \{(\mu_p)_{i\hat{j}}\} \ $ and
$ \ \mu_{ij} \ = \ 0 \ $ otherwise.
\end{theorem}
\begin{proof}
See supplementary material.
\end{proof}
Using the projection $\PP$ defined in
Eq.~\eqref{eq:Proj2},
the E-step updates for the posterior means $\Bmu_i$ are:
\begin{align} 
\label{eq:NewtonUpdate}
\Bmu_i^\nn \ = \ \PP \left( \Bmu_i^\oo \ + \ \gamma \left(\Bd \ - \
  \Bmu_i^\oo \right)\right) \ , \ \quad
\Bd \ = \ \PP \left(
\Bmu_i^\oo \ + \ \lambda \  \BH^{-1} \ \BSi_p^{-1} ( (\Bmu_p)_i \ - \  \Bmu_i^\oo)
\right) 
\end{align}
where we set for the projected Newton method 
$\BH^{-1}=\BSi_p$ (thus $\BH^{-1}  \BSi_p^{-1}=\BI$), and
for the projected gradient method $\BH^{-1}=\BI$.
For the scaled gradient projection algorithm with reduced matrix,
the $\epsilon$-active set for $i$ consists of all $j$ with $\mu_{ij}\leq \epsilon$.
The reduced matrix $\BH$ is the Hessian
$\BSi_p^{-1}$ with $\epsilon$-active columns and rows $j$ fixed to unit
vectors $\Be_j$. The resulting algorithm is a posterior regularization method
with a gradient based E- and M-step, leading to  
a {\em generalized alternating minimization}
(GAM) algorithm \cite{Gunawardana:05}. The RFN learning algorithm is
given in Alg.~\ref{alg:RFN}. Dropout regularization can be included 
before E-step2 by randomly setting code units $\mu_{ij}$ 
to zero with a predefined dropout rate (note that convergence results will no
longer hold).

\begin{algorithm}[t]
{\fontsize{9}{9}\selectfont
\caption{Rectified Factor Network.} \label{alg:RFN}
\vspace*{-.4cm}
\begin{multicols}{2}
\begin{algorithmic}[1]
\STATE $\BC=\frac{1}{n} \sum_{i=1}^n \Bv_i \Bv_i^T$
\vspace*{0.03cm}
\WHILE{STOP=false}
\STATE \textbf{------E-step1------}
\FORALL{$1 \leq i \leq n$}
\STATE
$(\Bmu_p)_i = \left(\BI +\BW^T \BPs^{-1}\BW \right)^{-1}
 \BW^T \BPs^{-1}  \Bv_i$
\ENDFOR
\STATE
$\BSi \ = \ \BSi_p \ = \  \left(  \BI \ + \ \BW^T
\BPs^{-1}\BW \right)^{-1}$
\STATE \textbf{------Constraint Posterior------}
\STATE
(1) projected Newton,
(2) projected gradient,
(3) scaled gradient projection,
(4) generalized reduced method,
(5) Rosen's gradient project.
\STATE \textbf{------E-step2------}
\vspace*{0.03cm}
\STATE
$\BU \ = \ \frac{1}{n} \ \sum_{i=1}^n \Bv_i  \  \Bmu_i^T $
\vspace*{0.03cm}
\STATE
$\BS \ = \  \frac{1}{n} \ \sum_{i=1}^n \Bmu_i \ \Bmu_i^T \ + \ \BSi$
\vspace*{0.03cm}
\STATE \textbf{------M-step------}
\STATE  $\BE \ = \
 \BC \ - \  \BU \ \BW^T  \  - \ \BW \  \BU
\ + \   \BW \ \BS \ \BW^T$
\STATE  $\BW \ = \ \BW \ + \ \eta \ \left(
 \BU \ \BS^{-1} \ - \  \BW \right)$
\FORALL{$1 \leq k \leq m$}
\STATE
$\Psi_{kk} \ = \ \Psi_{kk} \ + \
\eta \  \left( E_{kk} \ - \ \Psi_{kk} \right)$
\ENDFOR
\STATE if stopping criterion is met: STOP=true
\ENDWHILE
\end{algorithmic}
\end{multicols}
\vspace*{-.4cm}
{\bf Complexity:}
objective $\Fcal$: $O(\min\{nlm,nl^2\}+l^3)$;
E-step1: $O(\min\{m^2(m+l),l^2(m+l)\} + nlm)$;
projected Newton: $O(nl)$;
projected gradient: $O(\min\{nlm,nl^2\})$;
scaled gradient projection: $O(nl^3)$;
E-step2: $O(nl(m+l))$;
M-step: $O(ml(m+l))$;
overall complexity with projected Newton / gradient for $(l+m) < n$:
$O(n(m^2+l^2))$.
}
\end{algorithm}

\section{Convergence and Correctness of RFN Learning}

\paragraph{Convergence of RFN Learning.}
\label{sec:convergence}
Theorem~\ref{th:conv} states that
Alg.~\ref{alg:RFN} converges to a maximum of $\Fcal$.
\begin{theorem}[RFN Convergence]
\label{th:conv}
The rectified factor
network (RFN) learning algorithm given in Alg.~\ref{alg:RFN}
is a ``generalized alternating minimization'' (GAM) algorithm
and converges to a solution that maximizes the objective $\Fcal$.
\end{theorem}
\begin{proof}
We present a sketch of the proof which is given in detail in the
supplement. 
For convergence, we show that Alg.~\ref{alg:RFN}
is a GAM algorithm which convergences according to
Proposition~5 in \cite{Gunawardana:05}.

Alg.~\ref{alg:RFN} ensures to decrease the M-step objective which
is convex in $\BW$ and $\BPs^{-1}$. 
The update with $\eta = 1$ leads to the minimum
of the objective. Convexity of the objective
guarantees a decrease in the M-step for $0 < \eta \leq 1$ if not in
a minimum.
Alg.~\ref{alg:RFN} ensures to decrease the E-step objective by
using gradient projection methods.
All other requirements for GAM convergence
are also fulfilled.
\end{proof}
Proposition~5 in \cite{Gunawardana:05} is based on Zangwill's
generalized convergence theorem, thus updates of the RFN algorithm
are viewed as point-to-set mappings \cite{Zangwill:69}.
Therefore the numerical precision, the choice of the methods in the E-step, 
and GPU implementations are covered by the proof.

\paragraph{Correctness of RFN Learning.}
\label{sec:correctnes}
The goal of the RFN algorithm is
to explain the data and its covariance structure.
The {\em expected approximation error} $\BE$ is defined in line 14 of Alg.~\ref{alg:RFN}.
Theorem~\ref{th:fixedPointDiagnonal} states that the
RFN algorithm is correct, that is, it explains the data (low
reconstruction error) and captures the covariance structure as good
as possible.

\begin{theorem}[RFN Correctness]
\label{th:fixedPointDiagnonal}
The fixed point $\BW$ of Alg.~\ref{alg:RFN} minimizes
$\TR \left( \BPs \right)$ given $\Bmu_i$ and
$\BSi$ by ridge regression with
\begin{align}
 \TR \left( \BPs
\right) \ &= \ \frac{1}{n} \ \sum_{i=1}^n
\left\| \Bep_i \right\|_2^2 \ + \
\left\| \BW \ \BSi^{1/2}
\right\|_{\mathrm{F}}^2 \ ,
\end{align}
where
$\Bep_i  =  \Bv_i  -  \BW  \Bmu_i$.
The model explains the data covariance matrix by
\begin{align}
\label{eq:covApprox}
\BC \ = \ \BPs   \ + \   \BW \ \BS  \ \BW^T
\end{align}
up to an error, which is quadratic in $\BPs$
for  $\BPs \ll \BW \BW^T$.
The reconstruction error
$\frac{1}{n}  \sum_{i=1}^n \left\| \Bep_i \right\|_2^2$
is quadratic in $\BPs$
for  $\BPs \ll \BW \BW^T$.
\end{theorem}

\begin{proof}
The fixed point equation for the $\BW$ update is $\Delta \BW  =  \ \BU  \BS^{-1} -  \BW  =  \BZe \ 
\Rightarrow  \ \BW = \ \BU \BS^{-1}$.
Using the
definition of $\BU$ and $\BS$, we have
$\BW \ = \ \left(\frac{1}{n}  \sum_{i=1}^n \Bv_i  \
\Bmu_i^T \right)
 \left(\frac{1}{n}  \sum_{i=1}^n
\Bmu_i \ \Bmu_i^T \ + \ \BSi \right)^{-1}
 \ $.
$\BW$ is the ridge regression solution of
\begin{align}
\frac{1}{n} \ \sum_{i=1}^n
\left\| \Bv_i \ - \ \BW \ \Bmu_i \right\|_2^2 \ + \
\left\| \BW \ \BSi^{1/2}
\right\|_{\mathrm{F}}^2 
= \
\TR \left(\frac{1}{n} \ \sum_{i=1}^n \Bep_i \ \Bep_i^T \ + \ \BW \
  \BSi \ \BW^T \right) \ ,
\end{align}
where $\TR$ is the trace. After multiplying out all $\Bep_i \Bep_i^T$
in $1/n \sum_{i=1}^n \Bep_i \Bep_i^T$, we obtain:
\begin{align}
\label{eq:conv1}
\BE \ &= \ \frac{1}{n} \ \sum_{i=1}^n \Bep_i \ \Bep_i^T \ + \ \BW \
  \BSi \ \BW^T \ .
\end{align}
For the fixed point of $\BPs$, the update rule gives:
$\diag\left( \BPs \right) = \diag \left(\frac{1}{n} \sum_{i=1}^n \Bep_i \Bep_i^T + \BW \BSi  \BW^T \right)$.
Thus, $\BW$ minimizes  $\TR \left( \BPs \right)$ given $\Bmu_i$ and $\BSi$.
Multiplying the Woodbury identity for
$\left( \BW  \BW^T +  \BPs\right)^{-1}$
from left and right by $\BPs$ gives
\begin{align}
\label{eq:psi2}
\BW  \BSi  \BW^T \ &= \
\BPs  -  \BPs  \left( \BW \ \BW^T  +  \BPs \right)^{-1}
\ \BPs .
\end{align}
Inserting this into the expression for  $\diag (\BPs)$ and taking
the trace gives
\begin{align}\label{eq:trE}
 \TR \left(\frac{1}{n} \ \sum_{i=1}^n \Bep_i \
\Bep_i^T \right)   =  \TR \left( \BPs  \left( \BW  \BW^T  +  \BPs \right)^{-1}
 \BPs \right)  
\leq  \TR \left(\left( \BW  \BW^T  +  \BPs \right)^{-1}
\right)  \TR \left( \BPs \right)^2 \ .
\end{align}
Therefore for $\BPs \ll \BW \BW^T$ the error is quadratic in $\BPs$.
$\BW \BU^T =  \BW \BS \BW^T =  \BU  \BW^T$ follows from
fixed point equation $\BU=\BW \BS$.
Using  this and Eq.~\eqref{eq:psi2}, Eq.~\eqref{eq:conv1} is
\begin{align}
\label{eq:approxC}
\frac{1}{n} \ \sum_{i=1}^n \Bep_i \ \Bep_i^T \ - \
\BPs \ \left( \BW \ \BW^T \ + \ \BPs \right)^{-1} \ \BPs \ = \
 \BC \ - \ \BPs   \ - \   \BW \ \BS  \ \BW^T \ .
\end{align}
Using the trace norm (nuclear
norm or Ky-Fan n-norm) on matrices, Eq.~\eqref{eq:trE} states that the left hand
side of Eq.~\eqref{eq:approxC} is quadratic in $\BPs$ for $\BPs \ll \BW \BW^T$.
The trace norm of a positive
semi-definite matrix is its trace and bounds the Frobenius norm \cite{Srebro:04}.
Thus, for
$\BPs \ll  \BW \BW^T$, the covariance is approximated up to a
quadratic error in $\BPs$ according to Eq.~\eqref{eq:covApprox}. The
diagonal is exactly modeled.
\end{proof}

Since the minimization of the expected
reconstruction error $\TR \left( \BPs \right)$ is
based on $\Bmu_i$, the quality of reconstruction
depends on the correlation between $\Bmu_i$ and
$\Bv_i$.
We ensure maximal information in $\Bmu_i$
on $\Bv_i$ by the I-projection (the minimal Kullback-Leibler distance)
of the posterior onto the family of rectified and
normalized Gaussian distributions.

\section{Experiments}

\begin{table*}[!b]
\centering
\caption{Comparison of RFN with other unsupervised methods, where the
upper part contains methods that yielded sparse codes.
Criteria: sparseness of the code (SP),
reconstruction error (ER), difference between data and model
covariance (CO). The panels give the results for models with 50, 100 and 150 coding units.
Results are the mean of 900 instances, 100 instances for each dataset
D1 to D9 (maximal value: 999).
RFNs had the sparsest code, the lowest
reconstruction error, and the lowest covariance
approximation error of all methods that yielded sparse
representations (SP$>$10\%). \label{tab:compare}}
\begin{tabular*}{\textwidth}{l*{3}{>{\columncolor{mColor1} \raggedleft\arraybackslash}p{2.8em}}*{3}{>{\raggedleft\arraybackslash}p{2.8em}}*{3}{>{\columncolor{mColor1} \raggedleft\arraybackslash}p{2.8em}}}
\toprule
&\multicolumn{3}{c}{\small{{\bf undercomplete} 50 code units}}
&\multicolumn{3}{c}{\small{{\bf complete} 100 code units}}
&\multicolumn{3}{c}{\small{{\bf overcomplete} 150 code units}} \\[-0.30ex]
\cmidrule(r{0pt}){2-4}
\cmidrule(lr){5-7}
\cmidrule(l{0pt}){8-10}
& SP & ER & CO
& SP & ER & CO
& SP & ER & CO\\[-0.30ex]
RFN	&75\tiny$\pm$0	&249\tiny$\pm$3 &108\tiny$\pm$3\ &81\tiny$\pm$1	&68\tiny$\pm$9	&26\tiny$\pm$6\ 	&85\tiny$\pm$1	&17\tiny$\pm$6	&7\tiny$\pm$6 \\[-0.30ex]
RFNn &74\tiny$\pm$0	&295\tiny$\pm$4 &140\tiny$\pm$4\ &79\tiny$\pm$0	&185\tiny$\pm$5	&59\tiny$\pm$3\ 	&80\tiny$\pm$0	&142\tiny$\pm$4	&35\tiny$\pm$2 \\[-0.30ex]
DAE	&66\tiny$\pm$0	&251\tiny$\pm$3 & ---\ &69\tiny$\pm$0	&147\tiny$\pm$2	& ---\ 	&71\tiny$\pm$0	&130\tiny$\pm$2	& --- \\[-0.30ex]
RBM	&15\tiny$\pm$1	&310\tiny$\pm$4 & ---\ &7\tiny$\pm$1	&287\tiny$\pm$4	& ---\ 	&5\tiny$\pm$0	&286\tiny$\pm$4	& --- \\[-0.30ex]
FAsp	&40\tiny$\pm$1	&999\tiny$\pm$63 &999\tiny$\pm$99\ &63\tiny$\pm$0	&999\tiny$\pm$65	&999\tiny$\pm$99\ 	&80\tiny$\pm$0	&999\tiny$\pm$65	&999\tiny$\pm$99 \\
\midrule
FAlap	&4\tiny$\pm$0	&239\tiny$\pm$6 &341\tiny$\pm$19\ &6\tiny$\pm$0	&46\tiny$\pm$4	&985\tiny$\pm$45\ 	&4\tiny$\pm$0	&46\tiny$\pm$4	&976\tiny$\pm$53 \\ [-0.30ex]
ICA	&2\tiny$\pm$0	&174\tiny$\pm$2 & ---\ &3\tiny$\pm$1	&0\tiny$\pm$0	& ---\ 	&3\tiny$\pm$1	&0\tiny$\pm$0	& ---\\ [-0.30ex]
SFA	&1\tiny$\pm$0	&218\tiny$\pm$5 &94\tiny$\pm$3\ &1\tiny$\pm$0	&16\tiny$\pm$1	&114\tiny$\pm$5\ 	&1\tiny$\pm$0	&16\tiny$\pm$1	&285\tiny$\pm$7\\[-0.30ex]
FA	&1\tiny$\pm$0	&218\tiny$\pm$4 &90\tiny$\pm$3\ &1\tiny$\pm$0	&16\tiny$\pm$1	&83\tiny$\pm$4\ 	&1\tiny$\pm$0	&16\tiny$\pm$1	&263\tiny$\pm$6\\[-0.30ex]
PCA	&0\tiny$\pm$0	&174\tiny$\pm$2 & ---\ &2\tiny$\pm$0	&0\tiny$\pm$0	& --- 	&2\tiny$\pm$0	&0\tiny$\pm$0	& --- \\[-0.30ex]
\bottomrule
\end{tabular*}%
\end{table*}%

\paragraph{RFNs vs.~Other Unsupervised Methods.}
We assess the performance of rectified factor
networks (RFNs) as unsupervised methods for data representation.
We compare 
(1) \textbf{RFN}: rectified factor networks, 
(2) \textbf{RFNn}: RFNs without normalization,
(3) \textbf{DAE}: denoising autoencoders with ReLUs,
(4) \textbf{RBM}: restricted Boltzmann machines with Gaussian visible units,
(5) \textbf{FAsp}: factor analysis with Jeffrey's prior ($p(z)\propto
1/z$) on the hidden units which is
sparser than a Laplace prior,
(6) \textbf{FAlap}: factor analysis with Laplace prior on the
hidden units,
(7) \textbf{ICA}: independent component analysis by FastICA \cite{Hyvarinen:99a},
(8) \textbf{SFA}: sparse factor analysis with a Laplace prior on the parameters,
(9) \textbf{FA}: standard factor analysis,
(10) \textbf{PCA}: principal component analysis.
The number of components are fixed to 50, 100 and 150 for each method.
We generated nine different benchmark datasets (D1 to D9),
where each dataset consists of 100 instances. Each instance
has 100 samples and 100 features resulting in a
100$\times$100 matrix.
Into these matrices, biclusters are implanted \cite{Hochreiter:10s}.
A bicluster is a pattern of particular features which is found in
particular samples like a pathway activated in some samples.
An optimal representation will only code the biclusters 
that are present in a sample.
The datasets have different noise levels and different bicluster sizes.
Large biclusters have 20--30 samples and 20--30 features,
while small biclusters 3--8 samples and 3--8 features.
The pattern's signal strength in a particular sample
was randomly chosen according to the Gaussian $\Ncal\left(1,1\right)$.
Finally, to each matrix, zero-mean Gaussian
background noise was added with standard deviation 1, 5, or 10.
The datasets are characterized by Dx=$(\sigma,n_1,n_2)$ with
background noise $\sigma$, number of large biclusters $n_1$,
and the number of small biclusters $n_2$:
D1=(1,10,10),
D2=(5,10,10),
D3=(10,10,10),
D4=(1,15,5),
D5=(5,15,5),
D6=(10,15,5),
D7=(1,5,15),
D8=(5,5,15),
D9=(10,5,15). \\
We evaluated the methods according to the (1) {\em sparseness} of the
components, the (2) input {\em reconstruction error} from the code,
and the (3) {\em covariance reconstruction error} for generative models.
For RFNs sparseness is the percentage of the components that are
exactly 0, while 
for others methods it is the percentage of components with an absolute value
smaller than 0.01.
The reconstruction error is the sum
of the squared errors across samples.
The covariance reconstruction error is the Frobenius norm
of the difference between model and data covariance.
See supplement for more details on the data and for 
information on hyperparameter selection for the different methods.
Tab.~\ref{tab:compare} gives averaged results for models with 50
(undercomplete), 100 (complete) and 150 (overcomplete) coding units.
Results are the mean of 900 instances consisting of 100 instances for
each dataset D1 to D9. 
In the supplement, we separately 
tabulate the results for D1 to D9 and confirm them
with different noise levels.
FAlap did not yield sparse codes since the variational
parameter did not push the absolute representations below the
threshold of 0.01.
The variational approximation to the Laplacian is a Gaussian
distribution \cite{Palmer:06_short}.
{\em RFNs had the sparsest code, the lowest
reconstruction error, and the lowest covariance
approximation error of all methods that yielded sparse
representations (SP\textgreater 10\%).}

\paragraph{RFN Pretraining for Deep Nets.}

\begin{figure*}[t!]
\begin{center}
\subfigure[MNIST digits]{
\includegraphics[angle=0,width= 0.46\textwidth]{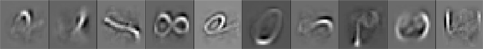}}
\subfigure[MNIST digits with random image background]{
\includegraphics[angle=0,width= 0.46\textwidth]{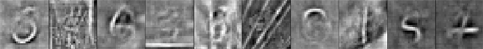}}\\[-2.0ex]
\subfigure[MNIST digits with random noise background]{
\includegraphics[angle=0,width= 0.46\textwidth]{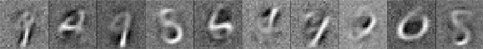}}
\subfigure[convex and concave shapes]{
\includegraphics[angle=0,width= 0.46\textwidth]{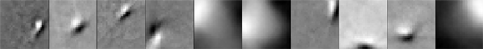}}\\[-2.0ex]
\subfigure[tall and wide rectangular]{
\includegraphics[angle=0,width= 0.46\textwidth]{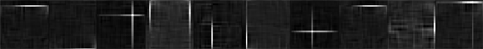}}
\subfigure[rectangular images on background images]{
\includegraphics[angle=0,width= 0.46\textwidth]{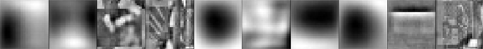}}\\[-2.0ex]
\subfigure[CIFAR-10 images (best viewed in color)] {
\includegraphics[angle=0,width= 0.46\textwidth]{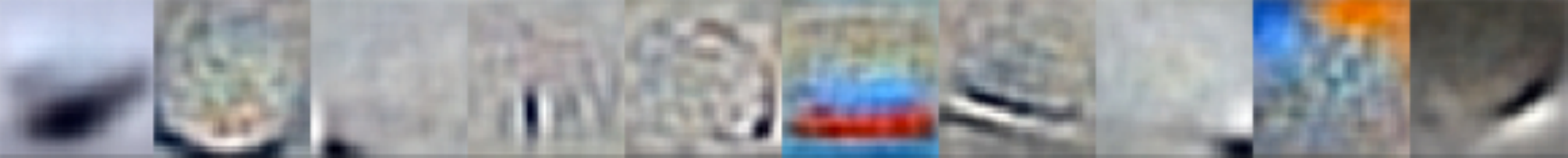}}
\subfigure[NORB images] {
\includegraphics[angle=0,width= 0.46\textwidth]{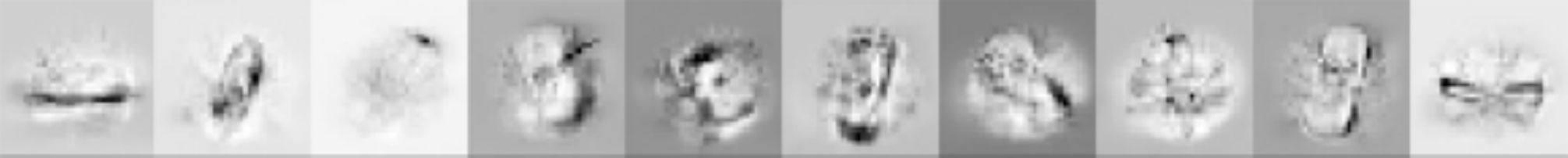}} 
\caption{Randomly selected filters trained on
image datasets using an RFN with 1024 hidden units.
RFNs learned stroke, local and global blob detectors.
RFNs are robust to background noise (b,c,f).\label{fig:repField}
}
\end{center}
\vspace*{-5pt}
\end{figure*}

We assess the performance of rectified factor
networks (RFNs) if used for pretraining of
deep networks.
Stacked RFNs are obtained by first training a single layer RFN
and then passing on the resulting representation as input for training
the next RFN.
The deep network architectures use a RFN pretrained 
first layer (RFN-1) or stacks of 3 RFNs giving a 3-hidden layer network.
The classification performance of deep networks with RFN pretrained layers
was compared to
(i) support vector machines,
(ii) deep networks pretrained by stacking denoising autoencoders (SDAE),
(iii) stacking regular autoencoders (SAE),
(iv) restricted Boltzmann machines (RBM), and
(v) stacking restricted Boltzmann machines (DBN).\\
\bgroup
\def\arraystretch{1.1}
\begin{table*}[b!]
\caption{Results of
deep networks pretrained by RFNs and other models 
(taken from \protect\cite{LeCun:04,Vincent:10_short,Larochelle:07_short,Krizhevsky:09}).
The test error rate
is reported together with the 95\% confidence interval.
The best performing method is given in bold,
as well as those for which confidence intervals overlap.
The first column gives the dataset,
the second the size of training,
validation and test set,
the last column indicates the number of hidden layers of the selected
deep network.
In only one case RFN pretraining was significantly worse than the
best method but still the second best.
In six out of the nine experiments RFN pretraining performed best, where in four
cases it was significantly the best.}

\label{tab:tab_res}%
\begin{tabular*}{\textwidth}{ll
*{5}{>{ \raggedleft\arraybackslash}p{3.7em}}
*{1}{>{ \raggedleft\arraybackslash}p{5.2em}}}
\toprule[1pt]
Dataset& & SVM & RBM & DBN & SAE   & SDAE  & RFN \\[-0.30ex]
\midrule
MNIST &\tiny{50k-10k-10k} & \textbf{1.40\tiny$\pm$0.23} & \textbf{1.21\tiny$\pm$0.21} & \textbf{1.24\tiny$\pm$0.22} & \textbf{1.40\tiny$\pm$0.23} & \textbf{1.28\tiny$\pm$0.22}  & \textbf{1.27\tiny$\pm$0.22}  \small(1)\\[-0.30ex]
basic &\tiny{10k-2k-50k}  & 3.03\tiny$\pm$0.15 & 3.94\tiny$\pm$0.17 & 3.11\tiny$\pm$0.15 & 3.46\tiny$\pm$0.16 & \textbf{2.84\tiny$\pm$0.15} & \textbf{2.66\tiny$\pm$0.14}  \small(1)\\[-0.30ex]
bg-rand & \tiny{10k-2k-50k} & 14.58\tiny$\pm$0.31 & 9.80\tiny$\pm$0.26 & \textbf{6.73\tiny$\pm$0.22} & 11.28\tiny$\pm$0.28 & 10.30\tiny$\pm$0.27  & 7.94\tiny$\pm$0.24  \small(3) \\[-0.30ex] 
bg-img & \tiny{10k-2k-50k}& 22.61\tiny$\pm$0.37 & \textbf{16.15\tiny$\pm$0.32} & \textbf{16.31\tiny$\pm$0.32} & 23.00\tiny$\pm$0.37 & \textbf{16.68\tiny$\pm$0.33} & \textbf{15.66\tiny$\pm$0.32}  \small(1)\\[-0.30ex]
rect &\tiny{1k-0.2k-50k}& 2.15\tiny$\pm$0.13 & 4.71\tiny$\pm$0.19 & 2.60\tiny$\pm$0.14 & 2.41\tiny$\pm$0.13 & 1.99\tiny$\pm$0.12  & \textbf{0.63\tiny$\pm$0.06}  \small(1)\\[-0.30ex]
rect-img & \tiny{10k-2k-50k}& 24.04\tiny$\pm$0.37 & 23.69\tiny$\pm$0.37 & 22.50\tiny$\pm$0.37 & 24.05\tiny$\pm$0.37 & 21.59\tiny$\pm$0.36 & \textbf{20.77\tiny$\pm$0.36}  \small(1)\\[-0.30ex]
convex & \tiny{10k-2k-50k}& 19.13\tiny$\pm$0.34 & 19.92\tiny$\pm$0.35 & 18.63\tiny$\pm$0.34 & 18.41\tiny$\pm$0.34 & 19.06\tiny$\pm$0.34 & \textbf{16.41\tiny$\pm$0.32}  \small(1)\\[-0.30ex]
NORB & \tiny{19k-5k-24k}& 11.6\tiny$\pm$0.40 & 8.31\tiny$\pm$0.35 & - & 10.10\tiny$\pm$0.38 & 9.50\tiny$\pm$0.37 & \textbf{7.00\tiny$\pm$0.32}  \small(1)\\[-0.30ex]
CIFAR & \tiny{40k-10k-10k}& 62.7\tiny$\pm$0.95 & \textbf{40.39\tiny$\pm$0.96} & 43.38\tiny$\pm$0.97 & 43.25\tiny$\pm$0.97 & - & \textbf{41.29\tiny$\pm$0.95}  \small(1)\\[-0.30ex]
\bottomrule
\end{tabular*}
\end{table*}
\egroup
The benchmark datasets and results are taken from previous
publications \cite{LeCun:04,Vincent:10_short,Larochelle:07_short,Krizhevsky:09} and
contain: (i) \emph{MNIST} (original MNIST),
(ii) \emph{basic} (a smaller subset of MNIST for training), (iii) \emph{bg-rand} (MNIST with
random noise background), (iv) \emph{bg-img} (MNIST with random image background),
(v) \emph{rect} (discrimination between tall and wide rectangles),
(vi) \emph{rect-img} (discrimination between tall and wide rectangular images overlayed on
random background images), (vii) \emph{convex} (discrimination between convex and concave shapes),
(viii) \emph{CIFAR-10} (60k color images in 10 classes), and
(ix) \emph{NORB} (29,160 stereo image pairs of 5
generic categories).
For each dataset its size of training, validation and test set is given in
the second column of Tab.~\ref{tab:tab_res}.
As preprocessing we only performed median centering.
Model selection is
based on the validation set performance \cite{Vincent:10_short}.
The RFNs hyperparameters are
(i) the number of units per layer from $\{1024, 2048, 4096\}$ and
(ii) the dropout rate from $\{0.0, 0.25, 0.5, 0.75\}$.
The learning rate was fixed to its default value of $\eta=0.01$.
For supervised fine-tuning with stochastic gradient descent,
we selected the learning rate
from $\{0.1, 0.01, 0.001\}$, the masking noise from $\{0.0, 0.25\}$,
and the number of layers from  $\{1, 3\}$.
Fine-tuning was stopped based on the
validation set performance, following \cite{Vincent:10_short}.
The test error rates together with the
95\% confidence interval (computed according to \cite{Vincent:10_short})
for deep network pretraining by RFNs and other methods
are given in Tab.~\ref{tab:tab_res}. Fig.~\ref{fig:repField} shows
learned filters.
The result of the best performing method is given in bold,
as well as the result of those methods for which confidence intervals overlap.
RFNs were only once significantly worse than the
best method but still the second best.
In six out of the nine experiments RFNs performed best, where in four
cases it was significantly the best.

\paragraph{RFNs in Drug Discovery.}
Using RFNs we analyzed gene expression datasets of two projects in the lead
optimization phase of a big pharmaceutical company \cite{Verbist:15_short}.
The first project aimed at finding
novel antipsychotics that target PDE10A. The second project was an
oncology study that focused on compounds inhibiting the FGF receptor.
In both projects, the expression data was summarized by FARMS \cite{Hochreiter:06} 
and standardized.
RFNs were trained with 500 hidden units, no masking noise, and a learning 
rate of $\eta=0.01$. 
The identified transcriptional modules are shown in Fig.~\ref{fig:QSTAR}. 
Panels A and B illustrate that RFNs found rare and small 
events in the input. 
In panel A only a few drugs are genotoxic 
(rare event) by downregulating the expression of a small number of
tubulin genes (small event). The genotoxic effect stems from 
the formation of micronuclei (panel C and D) since the mitotic
spindle apparatus is impaired. 
Also in panel B, RFN identified a rare and small event which is a 
transcriptional module that has a negative feedback to the MAPK signaling pathway. 
Rare events are unexpectedly inactive drugs (black dots), 
which do not inhibit the FGF receptor. 
Both findings were not detected by other unsupervised
methods, while they were highly relevant
and supported decision-making in both projects \cite{Verbist:15_short}.

\begin{figure*}[t!]
\begin{center}
\includegraphics[angle=0,width= \textwidth]{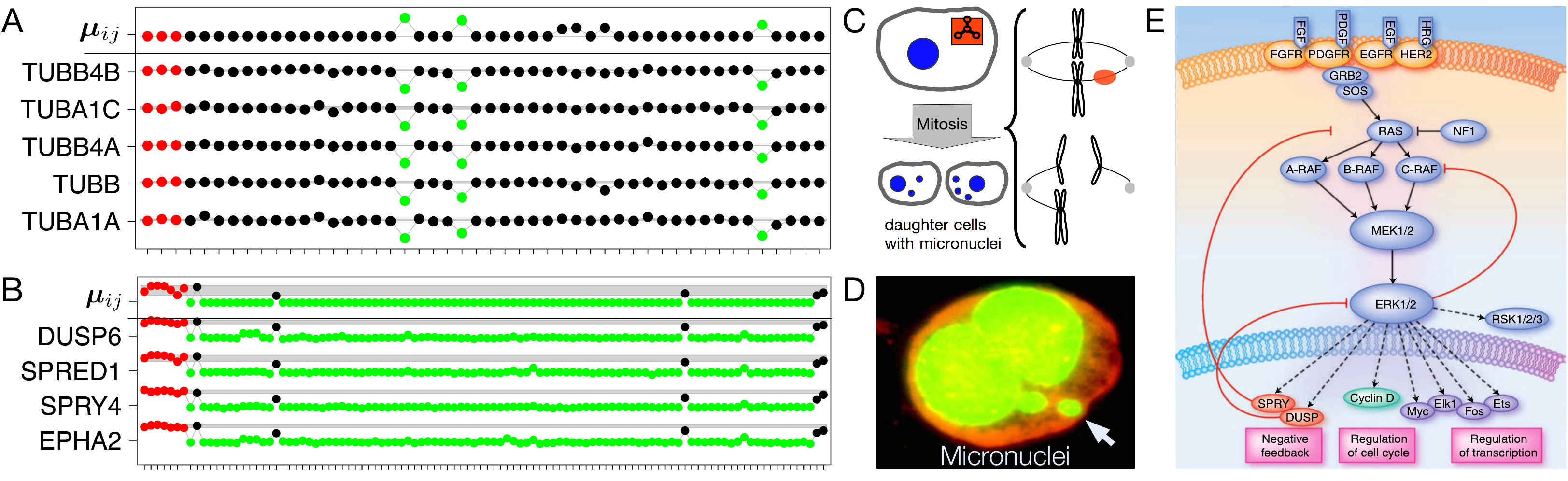}\caption{\label{fig:QSTAR}
Examples of small and rare events identified by RFN in two drug design studies, 
which were missed by previous methods.
Panel A and B: first row gives the coding unit, while 
the other rows display expression values of genes for 
controls (red), active drugs (green), and inactive drugs (black). 
Drugs (green) in panel A strongly downregulate the expression 
of tubulin genes which hints at a genotoxic effect by 
the formation of micronuclei (C). 
The micronuclei were confirmed by microscopic analysis (D). 
Drugs (green) in panel B show a transcriptional effect on genes 
with a negative feedback to the MAPK signaling pathway (E)
and therefore are potential cancer drugs.}
\end{center}
\end{figure*}

\section{Conclusion}
We have introduced rectified factor networks (RFNs) for
constructing very sparse and non-linear input representations with many
coding units in a generative framework.
Like factor analysis, RFN learning explains the data variance by its
model parameters.
The RFN learning algorithm is a
posterior regularization method which enforces
non-negative and normalized posterior means.
We have shown that RFN learning is a generalized alternating minimization
method which can be proved to converge and to be correct.
RFNs had the sparsest code, the lowest
reconstruction error, and the lowest covariance
approximation error of all methods that yielded sparse
representations (SP$>$10\%).
RFNs have shown that they improve performance if used for pretraining
of deep networks.
In two pharmaceutical drug discovery studies, RFNs 
detected small and rare gene modules that
were so far missed by other unsupervised methods.
These gene modules were highly relevant and 
supported the decision-making in both studies.
RFNs are geared to
large datasets, sparse coding, and many representational units,
therefore they
have high potential as unsupervised deep learning techniques.

\paragraph*{Acknowledgment.}
The Tesla K40 used for this research
was donated by the NVIDIA Corporation.

\small{
\bibliographystyle{unsrt}
\setlength{\bibsep}{1.5pt}

}

\vfill
\onecolumn

\newpage
\normalsize
\appendix
\section*{Supplementary Material}
\beginsupplement

\pagenumbering{arabic}

\thispagestyle{empty}

\contentsline {section}{\numberline {S1}Introduction}{5}{appendix.A}
\contentsline {section}{\numberline {S2}Rectified Factor Network (RFN) Algorithms}{5}{appendix.B}
\contentsline {section}{\numberline {S3}Convergence Proof for the RFN Learning Algorithm}{9}{appendix.C}
\contentsline {section}{\numberline {S4}Correctness Proofs for the RFN Learning Algorithms}{11}{appendix.D}
\contentsline {subsection}{\numberline {S4.1}Diagonal Noise Covariance Update}{12}{subsection.D.1}
\contentsline {subsection}{\numberline {S4.2}Full Noise Covariance Update}{14}{subsection.D.2}
\contentsline {section}{\numberline {S5}Maximum Likelihood Factor Analysis}{15}{appendix.E}
\contentsline {section}{\numberline {S6}The RFN Objective}{18}{appendix.F}
\contentsline {section}{\numberline {S7}Generalized Alternating Minimization}{20}{appendix.G}
\contentsline {section}{\numberline {S8}Gradient-based M-step}{22}{appendix.H}
\contentsline {subsection}{\numberline {S8.1}Gradient Ascent}{22}{subsection.H.1}
\contentsline {subsection}{\numberline {S8.2}Newton Update}{23}{subsection.H.2}
\contentsline {subsubsection}{\numberline {S8.2.1}Newton Update of the Loading Matrix}{24}{subsubsection.H.2.1}
\contentsline {subsubsection}{\numberline {S8.2.2}Newton Update of the Noise Covariance}{24}{subsubsection.H.2.2}
\contentsline {paragraph}{$\bm {\Psi }$ as parameter.}{24}{thmt@dummyctr.dummy.9}
\contentsline {paragraph}{$\bm {\Psi }^{-1}$ as parameter.}{25}{thmt@dummyctr.dummy.10}
\contentsline {section}{\numberline {S9}Gradient-based E-Step}{26}{appendix.I}
\contentsline {subsection}{\numberline {S9.1}Motivation for Rectifying and Normalization Constraints}{26}{subsection.I.1}
\contentsline {subsection}{\numberline {S9.2}The Full E-step Objective}{27}{subsection.I.2}
\contentsline {subsection}{\numberline {S9.3}E-step for Mean with Rectifying Constraints}{28}{subsection.I.3}
\contentsline {subsubsection}{\numberline {S9.3.1}The E-Step Minimization Problem}{28}{subsubsection.I.3.1}
\contentsline {subsubsection}{\numberline {S9.3.2}The Projection onto the Feasible Set}{29}{subsubsection.I.3.2}
\contentsline {subsection}{\numberline {S9.4}E-step for Mean with Rectifying and Normalizing Constraints}{30}{subsection.I.4}
\contentsline {subsubsection}{\numberline {S9.4.1}The E-Step Minimization Problem}{30}{subsubsection.I.4.1}
\contentsline {paragraph}{Generalized Reduced Gradient.}{30}{section*.19}
\contentsline {paragraph}{Gradient Projection Methods.}{30}{section*.20}
\contentsline {subsubsection}{\numberline {S9.4.2}The Projection onto the Feasible Set}{30}{subsubsection.I.4.2}
\contentsline {subsection}{\numberline {S9.5}Gradient and Scaled Gradient Projection and Projected Newton}{33}{subsection.I.5}
\contentsline {subsubsection}{\numberline {S9.5.1}Gradient Projection Algorithm}{33}{subsubsection.I.5.1}
\contentsline {subsubsection}{\numberline {S9.5.2}Scaled Gradient Projection and Projected Newton Method}{34}{subsubsection.I.5.2}
\contentsline {subsubsection}{\numberline {S9.5.3}Combined Method}{35}{subsubsection.I.5.3}
\contentsline {section}{\numberline {S10} Alternative Gaussian Prior}{36}{appendix.J}
\contentsline {section}{\numberline {S11} Hyperparameters Selected for Method Assessment}{38}{appendix.K}
\contentsline {section}{\numberline {S12} Data Set I}{39}{appendix.L}
\contentsline {section}{\numberline {S13} Data Set II}{43}{appendix.M}
\contentsline {section}{\numberline {S14} RFN Pretraining for Convolution Nets}{47}{appendix.N}
\newpage

\listoftheorems

\newpage

\listofalgorithms
\newpage

\section{Introduction}
\label{sec:intro}

This supplement contains additional information complementing the 
main manuscript and is structured as follows:
First, the rectified factor
network (RFN) learning algorithm with E- and M-step updates, weight 
decay and dropout regularization is given in Section~\ref{sec:alg}. 
In Section \ref{sec:conv},  we proof that the (RFN) learning algorithm  
is a ``generalized alternating minimization'' (GAM) algorithm and 
converges to a solution that maximizes the RFN objective. 
The correctness of the RFN algorithm is proofed in Section \ref{sec:correct}. 
Section~\ref{sec:mlfa} describes the maximum likelihood factor 
analysis model and the model selection by the EM-algorithm.
The RFN objective, which has to be
maximized, is described in Section~\ref{sec:objective}. 
Next, RFN's GAM algorithm via gradient descent both in the M-step 
and the E-step is reported in the Section~\ref{sec:GAM}.
The following sections~\ref{sec:gradientM} and \ref{sec:gradientE} describe 
the gradient-based M- and E-step, respectively.
In Section~\ref{sec:prior}, we describe how the RFNs sparseness
can be controlled by a Gaussian prior.
Additional information on the selected hyperparameters of the benchmark methods is given in Section~\ref{S_sec:hyperpars}. 
The sections \ref{S_sec:data1} and \ref{S_sec:data2} describe the data generation of the benchmark datasets and report
the results for three different experimental settings, namely for extracting 50 (undercomplete), 100 (complete) or 150 (overcomplete) factors / hidden units. 
Finally, Section~\ref{S_sec:ConvNets} describes experiments, that we have done to assess the performance of RFN {\em first layer} pretraining
on {\em CIFAR-10} and {\em CIFAR-100} for
three deep convolutional network architectures:
(i) the AlexNet \cite{Ciresan:12cvpr, Krizhevsky:12},
(ii) Deeply Supervised Networks (DSN) \cite{Lee:14}, and
(iii) our 5-Convolution-Network-In-Network (5C-NIN).

\section{Rectified Factor Network (RFN) Algorithms}
\label{sec:alg}

Algorithm~\ref{S_alg:RFN} is the rectified factor
network (RFN) learning algorithm.
The RFN algorithm calls 
Algorithm~\ref{alg:ProjEstep} to project the posterior probability  $p_i$
onto the family of rectified and normalized variational distributions
$Q_i$.  
Algorithm~\ref{alg:ProjEstep}
guarantees an improvement of the E-step objective  
$O =  \frac{1}{n} \sum_{i=1}^{n} D_{\mathrm{KL}}(Q_i \parallel p_i)$.
Projection Algorithm~\ref{alg:ProjEstep} relies on different
projections, where a more complicated projection is tried if a simpler
one failed to improve the E-step objective. 
If all following Newton-based gradient projection methods fail to
decrease the E-step objective, then projection
Algorithm~\ref{alg:ProjEstep} falls back to gradient projection
methods. First the equality constraints are solved and inserted into
the objective. Thereafter, the constraints are convex and gradient
projection methods are applied. 
This approach is 
called ``generalized reduced gradient method'' \cite{Abadie:69}, which
is our preferred alternative method. If this method fails, then 
Rosen's gradient projection method \cite{Rosen:61} is used. Finally, 
the method of Haug and Arora \cite{Haug:79} is used.

First we consider Newton-based projection methods, 
which are used by Algorithm~\ref{alg:ProjEstep}.
Algorithm~\ref{alg:Psimple} performs a simple projection,
which is the projected Newton method with learning rate set to one.
This projection is very fast and ideally suited to be performed on GPUs
for RFNs with many coding units.
Algorithm~\ref{alg:PsimpleRec} is the fast and simple projection
without normalization even simpler than Algorithm~\ref{alg:Psimple}.
Algorithm~\ref{alg:ProjScale} generalizes Algorithm~\ref{alg:Psimple}
by introducing step sizes $\lambda$ and $\gamma$. The step size
$\lambda$ scales the gradient step, while $\gamma$ scales the difference
between to old projection and the new projection.
For both $\lambda$ and $\gamma$ annealing steps, that is, learning
rate decay is used to find an appropriate update.

If these Newton-based update rules do not work, then 
Algorithm~\ref{alg:ProjScaleRed} is
used. Algorithm~\ref{alg:ProjScaleRed} performs a scaled projection
with a reduced Hessian matrix $\BH$ instead of the
full Hessian $\BSi_p^{-1}$. For computing $\BH$ an $\epsilon$-active set is
determined, which consists of all $j$ with $\mu_j\leq \epsilon$.
The reduced matrix $\BH$ is the Hessian 
$\BSi_p^{-1}$ with $\epsilon$-active columns and rows $j$ fixed to unit
vector $\Be_j$. 

The RFN algorithm allows regularization of the parameters $\BW$ and
$\BPs$ (off-diagonal elements) by weight decay.
Priors on the parameters can be introduced. If the priors are convex
functions, then convergence of the RFN algorithm is still ensured.
The weight decay Algorithm~\ref{alg:WD} 
can optionally be used
after the M-step of  Algorithm~\ref{S_alg:RFN}.
Coding units can be regularized by dropout.
However dropout is not covered by
the convergence proof for the RFN algorithm.
The dropout Algorithm~\ref{alg:dropout} is applied during the projection
between rectifying and normalization.
Methods like mini-batches or other stochastic gradient methods are not
covered by the convergence proof for the RFN algorithm. 
However, in \cite{Gunawardana:05} it is shown how to generalize the GAM
convergence proof to mini-batches as it is 
shown for the incremental EM algorithm. Dropout and other stochastic
gradient methods can be show to converge similar to mini-batches.

\begin{algorithm}
\caption{Rectified Factor Network} \label{S_alg:RFN}
\begin{algorithmic}
\STATE {\ ~}
\item[\textbf{Input}]
\STATE for $1 \leq i \leq n$: $\Bv_i \in \bbR^m$, 
\STATE number of coding units $l$ 
\item[\textbf{Hyper-Parameters}]
\STATE $\Psi_{\mathrm{min}}$, 
$W_{\mathrm{max}}$, $\eta_{\Psi}$, $\eta_W$, $\rho$, $\tau$, $1 < \eta \leq 1$
\item[\textbf{Initialization}]
\STATE $\BPs=\tau \BI$, $\BW$ element-wise random in $[-\rho,\rho]$,  
\STATE $\BC \ = \ \frac{1}{n} \ \sum_{k=1}^n  \Bv_k \ \Bv_k^T$, STOP=false
\item[\textbf{Main}]
\WHILE{STOP=false}
\STATE \textbf{------E-step1------}
\FORALL{$1 \leq i \leq n$}
\STATE 
$(\Bmu_p)_i \ = \  \left(\BI \
  + \ \BW^T \BPs^{-1}\BW \right)^{-1}
 \BW^T \BPs^{-1} \ \Bv_i$ 
\ENDFOR 
\STATE 
$\BSi \ = \  \left(  \BI \ + \ \BW^T
\BPs^{-1}\BW \right)^{-1}$ 
\STATE \textbf{------Projection------}
\STATE perform projection of $(\Bmu_p)_i$ onto the feasible set by
Algorithm~\ref{alg:ProjEstep} giving $\Bmu_i$
\STATE \textbf{------E-step2------}
\STATE 
$\BU \ = \ \frac{1}{n} \ \sum_{i=1}^n \Bv_i  \  \Bmu_i^T $
\STATE {\ ~}
\STATE 
$\BS \ = \  \frac{1}{n} \ \sum_{i=1}^n \Bmu_i \ \Bmu_i^T \ + \ \BSi$
\STATE \textbf{------M-step------}
\STATE  $\eta_W = \eta_{\Psi} = \eta$
\STATE  $\BE \ = \
 \BC \ - \  \BU \ \BW^T  \  - \ \BW \  \BU 
\ + \   \BW \ \BS \ \BW^T$
\STATE \textit{-----$\BW$ update------}
\STATE  $\BW \ = \ \BW \ + \ \eta_W \ \left( 
 \BU \ \BS^{-1} \ - \  \BW \right)$ 
\STATE \textit{-----diagonal $\BPs$ update------}
\FORALL{$1 \leq k \leq m$}
\STATE 
$\Psi_{kk} \ = \ \Psi_{kk} \ + \
\eta_{\Psi} \  \left( E_{kk} \ - \ \Psi_{kk} \right)$
\ENDFOR 
\STATE \textit{-----full $\BPs$ update------}
\STATE 
$\BPs \ = \ \BPs \ + \ \eta_{\Psi} \  \left( \BE \ - \BPs \right)$
\STATE \textit{-----bound parameters------}
\STATE $\BW \ = \ \mathrm{median}\{-W_{\mathrm{max}} \ , \ \BW \ , \ W_{\mathrm{max}}\}$
\STATE $\BPs \ = \ \mathrm{median}\{\Psi_{\mathrm{min}} \ , \ \BPs \ ,\ \max\{\BC\} \}$
\STATE if stopping criterion is met: STOP=true
\ENDWHILE
\end{algorithmic}
\end{algorithm}

\begin{algorithm}
\caption{Projection with E-Step Improvement} \label{alg:ProjEstep}
\begin{algorithmic}
\STATE {\ ~}
\item[\textbf{Goal}]
\STATE obtain $\Bmu_i^\nn=\Bmu_i$ that decrease the E-step objective
\item[\textbf{Input}]
\STATE $\BSi^\nn \ = \ \BSi_p$, $\BSi^\oo \ = \ \BSi_p^\oo$
\STATE for $1 \leq i \leq n$: $(\Bmu_p)_i$, $\Bmu_i^\oo$, $p_i=\Ncal ((\Bmu_p)_i, \BSi_p)$
\STATE simple projection $\PP$ (rectified or rectified \& normalized), 
\STATE E-step objective: $O = \frac{1}{n} \sum_{i=1}^{n}
D_{\mathrm{KL}}(Q_i \parallel p_i)$
\STATE $\gamma_{\mathrm{min}}$, $\lambda_{\mathrm{min}}$,
$\rho_{\mathrm{\gamma}}$, $\rho_{\mathrm{\lambda}}$, $\epsilon$ (for $\epsilon$-active set)
\item[\textbf{Main}]
\STATE \textbf{-----Simple Projection------}
\STATE perform {\tt Newton Projection} by
Algorithm~\ref{alg:Psimple} or Algorithm~\ref{alg:PsimpleRec} 
\STATE \textbf{-----Scaled Projection------}
\IF{$0 \leq \Delta O$}
\STATE following loop for: (1) $\gamma$, (2) $\lambda$, or (3) $\gamma$ and $\lambda$ annealing
\STATE $\gamma=\lambda=1$
\WHILE{$0 \leq \Delta O$ and $\lambda>\lambda_{\mathrm{min}}$ and $\gamma>\gamma_{\mathrm{min}}$}
\STATE $\gamma \ = \ \rho_{\mathrm{\gamma}} \ \gamma$ (skipped for $\lambda$ annealing)
\STATE $\lambda \ = \ \rho_{\mathrm{\lambda}} \ \lambda$ (skipped for $\gamma$ annealing)
\STATE perform {\tt Scaled Newton Projection}  by
Algorithm~\ref{alg:ProjScale}
\ENDWHILE
\ENDIF
\STATE \textbf{-----Scaled Projection With Reduced Matrix------}
\IF{$0 \leq \Delta O$}
\STATE determine $\epsilon$-active set as all $j$ with $\mu_j\leq \epsilon$
\STATE set $\BH$ to $\BSi_p^{-1}$ with $\epsilon$-active columns and rows $j$ fixed to $\Be_j$
\STATE following loop for: (1) $\gamma$, (2) $\lambda$, or (3) $\gamma$ and $\lambda$ annealing
\STATE $\gamma=\lambda=1$
\WHILE{$0 \leq \Delta O$ and $\lambda>\lambda_{\mathrm{min}}$ and $\gamma>\gamma_{\mathrm{min}}$}
\STATE $\gamma \ = \ \rho_{\mathrm{\gamma}} \ \gamma$ (skipped for $\lambda$ annealing)
\STATE $\lambda \ = \ \rho_{\mathrm{\lambda}} \ \lambda$ (skipped for $\gamma$ annealing)
\STATE perform {\tt Scaled Projection With Reduced Matrix} by
Algorithm~\ref{alg:ProjScaleRed}
\ENDWHILE
\ENDIF
\STATE \textbf{-----General Gradient Projection------}
\WHILE{$0 \leq \Delta O$}
\STATE use generalized reduced gradient \cite{Abadie:69} OR
\STATE use Rosen's gradient projection \cite{Rosen:61} OR
\STATE use method of Haug and Arora \cite{Haug:79}
\ENDWHILE
\end{algorithmic}
\end{algorithm}

\begin{algorithm}
\caption{Simple Projection: Rectifying} \label{alg:PsimpleRec}
\begin{algorithmic}
\STATE {\ ~}
\item[\textbf{Goal}]
\STATE for $1 \leq i \leq n$: project $(\Bmu_p)_i$ onto feasible set giving  $\Bmu_i$
\item[\textbf{Input}]
\STATE $(\Bmu_p)_i$
\item[\textbf{Main}]
\FORALL{$1 \leq j \leq l$}
\STATE 
$\mu_{ij} \ = \  \max\left\{0,
 \left[(\Bmu_p)_i \right]_j \right\} $
\ENDFOR 
\end{algorithmic}
\end{algorithm}

\begin{algorithm}
\caption{Simple Projection: Rectifying and Normalization} \label{alg:Psimple}
\begin{algorithmic}
\STATE {\ ~}
\item[\textbf{Goal}]
\STATE for $1 \leq i \leq n$: project $(\Bmu_p)_i$ onto feasible set giving  $\Bmu_i$
\item[\textbf{Input}]
\STATE for $1 \leq i \leq n$: $(\Bmu_p)_i$
\item[\textbf{Rectifier}]
\FORALL{$1 \leq i \leq n$}
\FORALL{$1 \leq j \leq l$}
\STATE 
$\hat{\mu}_{ij} \ = \  \max\left\{0,
 \left[(\Bmu_p)_i \right]_j \right\} $
\ENDFOR 
\ENDFOR 
\item[\textbf{Normalizer}]
\FORALL{$1 \leq i \leq n$}
\IF{at least one $\hat{\mu}_{ij}>0$}
\FORALL{$1 \leq j \leq l$}
\STATE 
$\mu_{ij} \ = \  
\frac{\hat{\mu}_{ij}}{\sqrt{\frac{1}{n} \ \sum_{s=1}^{n} \hat{\mu}_{sj}^2}} $
\ENDFOR 
\ELSE
\FORALL{$1 \leq j \leq l$}
\STATE 
$\mu_{ij} \ =  \ \left\{ 
\begin{array}{lcl}
\sqrt{n} & \mathrm{for} & j =   \arg \max_{\hat{j}} \{ \left[(\Bmu_p)_i \right]_{\hat{j}} \}\\
0 & \mathrm{otherwise} &
\end{array} \right.$
\ENDFOR 
\ENDIF
\ENDFOR 
\end{algorithmic}
\end{algorithm}

\begin{algorithm}
\caption{Scaled Newton Projection} \label{alg:ProjScale}
\begin{algorithmic}
\STATE {\ ~}
\item[\textbf{Goal}]
\STATE perform a scaled Newton step with subsequent projection
\item[\textbf{Input}]
\STATE for $1 \leq i \leq n$: $(\Bmu_p)_i$
\STATE for $1 \leq i \leq n$: $\Bmu_i^\oo$
\STATE simple projection $\PP$ (rectified or rectified \& normalized), 
\STATE $\lambda$ (gradient step size), $\gamma$ (projection difference)
\item[\textbf{Main}]
\STATE $\Bd \ = \ \PP \left( \Bmu_i^\oo \ + \ 
\lambda \   (  (\Bmu_p)_i \ - \ \Bmu_i^\oo) \right)$ 
\STATE $\Bmu_i^\nn \ = \ \PP \left( \Bmu_i^\oo  \ + \ 
\gamma \left(\Bd \ - \  \Bmu_i^\oo \right)\right)$
\end{algorithmic}
\end{algorithm}

\begin{algorithm}
\caption{Scaled Projection With Reduced Matrix} \label{alg:ProjScaleRed}
\begin{algorithmic}
\STATE {\ ~}
\item[\textbf{Goal}]
\STATE perform a scaled projection step with reduced matrix
\item[\textbf{Input}]
\STATE for $1 \leq i \leq n$: $(\Bmu_p)_i$
\STATE for $1 \leq i \leq n$: $\Bmu_i^\oo$
\STATE simple projection $\PP$ (rectified or rectified \& normalized), 
\STATE $\lambda$, $\gamma$, $\BH$, $\BSi_p^{-1}$
\item[\textbf{Main}]
\STATE $\Bd \ = \ \PP \left( 
\Bmu_i^\oo \ + \ \lambda \  \BH^{-1} \ \BSi_p^{-1} (
(\Bmu_p)_i  \ - \ \Bmu_i^\oo ) 
\right)$ 
\STATE $\Bmu_i^\nn \ = \ \PP \left( \Bmu_i^\oo \ + \ \gamma \left(\Bd \ - \
  \Bmu_i^\oo \right)\right)$
\end{algorithmic}
\end{algorithm}

\begin{algorithm}
\caption{Weight Decay} \label{alg:WD}
\begin{algorithmic}
\STATE {\ ~}
\item[\textbf{Input}]
\STATE Parameters $\BW$
\STATE Weight decay factors $\gamma_G$ (Gaussian) and $\gamma_L$ (Laplacian)
\item[\textbf{Gaussian}]
\STATE $\BW \ = \ \BW \ - \ \gamma_G \ \BW$
\item[\textbf{Laplacian}]
\STATE 
$\hat{\BW} \ = \  \mathrm{median}\{-\gamma_L \ , \  \BW \ , \ \gamma_L \}$
\STATE 
$\BW \ = \ \BW \ - \ \hat{\BW}$
\end{algorithmic}
\end{algorithm}

\begin{algorithm}
\caption{Dropout} \label{alg:dropout}
\begin{algorithmic}
\STATE {\ ~}
\item[\textbf{Input}]
\STATE for $1 \leq i \leq n$: $\Bmu_i$
\STATE dropout probability $d$
\item[\textbf{Main}]
\FORALL{$1 \leq i \leq n$}
\FORALL{$1 \leq j \leq l$}
\STATE
$\Pr(\delta=0) \ = \ d $
\STATE
$\mu_{ij} \ = \ \delta \ \mu_{ij}$
\ENDFOR 
\ENDFOR 
\end{algorithmic}
\end{algorithm}

\vspace*{\fill}

\section{Convergence Proof for the RFN Learning Algorithm}
\label{sec:conv}

\begin{theorem}[RFN Convergence]
\label{th:RFNconvergence}
The rectified factor
network (RFN) learning algorithm given in Algorithm~\ref{S_alg:RFN}
is a ``generalized alternating minimization'' (GAM) algorithm
and converges to a solution that maximizes the objective $\Fcal$. 
\end{theorem}

\begin{proof}

The factor analysis 
EM algorithm is given by Eq.~\eqref{S_eq:mfirststepR}  and Eq.~\eqref{S_eq:mstepR}
in Section~\ref{sec:mlfa}. Algorithm~\ref{S_alg:RFN} is 
the factor analysis EM algorithm with modified the E-step and the M-step.
The E-step is modified by constraining the variational
distribution $Q$ to non-negative means and by
normalizing its means across the samples. The M-step is modified to
a Newton direction gradient step.

Like EM factor analysis, Algorithm~\ref{S_alg:RFN} aims at maximizing
the negative {\em free energy} $\Fcal$, which is
\begin{align}
\label{S_eq:objectiveProof}
\Fcal \ &= \ \frac{1}{n} \ \sum_{i=1}^n \log p(\Bv_i) \ - \
\frac{1}{n} \ \sum_{i=1}^n D_{\mathrm{KL}}(Q(\Bh_i) \parallel p(\Bh_i \mid \Bv_i) ) \\ \nonumber 
&= \  \frac{1}{n} \ \sum_{i=1}^n \int Q(\Bh_i) \ \log p(\Bv_i) \
d\Bh_i \ - \  \frac{1}{n} \ \sum_{i=1}^n \int Q(\Bh_i) \ \log
\frac{Q(\Bh_i)}{p(\Bh_i \mid \Bv_i)}  \ d\Bh_i \\\nonumber
&= \   - \  \frac{1}{n} \ \sum_{i=1}^n \int Q(\Bh_i) \ 
\log \frac{Q(\Bh_i)}{p(\Bh_i , \Bv_i)} \ d\Bh_i  \\\nonumber
&= \ - \  \frac{1}{n} \ \sum_{i=1}^n \int Q(\Bh_i) \ \log \frac{Q(\Bh_i)}{p(\Bh_i)} \ d\Bh_i \
+ \  \frac{1}{n} \ \sum_{i=1}^n \int Q(\Bh_i) \ \log p(\Bv_i \mid \Bh_i) \ d\Bh_i \\\nonumber
&= \  \frac{1}{n} \ \sum_{i=1}^n \int Q(\Bh_i) \ \log p(\Bv_i \mid
\Bh_i) \ d\Bh_i \ -  \frac{1}{n} \ \sum_{i=1}^n \ D_{\mathrm{KL}}(Q(\Bh_i)
\parallel p(\Bh_i) )  \ .
\end{align}
$D_{\mathrm{KL}}$ denotes the Kullback-Leibler (KL) divergence
\cite{Kullback:51}, which is larger than or equal to zero. 

Algorithm~\ref{S_alg:RFN} decreases 
$\frac{1}{n}  \sum_{i=1}^n D_{\mathrm{KL}}(Q(\Bh_i) \parallel p(\Bh_i \mid \Bv_i) )$
(the E-step objective) in its E-step
under constraints for non-negative means and normalization.
The constraint optimization problem from Section~\ref{sec:fullE} for
the E-step is
\begin{align}
\label{S_eq:EobjectiveRNo}
\min_{Q(\Bh_i)}  {\mbox{\ ~} } & \frac{1}{n} 
\sum_{i=1}^n D_{\mathrm{KL}}(Q(\Bh_i) \parallel p(\Bh_i \mid \Bv_i) )\\ \nonumber 
\mbox{ s.t. }  {\mbox{\ ~} } &\forall_i: \ \Bmu_i \ \geq \ \BZe \ , \\ \nonumber 
&\forall_j: \ \frac{1}{n}  \ \sum_{i=1}^{n} \mu_{ij}^2 \ = \ 1 \ .
\end{align}
The M-step of Algorithm~\ref{S_alg:RFN} aims at decreasing
\begin{align}
 \Ecal \ &= \  - \ \frac{1}{n} \ \sum_{i=1}^{n} \int_{\bbR^l} Q(\Bh_i) \ 
\log \left( p(\Bv_i \mid \Bh_i) \right) \  d\Bh_i \ .
\end{align}
Algorithm~\ref{S_alg:RFN} performs one gradient descent step in the Newton
direction to decrease $\Ecal$, while EM factor analysis
minimizes $\Ecal$.

From the modification of the E-step and the M-step follows that 
Algorithm~\ref{S_alg:RFN} is a 
{\em Generalized Alternating Minimization (GAM)} 
algorithm according to \cite{Gunawardana:05}. 
GAM is an EM algorithm that increases $\Fcal$ in the E-step
and increases $\Fcal$ in the M-step (see also Section~\ref{sec:GAM}).
The most important requirements for the convergence of the GAM
algorithm according to Theorem~\ref{th:GAMconvergence} (Proposition~5 in
\cite{Gunawardana:05}) are the increase of the objective $\Fcal$ in
both the E-step and the M-step. Therefore we first show these two decreases
before showing that all requirements of convergence
Theorem~\ref{th:GAMconvergence} are met.

{\bf Algorithm~\ref{S_alg:RFN} ensures to decrease the M-step objective.}
The M-step objective $\Ecal$ is convex in $\BW$ and $\BPs^{-1}$ according to
Theorem~\ref{th:NewtonLoading} and Theorem~\ref{th:NewtonInvNoise}.
The update with $\eta_W = \eta_{\Psi} = \eta = 1$ leads to the minimum
of $\Ecal$ according to 
Theorem~\ref{th:NewtonLoading} and Theorem~\ref{th:NewtonInvNoise}.
The convexity of $\Ecal$ guarantees that each update with $0 < \eta_W
= \eta_{\Psi} = \eta \leq 1$ decreases the M-step objective $\Ecal$,
except the current $\BW$ and $\BPs^{-1}$ are already the minimizers.

{\bf Algorithm~\ref{S_alg:RFN} ensures to decrease the E-step objective.}
The E-step decrease of Algorithm~\ref{S_alg:RFN} 
is performed by Algorithm~\ref{alg:ProjEstep}.
According to Theorem~\ref{th:projNewton} 
the scaled projection with reduced matrix
ensures a decrease of the E-step objective for
rectifying constraints (convex feasible set).
According to Theorem~\ref{th:gradProj} also gradient projection methods
ensure a decrease of the E-step objective for rectifying constraints.
For rectifying constraints and normalization, the feasible set is not
convex because of the equality constraints.
To optimize such problems,
the generalized reduced gradient method \cite{Abadie:69}
solves each equality constraint for one variable and inserts it into
the objective. For our problem Eq.~\eqref{S_eq:solveIt} gives the solution and
Eq.~\eqref{S_eq:solveIta} the resulting convex constraints.
Now scaled projection and gradient projection methods can be applied.
For rectifying and normalizing constraints, 
also Rosen's \cite{Rosen:61} and
Haug \& Arora's \cite{Haug:79}  
gradient projection method ensures a decrease of the E-step objective
since they can be applied to non-convex problems.

We show that the requirements as given in 
Section~\ref{sec:GAM} for GAM convergence according to
Theorem~\ref{th:GAMconvergence} (Proposition~5 in
\cite{Gunawardana:05}) are fulfilled:
\begin{enumerate}
\item
the learning rules, that is, the E-step and the M-step, are closed
maps $\longrightarrow$ ensured by continuous and continuous
differentiable maps,
\item
the parameter set is compact $\longrightarrow$ ensured by bounding
$\BPs$ and $\BW$,
\item
the family of
variational distributions is compact (often described by the feasible
set of parameters of the variational distributions) $\longrightarrow$
ensured by continuous and continuous
differentiable functions for the constraints and by the bounds on the
variational parameters
$\Bmu$ and $\BSi$ determined by bounds on the parameters and the data,
\item
the support of the density
models does not depend on the parameter $\longrightarrow$ ensured by
Gaussian models with full-rank covariance matrix,
\item
the density models are continuous in the parameters $\longrightarrow$ ensured by
Gaussian models
\item
the E-step has a
unique maximizer $\longrightarrow$ ensured by the convex, 
continuous, and continuous
differentiable function that is
minimized \cite{Dredze:08,Dredze:12}
together with compact feasible set for the variational parameters,
the maximum may be local for non-convex feasible sets stemming from normalization,
\item
the E-step increases the objective 
if not at the maximizer $\longrightarrow$ ensured as shown above,
\item
the M-step has a
unique maximizer (this is not required) $\longrightarrow$ ensured by minimizing a convex,
continuous and continuous differentiable 
function in the model parameter and a convex feasible set,
the maximum is a global maximum,
\item
the M-step increases the objective 
if not at the maximizer $\longrightarrow$ ensured as shown above.
\end{enumerate}
\end{proof}
Since this Proposition~5 in \cite{Gunawardana:05} is based on Zangwill's
generalized convergence theorem, updates of the RFN algorithm
are viewed as point-to-set mappings \cite{Zangwill:69}.
Therefore the numerical precision, the choice of the methods in the E-step,
and GPU implementations are covered by the proof.
That the M-step has a
unique maximizer is not required to proof 
Theorem~\ref{th:RFNconvergence} by Theorem~\ref{th:GAMconvergence}.
However we obtain an alternative proof by exchanging the variational
distribution $Q$ and the parameters $(\BW,\BPs)$, that is, exchanging
the E-step and the M-step. A theorem analog to 
Theorem~\ref{th:GAMconvergence} but with E-step and M-step conditions
exchanged can be derived from Zangwill's 
generalized convergence theorem \cite{Zangwill:69}.

The resulting model from the GAM procedure 
is at a local maximum of
the objective given the model family and the family of 
variational distributions.
{\em The solution minimizes
the KL-distance between the family of full 
variational distributions and full model family}.
``Full'' means that both the observed and the hidden variables are
taken into account, where for the variational distributions the
probability of the observations is set to 1.
The {\em desired family} is defined as the set of 
all probability distributions that assign probability
one to the observation.
In our case the family of 
variational distributions is not the desired family since some
distributions are excluded by the constraints.
Therefore the solution of the GAM optimization 
does not guarantee stationary points 
in likelihood \cite{Gunawardana:05}. 
This means that we do not maximize the likelihood but minimize
\begin{align}
- \ \Fcal \ \approx \ D_{\mathrm{KL}}(Q(\Bh, \Bv) \parallel p(\Bh ,
\Bv) ) \ + \ c
\end{align}
according to Eq.~\eqref{S_eq:objectiveDerive}, where $c$ is a constant
independent of $Q$ and independent of the model parameters.

\section{Correctness Proofs for the RFN Learning Algorithms}
\label{sec:correct}

The RFN algorithm is correct if it has a low reconstruction error and
explains the data covariance matrix by its parameters like factor analysis.
We show in Theorem~\ref{S_th:fixedPointDiagnonal} and 
Theorem~\ref{th:fixedPointFull} that the RFN algorithm 
\begin{enumerate}
\item
minimizes the reconstruction error given $\Bmu_i$
and $\BSi$ (the error is quadratic in $\BPs$);
\item
explains the covariance matrix by its parameters $\BW$ and $\BPs$ plus
an estimate of the second moment of the coding units $\BS$.
\end{enumerate}
Since the minimization of the reconstruction error is
based on $\Bmu_i$, the quality of reconstruction and
covariance explanation 
depends on the correlation between $\Bmu_i$ and
$\Bv_i$.
The larger the correlation between $\Bmu_i$ and
$\Bv_i$, the lower the reconstruction error and 
the better the explanation of the data covariance. 
We ensure maximal information in $\Bmu_i$
on $\Bv_i$ by the I-projection (the minimal Kullback-Leibler distance)
of the posterior onto the family of rectified and
normalized Gaussian distributions.

The reconstruction error 
for given mean values $\Bmu_i$ is
\begin{align}
 \frac{1}{n} \ \sum_{i=1}^n \left\| \Bep_i \right\|_2^2 \ ,
\end{align}
where
\begin{align}
\Bep_i \ &= \   \Bv_i \ - \ \BW \ \Bmu_i \ .
\end{align}
The reconstruction error for using the whole variational distribution $Q(\Bh_i)$
instead of its means is $\BPs$.
Below we will derive  Eq.~\eqref{S_eq:psiA}, which is
\begin{align}
\BPs \ &= \ \diag \left(\frac{1}{n} \ \sum_{i=1}^n \Bep_i \ \Bep_i^T \ + \ \BW \
  \BSi \ \BW^T \right) \ .
\end{align}
Therefore $\BPs$ is the reconstruction error for given mean values plus the
variance  $\BW \BSi \BW^T$ introduced by the hidden variables.

\subsection{Diagonal Noise Covariance Update}
\label{sec:diagUpdata}

\begin{theorem}[RFN Correctness: Diagonal Noise Covariance Update]
\label{S_th:fixedPointDiagnonal}
The fixed point $\BW$ minimizes $\TR \left( \BPs
\right)$ given $\Bmu_i$ and 
$\BSi$ by ridge regression with
\begin{align}
 \TR \left( \BPs
\right) \ &= \ \frac{1}{n} \ \sum_{i=1}^n 
\left\| \Bep_i \right\|_2^2 \ + \ 
\left\| \BW \ \BSi^{1/2} 
\right\|_{\mathrm{F}}^2 \ ,
\end{align}
where we used the error
\begin{align}
\Bep_i \ &= \   \Bv_i \ - \ \BW \ \Bmu_i
\end{align}
The model explains the data covariance matrix by
\begin{align}
\label{S_eq:covApprox}
\BC \ = \ \BPs   \ + \   \BW \ \BS  \ \BW^T
\end{align}
up to an error, which is quadratic in $\BPs$
for  $\BPs  \ll  \BW \BW^T$.
The reconstruction error 
\begin{align}
&\frac{1}{n}  \sum_{i=1}^n \left\| \Bep_i \right\|_2^2
\end{align}
is quadratic in $\BPs$ 
for  $\BPs  \ll  \BW \BW^T$.
\end{theorem}

\begin{proof}
The fixed point equation for the $\BW$ update is
\begin{align}
\label{S_eq:fix1}
&\Delta \BW \ = \  \BU \ \BS^{-1} \ - \  \BW \ = \ \BZe \ \
\Rightarrow \ \ \BW \ = \ \BU \ \BS^{-1} \ .
\end{align}
Using the
definition of $\BU$ and $\BS$,
the fixed point equation Eq.~\eqref{S_eq:fix1} gives 
\begin{align}
\BW \ = \  \left(\frac{1}{n} \ \sum_{i=1}^n \Bv_i  \  
\Bmu_i^T \right)
\  \left(\frac{1}{n} \ \sum_{i=1}^n
\Bmu_i \ \Bmu_i^T \ + \ \BSi \right)^{-1}
\end{align}
Therefore $\BW$ is a {\em ridge regression} estimate, also called
{\em generalized Tikhonov regularization} estimate, which
minimizes
\begin{align}
& \frac{1}{n} \ \sum_{i=1}^n 
\left\| \Bv_i \ - \ \BW \ \Bmu_i \right\|_2^2 \ + \ 
\left\| \BW \ \BSi^{1/2} 
\right\|_{\mathrm{F}}^2 \\\nonumber 
&= \
\frac{1}{n} \ \sum_{i=1}^n 
\left\| \Bep_i \right\|_2^2 \ + \ 
\left\| \BW \ \BSi^{1/2} 
\right\|_{\mathrm{F}}^2 \\\nonumber 
&= \
\frac{1}{n} \ \sum_{i=1}^n  \Bep_i^T \ \Bep_i \ + \ 
\TR \left( \BW \ \BSi^{1/2}  \ \BSi^{1/2} \BW^T \right) \\\nonumber 
&= \ 
\TR \left(\frac{1}{n} \ \sum_{i=1}^n \Bep_i \ \Bep_i^T \ + \ \BW \
  \BSi \ \BW^T \right) \ ,
\end{align}
where we used the reconstruction error
\begin{align}
\Bep_i \ &= \   \Bv_i \ - \ \BW \ \Bmu_i \ .
\end{align}

We obtain with this definition of the error
\begin{align}
\label{S_eq:conv1}
&\frac{1}{n} \ \sum_{i=1}^n \Bep_i \ \Bep_i^T \ + \ \BW \
  \BSi \ \BW^T \\\nonumber 
&= \
\frac{1}{n} \ \sum_{i=1}^n \Bv_i \ \Bv_i^T \ - \
 \frac{1}{n} \ \sum_{i=1}^n \Bv_i \ \Bmu_i^T \ \BW^T
 \ - \  \frac{1}{n} \ \sum_{i=1}^n \BW \ \Bmu_i \
 \Bv_i^T \\\nonumber 
& \ \ \  + \  \frac{1}{n} \ \sum_{i=1}^n \BW \ \Bmu_i
\ \Bmu_i^T \ \BW^T \ + \ \BW \
  \BSi \ \BW^T  \\\nonumber
&= \  \BC \ - \ \BU \ \BW^T \ - \ \BW \ \BU^T  \ + \   \BW \
\BS  \ \BW^T \ .
\end{align}

Therefore from the fixed point equation for  
$\BPs$ with the diagonal update rule follows
\begin{align}
\label{S_eq:psiA}
& \BPs \ = \ \diag \left(\frac{1}{n} \ \sum_{i=1}^n \Bep_i \ \Bep_i^T \ + \ \BW \
  \BSi \ \BW^T \right) \ ,
\end{align}
where ``$\diag$''
projects a matrix to a diagonal matrix. From this follows that
\begin{align}
\label{S_eq:psi1}
& \TR \left( \BPs \right) \ = \ \TR \left(\frac{1}{n} \ \sum_{i=1}^n \Bep_i \ \Bep_i^T \ + \ \BW \
  \BSi \ \BW^T \right) \ .
\end{align}
Consequently, the fixed point $\BW$ minimizes  $\TR \left( \BPs
\right)$ given $\Bmu_i$ and $\BSi$.

After convergence of the algorithm
$\BSi= \left(\BI  +   \BW^T \BPs^{-1} \BW\right)^{-1}$ holds. 
The Woodbury identity (matrix inversion lemma) states
\begin{align}
& \left( \BW \ \BW^T \ + \ \BPs \right)^{-1} \ = \ \BPs^{-1} \ - \
\BPs^{-1} \BW \left(\BI \ + \  \BW^T \BPs^{-1} \BW\right)^{-1}
\BW^T \BPs^{-1} 
\end{align}
from which follows by multiplying the equation from right and left by $\BPs$ that
\begin{align}
\label{S_eq:psi2}
\BW \ \BSi \ \BW^T \ &= \  \BW \ \left(\BI \ + \  
\BW^T \BPs^{-1} \BW\right)^{-1}
\ \BW^T \\\nonumber
& = \ \BPs \ - \ \BPs \ \left( \BW \ \BW^T \ + \ \BPs \right)^{-1}
\ \BPs
\end{align}

Inserting this equation Eq.~\eqref{S_eq:psi2} into Eq.~\eqref{S_eq:psiA} gives
\begin{align}
\label{S_eq:diagEq}
& \BPs \ = \ \diag \left(\frac{1}{n} \ \sum_{i=1}^n \Bep_i \ \Bep_i^T \ + \ \BPs \ - \ \BPs \ \left( \BW \ \BW^T \ + \ \BPs \right)^{-1}
\ \BPs \right) \\\nonumber
& = \  \BPs \ + \ \diag \left(\frac{1}{n} \ \sum_{i=1}^n \Bep_i \ \Bep_i^T \ - \ \BPs \ \left( \BW \ \BW^T \ + \ \BPs \right)^{-1}
\ \BPs \right) \ .
\end{align}
Therefore we have
\begin{align}
\label{S_eq:diagZero}
& \diag \left(\frac{1}{n} \ \sum_{i=1}^n \Bep_i \ \Bep_i^T \ - \ \BPs
  \ \left( \BW \ \BW^T \ + \ \BPs \right)^{-1} \ \BPs \right) \ = \
\BZe \ .
\end{align}
It follows that
\begin{align}\nonumber
& \TR \left(\frac{1}{n} \ \sum_{i=1}^n \Bep_i \
\Bep_i^T \right)  \ = \ \TR \left( \BPs \ \left( \BW \ \BW^T \ + \ \BPs \right)^{-1}
\ \BPs \right)  \\ \label{S_eq:bpEps}
&\leq \ \TR \left(\left( \BW \ \BW^T \ + \ \BPs \right)^{-1}
\right) \ \TR \left( \BPs \right)^2 \ .
\end{align}
The inequality uses the fact that for positive definite matrices $\BA$ and
$\BB$ inequality $\TR(\BA \BB) \leq \TR(\BA) \TR(\BB)$ holds \cite{Patel:79}.
Thus, for $\BPs  \ll  \BW \BW^T$ the error $\TR \left(\frac{1}{n}  \sum_{i=1}^n \Bep_i 
\Bep_i^T \right)=\frac{1}{n}  \sum_{i=1}^n \Bep_i^T 
\Bep_i$ is quadratic in $\BPs$.

Multiplying the fixed point equation Eq.~\eqref{S_eq:fix1}
by $\BS$ gives $\BU=\BW \BS$. Therefore we have:
\begin{align}
\label{S_eq:equalB}
\BW \ \BU^T \ = \ \BW \ \BS \ \BW^T \ = \ \BU  \ \BW^T \ .
\end{align}
Inserting Eq.~\eqref{S_eq:psi2} into the first line of Eq.~\eqref{S_eq:conv1} 
and Eq.~\eqref{S_eq:equalB} for simplifying 
the last line of Eq.~\eqref{S_eq:conv1} gives
\begin{align}
\frac{1}{n} \ \sum_{i=1}^n \Bep_i \ \Bep_i^T \ - \ 
\BPs \ \left( \BW \ \BW^T \ + \ \BPs \right)^{-1} \ \BPs \ = \
 \BC \ - \ \BPs   \ - \   \BW \ \BS  \ \BW^T \ .
\end{align} 
Using the trace norm (nuclear
norm or Ky-Fan n-norm) on matrices, Eq.~\eqref{S_eq:bpEps} states that the left hand
side is quadratic in  $\BPs$ for $\BPs  \ll \BW \BW^T$.
The trace norm of a positive
semi-definite matrix is its trace and bounds the Frobenius norm \cite{Srebro:04}.
Furthermore, Eq.~\eqref{S_eq:diagZero} states that the left hand side of
this equation has zero diagonal entries.
Therfore it follows that
\begin{align}
 \BC \ = \ \BPs   \ + \   \BW \ \BS  \ \BW^T 
\end{align}
holds except an error, which is quadratic in $\BPs$ 
for  $\BPs \ll \BW \BW^T$. The
diagonal is exactly modeled according to Eq.~\eqref{S_eq:diagZero}.
\end{proof}

Therefore the model corresponding to the fixed point explains the
empirical matrix of second
moments $\BC$ by a noise part $\BPs$ and a signal 
part $\BW \BS \BW^T$.
Like factor analysis the data variance is explained by the model via
the parameters  $\BPs$ (noise) and $\BW$ (signal).

\subsection{Full Noise Covariance Update}
\label{sec:fullUpdate}

\begin{theorem}[RFN Correctness: Full Noise Covariance Update]
\label{th:fixedPointFull}
The fixed point $\BW$ minimizes $\TR \left( \BPs
\right)$ given $\Bmu_i$ and $\BSi$ by ridge regression with
\begin{align}
 \TR \left( \BPs
\right) \ &= \ \frac{1}{n} \ \sum_{i=1}^n 
\left\| \Bep_i \right\|_2^2 \ + \ 
\left\| \BW \ \BSi^{1/2} 
\right\|_{\mathrm{F}}^2 \ ,
\end{align}
where we used the error
\begin{align}
\Bep_i \ &= \   \Bv_i \ - \ \BW \ \Bmu_i
\end{align}
The model explains the data covariance matrix by
\begin{align}
 \BC \ &= \ \BPs   \ + \   \BW \ \BS  \ \BW^T \ .
\end{align}
The reconstruction error 
\begin{align}
&\frac{1}{n}  \sum_{i=1}^n \left\| \Bep_i \right\|_2^2 
\end{align}
is quadratic in $\BPs$ 
for  $\BPs  \ll \BW \BW^T$.
\end{theorem}

\begin{proof}
The first part follows from previous Theorem~\ref{S_th:fixedPointDiagnonal}.
The fixed point equation for the $\BPs$ update is
\begin{align}
\BPs \ &= \  \BC \ - \ \BU \ \BW^T \ - \ \BW \ \BU^T  \ + \   \BW \
\BS  \ \BW^T \ ,
\end{align}
using Eq.~\eqref{S_eq:equalB} this leads to
\begin{align}
 \BC \ &=  \ \BPs \ + \  \BW \ \BS \ \BW^T \ .
\end{align}

From Eq.~\eqref{S_eq:conv1} follows for the fixed point of $\BPs$ with
the full update rule:
\begin{align}
\label{S_eq:psiA1}
\BPs \ &= \ \frac{1}{n} \ \sum_{i=1}^n \Bep_i \ \Bep_i^T \ + \ \BW \
  \BSi \ \BW^T  \ .
\end{align}
Inserting Eq.~\eqref{S_eq:psi2} into Eq.~\eqref{S_eq:psiA1} gives 
\begin{align}
\BPs \ &= \ \frac{1}{n} \ \sum_{i=1}^n \Bep_i \
  \Bep_i^T \ + \ 
 \BPs \ - \ \BPs \ \left( \BW \ \BW^T \ + \ \BPs \right)^{-1}
\ \BPs \ ,
\end{align}
from which follows
\begin{align}
&\frac{1}{n} \ \sum_{i=1}^n \Bep_i \ \Bep_i^T \ = \  \BPs \ \left( \BW \ \BW^T \ + \ \BPs \right)^{-1} \ \BPs  \ .
\end{align}
Thus, the error  $\TR \left(\frac{1}{n}  \sum_{i=1}^n \Bep_i 
\Bep_i^T \right)=\frac{1}{n}  \sum_{i=1}^n \Bep_i^T 
\Bep_i $ is quadratic in $\BPs$, for  $\BPs \ll \BW \BW^T$.
\end{proof}

\section{Maximum Likelihood Factor Analysis}
\label{sec:mlfa}

We are given the data $\{\Bv\}  = \{\Bv_1,\ldots,\Bv_n\}$ which
is assumed to be centered. Centering can be done by subtracting the mean $\Bmu$
from the data. The model is
\begin{align}
\label{S_eq:famodel} \Bv \ &= \ \BW \Bh \ + \ \Bep \ ,
\end{align}
where
\begin{align}
\Bh \ &\sim \ \Ncal\left(\BZe,\BI\right) \quad \mbox{and} \quad
\Bep \ \sim \ \Ncal\left(\BZe,\BPs \right) \ .
\end{align}
The model includes the {\em observations} $\Bv \in \bbR^m$,
the {\em noise} $\Bep \in \bbR^m$,
the {\em factors} $\Bh \in \bbR^l$,
the {\em factor loading matrix} $\BW \in \bbR^{m \times l}$,
and the {\em noise covariance matrix} $\BPs \in \bbR^{m \times m}$.
Typically we assume that $\BPs$ is a diagonal matrix
to explain data covariance by signal and not by noise.
The data variance is explained through a signal part $\BW \Bh$ and through
a noise part $\Bep$.
The parameters of the model are $\BW$ and $\BPs$.
From the model assumption it follows that
if $\Bh$ is given, then only the noise $\Bep$ is a random variable and
we have
\begin{align}
\label{S_eq:fa_x_given_z}
\Bv \mid \Bh \ &\sim \ \Ncal\left( \BW \Bh , \BPs\right) \ .
\end{align}

We want to derive the {\em likelihood} of the data under the model,
that is, the likelihood that the model has produced the data. 
Let $\EXP$ denote 
the expectation of the data including the prior distribution of the factors and the
noise distribution. We obtain for the first two moments and the variance:
\begin{align}
\EXP(\Bv) \ &= \ \EXP( \BW \Bh \ + \ \Bep) \ = \  \BW \EXP(\Bh) \
+ \ \EXP(\Bep) \ = \ \BZe \ , \\\nonumber
\EXP \left(\Bv \ \Bv^T
\right) \ &= \ \EXP\left( (\BW \Bh \ + \ \Bep) (\BW \Bh \ + \
\Bep)^T \right) \ =
\\\nonumber    & \ \ \BW \EXP\left(\Bh \ \Bh^T\right)  \BW^T \ + \  \BW
\EXP\left(\Bh\right) \EXP\left(\Bep^T\right) \\\nonumber
&+ \
\EXP\left(\Bep\right) \ \EXP\left(\Bh^T\right) \ \BW^T \ + \ \EXP\left(\Bep \ \Bep^T \right)
\ =
\\\nonumber & \ \ \BW  \ \BW^T \ + \ \BPs \\
\mathrm{var}(\Bv) \  &= \
\EXP \left(\Bv \ \Bv^T
\right) \ - \  \left( \EXP(\Bv)\right)^2 \ =  \  \BW  \ \BW^T \ + \ \BPs \ .
\end{align}

The observations are Gaussian distributed since their distribution is the product of
two Gaussian densities divided by a normalizing constant.
Therefore, the marginal distribution for $\Bv$ is
\begin{align}
\label{S_eq:fa_x}
\Bv \ &\sim \ \Ncal\left(\BZe \ , \ \BW \BW^T
\ + \ \BPs\right) \ .
\end{align}

The $\log$-likelihood $\log \prod_{i=1}^{n} p(\Bv_i)$
of the data $\{\Bv\}$ under the model $(\BW ,
\BPs)$ is
\begin{align}
\label{S_eq:fa_likelihood}
&\log \prod_{i=1}^{n} p(\Bv_i) \ = \
\log \ \prod_{i=1}^{n} \left( 2 \pi
\right)^{-m/2} \left|  \BW \BW^T  \ + \ \BPs \right|^{-1/2}
\\\nonumber &\exp \left( - \frac{1}{2} \left( \Bv_i^T \left(
\BW \BW^T  \ + \ \BPs\right)^{-1} \Bv_i \right) \right) \\ \nonumber
&= \ - \ \frac{n \ m }{2} \ \log \left( 2 \pi \right) 
 - \ \frac{n}{2} \ \log \left|  \BW \BW^T  \ + \ \BPs \right|
\\\nonumber &- \frac{1}{2} \sum_{i=1}^{n}  \Bv_i^T \left(
\BW \BW^T  \ + \ \BPs\right)^{-1} \Bv_i  \ ,
\end{align}
where $|.|$ denotes the absolute value of the determinant of a matrix.

To maximize the likelihood is difficult since a closed
form for the maximum does not exists.
Therefore, typically the expectation maximization (EM) algorithm is
used to maximize the likelihood. For the EM algorithm a variational
distribution $Q$ is required which estimates the factors given the
observations.

We consider a single data vector $\Bv_i$.
The posterior is also Gaussian with mean $(\Bmu_p)_i$ and
covariance matrix $\BSi_p$:
\begin{align}
\label{S_eq:orgsig}
\Bh_i \mid \Bv_i  \ &\sim \ \Ncal\left(
(\Bmu_p)_i ,  \BSi_p \right) \\\nonumber 
(\Bmu_p)_i \ &= \  \BW^T 
\left( \BW \ \BW^T \ + \ \BPs \right)^{-1} \Bv_i \\\nonumber
\BSi_p \ &= \ \BI \ - \  \BW^T \ \left( \BW \
\BW^T \ + \ \BPs \right)^{-1} \BW \ ,
\end{align}
where we used the fact that
\begin{align}
\Ba  \ &\sim \ \Ncal\left( \Bmu_a ,  \Sigma_{aa} \right) \ ,
\ \Bu \ \sim \ \Ncal\left( \Bmu_u ,  \Sigma_{uu}
  \right) \ , \\\nonumber
\Sigma_{ua} \ &= \ \COV(\Bu,\Ba) \mbox{
  and } \ \ \Sigma_{au} \ = \ \COV(\Ba,\Bu): \\\nonumber
\Ba \mid \Bu  \ &\sim \  \Ncal\left( \Bmu_a \ + \ \Sigma_{au}
  \Sigma_{uu}^{-1} \left( \Bu \ - \ \Bmu_u \right) \ , \
\Sigma_{aa} \ - \  \Sigma_{au}  \Sigma_{uu}^{-1} \Sigma_{ua} \right)
\end{align}
and
\begin{align}
\EXP (\Bh \Bv) \ &= \ \BW \ \EXP (\Bh \ \Bh^T) \ = \ \BW \ .
\end{align}
The EM algorithm sets $Q$ to the posterior distribution for data vector $\Bv_i$:
\begin{align}
Q_i(\Bh_i) \ &= \ p \left( \Bh_i \mid \Bv_i; \BW,\BPs \right) \ = \ \Ncal\left(
(\Bmu_p)_i ,  \BSi_p \right) \ ,
\end{align}
therefore we obtain for standared EM
\begin{align}
\Bmu_i  \ &= \ (\Bmu_q)_i  \ = \ (\Bmu_p)_i  \\
\BSi \ &= \ \BSi_q \ = \ \BSi_p \ . 
\end{align}

The matrix
inversion lemma (Woodbury identiy) can be used to compute 
$\Bmu_i $ and $\BSi$:
\begin{align}
&\left( \BW \ \BW^T \ + \ \BPs \right)^{-1} \ = \ \BPs^{-1} \ - \
\BPs^{-1} \BW \left(\BI \ + \  \BW^T \BPs^{-1} \BW\right)^{-1}
\BW^T \BPs^{-1} \ .
\end{align}
Using this identity,
the mean and the covariance matrix
can be computed as:
\begin{align}
\label{S_eq:newsig}
\Bmu_i \ &= \  \BW^T \
\left( \BW \ \BW^T \ + \ \BPs \right)^{-1} \ \Bv_i \ = \ \left( \BI \
  + \ \BW^T \BPs^{-1}\BW \right)^{-1} \ \BW^T
\BPs^{-1} \ \Bv_i  \ , \\\nonumber
\BSi \ &= \ \BI \ - \  \BW^T \ \left( \BW \
\BW^T \ + \ \BPs \right)^{-1} \BW \ = \ \left( \BI \ + \ \BW^T
\BPs^{-1}\BW \right)^{-1} \ .
\end{align}

The EM algorithm 
maximizes a lower bound $\Fcal$ on the $\log$-likelihood:
\begin{align}
\label{S_eq:objective1}
\Fcal \ &= \ \log p(\Bv_i) \ - \ D_{\mathrm{KL}}(Q(\Bh_i) \parallel p(\Bh_i \mid \Bv_i) ) \\ \nonumber 
&= \ \int Q(\Bh_i) \ \log p(\Bv_i) \ d\Bh_i \ - \ \int Q(\Bh_i) \ \log
\frac{Q(\Bh_i)}{p(\Bh_i \mid \Bv_i)}  \ d\Bh_i \\\nonumber
&= \   - \ \int Q(\Bh_i) \ \log \frac{Q(\Bh_i)}{p(\Bh_i , \Bv_i)} \ d\Bh_i \\\nonumber
&= \ - \ \int Q(\Bh_i) \ \log \frac{Q(\Bh_i)}{p(\Bh_i)} \ d\Bh_i \
+ \ \int Q(\Bh_i) \ \log p(\Bv_i \mid \Bh_i) \ d\Bh_i \\\nonumber
&= \ \int Q(\Bh_i) \ \log p(\Bv_i \mid \Bh_i) \ d\Bh_i \ - \ D_{\mathrm{KL}}(Q(\Bh_i)
\parallel p(\Bh_i) )  \ .
\end{align}
$D_{\mathrm{KL}}$ denotes the Kullback-Leibler (KL) divergence
\cite{Kullback:51} which is larger than zero. 

$\Fcal$ is the EM objective which has to be
maximized in order to maximize the likelihood.
The {\bf E-step} maximizes $\Fcal$ with respect to the variational
distribution $Q$, therefore the E-step minimizes
$D_{\mathrm{KL}}(Q(\Bh_i) \parallel p(\Bh_i \mid \Bv_i) )$.
After the standard unconstrained E-step,
the variational distribution is equal to the posterior, i.e.\
$Q(\Bh_i)  =  p(\Bh_i \mid \Bv_i)$. 
Therefore the KL divergence 
\begin{align}
&D_{\mathrm{KL}}(Q(\Bh_i) \parallel p(\Bh_i \mid \Bv_i)) \ = \ 0 
\end{align}
is zero, thus
$\Fcal$ is equal to the log-likelihood $\log p(\Bv_i)$ ($\Fcal=\log p(\Bv_i)$).
The {\bf M-step} maximizes $\Fcal$ with respect to the 
parameters $(\BW,\BPs)$, therefore the M-step maximizes
$\int Q(\Bh_i)  \log p(\Bv_i \mid \Bh_i)  d\Bh_i$.

We next consider again all $n$ samples $\{\Bv\}  = \{\Bv_1,\ldots,\Bv_n\}$.
The {\em expected reconstruction error} $\Ecal$ 
for these $n$ data samples is
\begin{align}
 \Ecal \ &= \  - \ \frac{1}{n} \ \sum_{i=1}^{n} \int_{\bbR^l} Q(\Bh_i) \ 
\log \left( p(\Bv_i \mid \Bh_i) \right) \  d\Bh_i \ = \ \frac{1}{n} \
\sum_{i=1}^{n} \EXP_Q \left(\log \left( p(\Bv_i \mid \Bh_i) \right)\right)
\end{align}
and objective to maximize becomes
\begin{align} 
\Fcal \ &= \ - \ \Ecal \ - \   \frac{1}{n} \ \sum_{i=1}^{n}   D_{\mathrm{KL}}(Q(\Bh_i)
\parallel p(\Bh_i) ) \ .
\end{align}

The M-step requires to minimize
$\Ecal$:
\begin{align}
\label{S_eq:recError}
\Ecal \ &= \
\frac{m}{2} \log \left( 2 \pi \right) \  + \ \frac{1}{2} \log
 \left| \BPs \right| \ + \\\nonumber
&\frac{1}{2 \ n} \sum_{i=1}^{n} \EXP_Q \left(\left(
\Bv_i \ - \  \BW \Bh_i \right)^T \BPs^{-1} \left(  \Bv_i \ - \
\BW \Bh_i\right) \right) \\
&=
\frac{m}{2} \log \left( 2 \pi \right) \  + \ \frac{1}{2} \log
 \left| \BPs \right| \ + \\\nonumber
&\frac{1}{2\ n} \sum_{i=1}^{n} \EXP_Q \left(
\Bv_i^T \BPs^{-1}\Bv_i \ - \
2 \ \Bv_i^T \BPs^{-1} \BW \Bh_i \ + \ \Bh_i^T \BW^T \BPs^{-1} \BW \Bh_i
\right)\\
&=
\frac{m}{2} \log \left( 2 \pi \right) \  + \ \frac{1}{2} \log
 \left| \BPs \right| \ + \ \frac{1}{2\ n} \sum_{i=1}^{n} 
\Bv_i^T \BPs^{-1}\Bv_i \\\nonumber
&- \ \TR  \left( \BPs^{-1} \BW \sum_{i=1}^{n} \EXP_Q
  \left(\Bh_i\right) \Bv_i^T\right)
 \ + \ \frac{1}{2} \TR  \left(\BW^T \BPs^{-1} \BW \sum_{i=1}^{n} \EXP_Q \left(\Bh_i \Bh_i^T\right)
\right) \\
&=
\frac{m}{2} \log \left( 2 \pi \right) \  + \ \frac{1}{2} \log
 \left| \BPs \right| \ + \ \frac{1}{2}  \TR \left(\BPs^{-1}  \frac{1}{n} \ \sum_{i=1}^{n} \Bv_i \Bv_i^T \right)\\\nonumber
&- \ \TR  \left( \BPs^{-1} \BW  \frac{1}{n} \ \sum_{i=1}^{n} \Bmu_i \Bv_i^T\right)
 \ + \ \frac{1}{2} \TR  \left(\BW^T \BPs^{-1} \BW  \frac{1}{n} \ \sum_{i=1}^{n} \left(
     \BSi \ + \ \Bmu_i \Bmu_i^T
\right)\right)\\
&= \frac{1}{2} \left(
m \ \log \left( 2 \pi \right) \  + \  \log
 \left| \BPs \right| \ + \   \TR \left(\BPs^{-1} \BC \right) \right.\\\nonumber
&\left. - \
2 \  \TR  \left( \BPs^{-1} \BW \BU^T \right)
 \ + \ \TR  \left(\BW^T \BPs^{-1} \BW \BS \right)\right) \ ,
\end{align}
where $\TR$ gives the trace of a matrix.

The derivatives with respect to the parameters are set to zero for the
optimal parameters:
\begin{align}
\label{S_eq:optmap}
\nabla_{\BW}  \Ecal \ &= \
 - \ \frac{1}{2 \ n} \sum_{i=1}^{n}
\BPs^{-1} \ \BW \ \EXP_Q  \left(\Bh_i \
\Bh_i^T\right)
 \ + \ \frac{1}{2 \ n} \sum_{i=1}^{n} \BPs^{-1} \ \Bv_i \ \EXP_Q^T
\left(\Bh_i\right) \ = \ \BZe \
\end{align}
and
\begin{align}
\label{S_eq:optpi}
\nabla_{\BPs}   \Ecal \ &= \
-\frac{1}{2} \BPs^{-1} \ + \\\nonumber & \ \  \frac{1}{2 \ n}
 \sum_{i=1}^{n} \EXP_Q \left(
\BPs^{-1} \left( \Bv_i \ - \  \BW \Bh_i \right) \ \left( \Bv_i
\ - \  \BW \Bh_i \right)^T \BPs^{-1} \right) \ = \ \BZe \ .
\end{align}

Solving above equations gives:
\begin{align}
\label{S_eq:optpipre} \BW^\nn \ &= \ \left( \frac{1}{n}
\sum_{i=1}^n \Bv_i \ \EXP_{\Bh_i  \mid \Bv_i}^T \left(\Bh_i\right) \right) \
\left( \frac{1}{n} \sum_{i=1}^n \EXP_Q \left(\Bh_i
\ \Bh_i^T\right) \right)^{-1}
\end{align}
and
\begin{align}
\label{S_eq:optpp}
\BPs^{\nn} \ &= \
 \frac{1}{n} \sum_{i=1}^n
 \EXP_Q \left(
\left( \Bv_i \ - \  \BW^{\nn} \Bh_i \right) \left(
\Bv_i \ - \  \BW^{\nn} \Bh_i \right)^T  \right)
 \ = \\\nonumber
&  \frac{1}{n} \sum_{i=1}^n  \Bv_i \ \Bv_i^T
\ - \  \frac{1}{n} \sum_{i=1}^n
\Bv_i \EXP_Q^T \left(  \Bh_i \right)
 \left( \BW^{\nn}\right)^T 
\ -  \\\nonumber
&  \frac{1}{n} \sum_{i=1}^n
\BW^{\nn}  \EXP_Q \left(  \Bh_i \right)
 \Bv_i^T
\ + \   \BW^{\nn} \frac{1}{n} \sum_{i=1}^n
 \EXP_Q \left(  \Bh_i \ \Bh_i^T \right)
\left( \BW^{\nn}\right)^T \ .
\end{align}

We obtain the following EM updates:
\begin{align}\label{S_eq:mfirststep}
\mbox{{\bf E-step:}}& ~ \\\nonumber
\Bmu_i \ &= \  \left(\BI \
  + \ \BW^T \BPs^{-1}\BW \right)^{-1}
 \BW^T \BPs^{-1} \ \Bv_i \ , \\\nonumber
\BSi \ &=  \ \left(  \BI \ + \ \BW^T
\BPs^{-1}\BW \right)^{-1} \ , \\\nonumber
\EXP_Q
\left(\Bh_i\right) \ &= \ \Bmu_i \\\nonumber
\EXP_Q \left(\Bh_i \ \Bh_i^T\right) \ &= \
\Bmu_i \ \Bmu_i^T \ + \ \BSi  \\\nonumber
 ~ ~ & ~ \\ \label{S_eq:mstep}
\mbox{{\bf M-step:}}& ~ \\\nonumber
\BW^\nn \ &= \
 \left( \frac{1}{n}\sum_{i=1}^n \Bv_i \ 
\EXP_{\Bh_i  \mid \Bv_i}^T \left(\Bh_i\right) \right) \
\left( \frac{1}{n} \sum_{i=1}^n \EXP_Q \left(\Bh_i
\ \Bh_i^T\right) \right)^{-1} \\
\BPs^{\nn} \ &= \ \frac{1}{n} \sum_{i=1}^n  \Bv_i \ \Bv_i^T
\ - \  \frac{1}{n} \sum_{i=1}^n
\Bv_i \EXP_Q^T \left(  \Bh_i \right)
 \left( \BW^{\nn}\right)^T 
\ -  \\\nonumber
&  \frac{1}{n} \sum_{i=1}^n
\BW^{\nn}  \EXP_Q \left(  \Bh_i \right)
 \Bv_i^T
\ + \   \BW^{\nn} \frac{1}{n} \sum_{i=1}^n
 \EXP_Q \left(  \Bh_i \ \Bh_i^T \right)
\left( \BW^{\nn}\right)^T \ .
\end{align}

The EM algorithms can be reformulated as:
\begin{align}\label{S_eq:mfirststepR}
\mbox{{\bf E-step:}}& ~ \\\nonumber
\Bmu_i \ &= \  \left(\BI \
  + \ \BW^T \BPs^{-1}\BW \right)^{-1}
 \BW^T \BPs^{-1} \ \Bv_i \ , \\\nonumber
\BSi \ &=  \ \left(  \BI \ + \ \BW^T
\BPs^{-1}\BW \right)^{-1} \ , \\\nonumber
\EXP_Q
\left(\Bh_i\right) \ &= \ \Bmu_i \\\nonumber
\EXP_Q \left(\Bh_i \ \Bh_i^T\right) \ &= \
\Bmu_i \ \Bmu_i^T \ + \ \BSi  \\\nonumber
 ~ ~ & ~ \\ \label{S_eq:mstepR}
\mbox{{\bf M-step:}}& ~ \\\nonumber
\BC \ &= \   \frac{1}{n} \sum_{i=1}^n  \Bv_i \ \Bv_i^T \\
\BU \ &= \ \frac{1}{n}\sum_{i=1}^n \Bv_i \ 
\EXP_Q^T \left(\Bh_i\right) \\
\BS  \ &= \  \frac{1}{n} \sum_{i=1}^n \EXP_Q \left(\Bh_i
\ \Bh_i^T\right) \\
\BW^\nn \ &= \ \BU \ \BS^{-1} \\
\BPs^{\nn} \ &= \ 
\BC \ - \ \BU \BW^T \ - \ \BW \BU^T \ + \ \BW \BS \BW^T \ .
\end{align}

\section{The RFN Objective}
\label{sec:objective}

{\em Our goal is to find a sparse, non-negative representation of 
the input which extracts structure from the input.} 
A sparse, non-negative representation is desired to code only 
events or objects that have caused the input. 
We assume that only few events or objects caused the input, therefore, 
we aim at sparseness.
Furthermore, we do not want to code the degree of absence 
of events or objects. As the vast majority of events and objects is
supposed to be absent, to code for their degree of absence would
introduce a high level of random fluctuations. 

We aim at extracting structures from the input, therefore
generative models are use as they explicitly model input
structures. For example factor analysis models the covariance
structure of the data.
However a generative model cannot enforce 
sparse, non-negative representation of the input.
The input representation of a generative model is 
the posterior's mean, median, or mode.
Generative models with rectified priors 
(zero probability for negative values) lead to rectified
posteriors. 
However these posteriors do not have 
sparse means (they must be positive), 
that is, they do not yield sparse codes \cite{Frey:99}. 
For example, rectified factor analysis, which rectifies
Gaussian priors and selects models using a variational Bayesian
learning procedure, does not yield posteriors with sparse means
\cite{Harva:05,Harva:07}.
A generative model with hidden units $\Bh$ and data $\Bv$
is defined by its prior $p(\Bh)$ and its likelihood
$p(\Bv \mid \Bh)$. 
The posterior $p(\Bh \mid \Bv)$ supplies the 
input representation of a model by the posterior's mean, median, or mode.
However, the posterior depends on the data $\Bv$, therefore sparseness and
non-negativity of its means cannot be guaranteed independent of the data.
Problem at coding the input by generative models 
is the data-dependency of the posterior means.

Therefore we use the {\em posterior regularization method} 
({\em posterior constraint method}) \cite{Ganchev:10,Graca:09,Graca:07}.
The posterior regularization framework
separates model characteristics from data dependent characteristics
like the likelihood or posterior constraints.
Posterior regularization incorporates data-dependent characteristics
as constraints on model posteriors given the observed data, 
which are difficult to encode via model 
parameters by Bayesian priors.

A generative model with prior $p(\Bh)$ and likelihood
$p(\Bv \mid \Bh)$ has the 
full model distribution $p(\Bh , \Bv)= p(\Bv \mid \Bh) p(\Bh)$.
It can be written as $p(\Bh , \Bv)= p(\Bh \mid \Bv)
p(\Bv)$, where $p(\Bh \mid \Bv)$ is the model posterior of the hidden
variables and $p(\Bv)$ is the evidence, that is, the likelihood of the
data to be produced by the model.
The model family 
and its parametrization determines which structures are extracted from
the data. Typically the model parameters enter 
the likelihood $p(\Bv \mid \Bh)$ and are adjusted to the observed data. 
For the posterior regularization method, a family $\Qcal$ of allowed
posterior distributions is introduced. $\Qcal$ is defined by the
expectations of constraint features. In our case the posterior means
have to be non-negative. Distributions $Q \in \Qcal$ are called 
{\em variational distributions} (see later for using this term).
The full variational distribution 
is $Q(\Bh, \Bv)=  Q(\Bh \mid \Bv) p_v(\Bv)$ with $Q(\Bh \mid \Bv) \in \Qcal$.
The distribution $p_v(\Bv)$ is the unknown distribution of
observations as determined by the world or the data generation
process. This distribution is
approximated by samples drawn from the world, namely the training
samples. 
{\em $p(\Bh , \Bv)$ contains all model assumptions like the
  structures used to model the data, while $Q(\Bh, \Bv)$ contains all
  data dependent characteristics including data dependent constraints
  on the posterior.}

The goal is to achieve $Q(\Bh, \Bv)= p(\Bh , \Bv)$, to obtain
(1) a desired structure that is extracted from the data and (2) desired code properties.
However in general it is  to achieve this identity, 
therefore we
want to minimize the distance between these distributions.
We use the Kullback-Leibler (KL) divergence \cite{Kullback:51}
$D_{\mathrm{KL}}$ to measure the distance between these 
distributions. Therefore our objective
is $D_{\mathrm{KL}}(Q(\Bh, \Bv) \parallel p(\Bh , \Bv) )$.
Minimizing this KL divergence 
(1) extracts the desired structure from the data by increasing the 
likelihood, that is, $p_v(\Bv) \approx p(\Bv)$, and 
(2) enforces desired code properties 
by $Q(\Bh \mid \Bv) \approx p(\Bh \mid \Bv)$.
Thus, the code derived from $Q(\Bh \mid \Bv)$ has the desired properties and t
extracts the desired input data structures.

We now approximate the KL divergence by approximating the expectation
over $p_v(\Bv)$ by the empirical mean of samples
$\{\Bv\}  = \{\Bv_1,\ldots,\Bv_n\}$ drawn from $p_v(\Bv)$:
\begin{align}
\label{S_eq:objectiveDerive}
&D_{\mathrm{KL}}(Q(\Bh, \Bv) \parallel p(\Bh , \Bv) ) \ = \
 \int Q(\Bh, \Bv) \ \log \frac{Q(\Bh, \Bv)}{p(\Bh , \Bv)} \ d\Bh \ d\Bv \\\nonumber
&= \ \int_{V} p_v(\Bv) \ \int_{H} Q(\Bh \mid \Bv) \ 
\log \frac{Q(\Bh, \Bv)}{p(\Bh , \Bv)} \ d\Bh \ d\Bv \\\nonumber
&\approx \
\frac{1}{n} \ \sum_{i=1}^{n} \int_{H} Q(\Bh \mid \Bv_i) \ 
\log \frac{Q(\Bh, \Bv_i)}{p(\Bh , \Bv_i)} \ d\Bh \\\nonumber
&= \ \frac{1}{n} \ \sum_{i=1}^{n} \int_{H} Q(\Bh \mid \Bv_i) \ 
\log \frac{Q(\Bh \mid \Bv_i)}{p(\Bh , \Bv_i)} \ d\Bh \ + \  \frac{1}{n} \
\sum_{i=1}^{n}  \log p_v(\Bv_i)  \ .
\end{align}
The last term $\frac{1}{n} \sum_{i=1}^{n}  \log p_v(\Bv_i)$ neither
depends on $Q$ nor on the model, therefore we will neglect it.
In the following, we often abbreviate $Q(\Bh \mid \Bv_i)$ by $Q(\Bh_i)$ or write 
$Q(\Bh_i \mid \Bv_i)$, since the hidden variable is based on the
observation $\Bv_i$. Similarly we often write  $p(\Bh_i , \Bv_i)$
instead of $p(\Bh , \Bv_i)$ and even more often  $p(\Bh_i \mid \Bv_i)$
instead of $p(\Bh \mid \Bv_i)$.

We obtain the objective $\Fcal$ (to be maximized) of the
{\em posterior constraint method} \cite{Ganchev:10,Graca:09,Graca:07}:
\begin{align}
\label{S_eq:objective1a}
\Fcal \ &= \  \frac{1}{n} \ \sum_{i=1}^{n}\log p(\Bv_i) \ - \
\frac{1}{n} \ \sum_{i=1}^{n} 
D_{\mathrm{KL}}(Q(\Bh_i) \parallel p(\Bh_i \mid \Bv_i) ) \\ \nonumber 
&= \  \frac{1}{n} \ \sum_{i=1}^{n} \int Q(\Bh_i) \ \log p(\Bv_i) \
d\Bh_i \ - \  \frac{1}{n} \ \sum_{i=1}^{n} \int Q(\Bh_i) \ \log
\frac{Q(\Bh_i)}{p(\Bh_i \mid \Bv_i)}  \ d\Bh_i \\\nonumber
&= \   - \  \frac{1}{n} \ \sum_{i=1}^{n}
\int Q(\Bh_i) \ \log \frac{Q(\Bh_i)}{p(\Bh_i , \Bv_i)} \ d\Bh_i \\nonumber
&= \ - \  \frac{1}{n} \ \sum_{i=1}^{n}
\int Q(\Bh_i) \ \log \frac{Q(\Bh_i)}{p(\Bh_i)} \ d\Bh_i \
+ \  \frac{1}{n} \ \sum_{i=1}^{n}
\int Q(\Bh_i) \ \log p(\Bv_i \mid \Bh_i) \ d\Bh_i \\\nonumber
&= \  \frac{1}{n} \ \sum_{i=1}^{n}
\int Q(\Bh_i) \ \log p(\Bv_i \mid \Bh_i) \ d\Bh_i \ - \  \frac{1}{n} \
\sum_{i=1}^{n} D_{\mathrm{KL}}(Q(\Bh_i) \parallel p(\Bh_i) )  \ .
\end{align}
The first line is the negative objective of the posterior constraint method
while the third line is the negative Eq.~\eqref{S_eq:objectiveDerive} without the
term  $\frac{1}{n} \sum_{i=1}^{n}  \log p_v(\Bv_i)$.

{\bf $\Fcal$ is the objective in our framework which has to be
maximized.}
Maximizing $\Fcal$ 
(1) increases the model likelihood  $\frac{1}{n}  \sum_{i=1}^{n}\log
p(\Bv_i)$,
(2) finds a proper input representation by small
$D_{\mathrm{KL}}(Q(\Bh_i) \parallel p(\Bh_i \mid \Bv_i) )$.
Thus, the data representation (1) extracts structures from the data 
as imposed by the generative model while (2) ensuring desired code
properties via $Q  \in \Qcal$.

In the variational framework, $Q$ is the variational distribution and 
$\Fcal$ is called the negative {\em free energy} \cite{Neal:98}. 
This physical term is used 
since variational methods were introduced for quantum physics by Richard
Feynman \cite{Feynman:72}.
The hidden variables can be considered 
as the fictive causes or 
explanations of environmental fluctuations \cite{Friston:12}.

If $p(\Bh \mid \Bv) \in \Qcal$, 
then $Q(\Bh \mid \Bv) = p(\Bh \mid \Bv)$ and we obtain the 
classical EM algorithm.
The EM algorithm maximizes the lower 
bound $\Fcal$ on the $\log$-likelihood as seen at
the first line of Eq.~\eqref{S_eq:objective1a} and ensures in its E-step
$Q(\Bh \mid \Bv) = p(\Bh \mid \Bv)$.

\section{Generalized Alternating Minimization}
\label{sec:GAM}

Instead of the EM algorithm we use the
{\em Generalized Alternating Minimization (GAM)} algorithm
\cite{Gunawardana:05} to allow for gradient descent both in the M-step
and the E-step.
The representation of an input by a generative model is the vector of the mean values
of the posterior, that is, the most likely hidden variables that
produced the observed data.
We have to modify the E-step to enforce variational distributions 
which lead to sparse codes via zero values of the components of its mean vector.
Sparse codes, that is, many components of the mean vector are zero, 
are obtained by enforcing non-negative means. This rectification is
analog to rectified linear units for neural networks, which have enabled
sparse codes for neural networks. 
Therefore
the variational distributions are restricted to stem from 
a family with non-negative constraints on the means. 
To impose constraints on the posterior is known as the
{\em posterior constraint method} \cite{Ganchev:10,Graca:09,Graca:07}.
The posterior constraint method maximizes the
objective both in the E-step and the M-step.
The posterior constraint method is computationally 
infeasible for our approach, 
since we assume a large number of hidden units.
For models with many hidden units, the maximization 
in the E-step would take too much time.
The posterior constraint method does not support
fast implementations on GPUs and  
stochastic gradients, which we want to allow in order to use 
mini-batches and dropout regularization. 

Therefore we perform only one gradient descent step both in the E-step
and in the M-step. 
Unfortunately, the convergence proofs of the EM algorithm are no longer
valid.
However we show that our algorithm is
a generalized alternating minimization (GAM) method.
Gunawardana and Byrne showed that the GAM converges
\cite{Gunawardana:05} (see also \cite{Wu:83}).

The following GAM convergence Theorem~\ref{th:GAMconvergence}
is Proposition~5 in \cite{Gunawardana:05} and proves the convergence
of the GAM algorithm to a solution that minimizes $-  \Fcal$.

\begin{theorem}[GAM Convergence Theorem]
\label{th:GAMconvergence}
Let the point-to-set map $\mathrm{FB}$ the composition $\mathrm{B}
\circ \mathrm{F}$ of point-to-set maps $\mathrm{F}: \mathcal{D} \times \BTh \rightarrow
 \mathcal{D} \times \BTh$ and  $\mathrm{B}: \mathcal{D} \times \BTh \rightarrow
 \mathcal{D} \times \BTh$. Suppose that the point-to-set maps $\mathrm{F}$ and
 $\mathrm{B}$ are defined so that
\begin{enumerate}
\item[(1)] $\mathrm{F}$ and
 $\mathrm{B}$ are closed on $\mathcal{D}' \times \BTh$
\item[(2)]  $F(\mathcal{D}' \times \BTh) \subseteq  \mathcal{D} \times \BTh$
and $B(\mathcal{D}' \times \BTh) \subseteq  \mathcal{D} \times \BTh$
\end{enumerate}
Suppose also that $\mathrm{F}$ is such that all $(Q'_{X},\Bth') \in
\mathrm{F}(Q_{X} ,\Bth)$ have $\Bth' = \Bth$ and satisfy
\begin{align}\nonumber
\mbox{(GAM.F):}&\qquad   D_{\mathrm{KL}}( Q'_{X} \parallel p_{X;\Bth})
 \ \leq \  D_{\mathrm{KL}}( Q_{X} \parallel p_{X;\Bth})
\end{align}
with equality only if
\begin{align}\nonumber
\mbox{(EQ.F):}&\qquad   
Q_{X} \ = \ \arg\min_{Q''_{X} \in \mathcal{D}} D_{\mathrm{KL}}( Q''_{X} \parallel
p_{X;\Bth}) \ ,
\end{align}
with $Q_{X}$
being the unique minimizer. 
Suppose also that the point-to-set map  $\mathrm{B}$ is such that all
 $(Q'_{X},\Bth') \in
\mathrm{B}(Q_{X} ,\Bth)$ have $Q'_{X} = Q_{X}$ and satisfy
\begin{align}\nonumber
\mbox{(GAM.B):}&\qquad   D_{\mathrm{KL}}( Q_{X} \parallel p_{X;\Bth'})
 \ \leq \  D_{\mathrm{KL}}( Q_{X} \parallel p_{X;\Bth})
\end{align}
with equality only if
\begin{align}\nonumber
\mbox{(EQ.B):}&\qquad   
\Bth \ \in \ \arg\min_{\Bxi \in \BTh} D_{\mathrm{KL}}( Q_{X} \parallel
p_{X;\Bxi}) \ .
\end{align}
Then,
\begin{enumerate}
\item[(1)]
the point-to-set map $\mathrm{FB}$ is closed on $\mathcal{D}' \times \BTh$
\item[(2)] $FB(\mathcal{D}' \times \BTh) \subseteq  \mathcal{D} \times \BTh$
\end{enumerate}
and $\mathrm{FB}$ satisfies the GAM and EQ conditions of the GAM convergence
theorem, that is, Theorem~3 in \cite{Gunawardana:05}.
\end{theorem}

\begin{proof}
See Proposition~5 in \cite{Gunawardana:05}.
\end{proof}

The point-to-set mappings allow extended E-step and M-steps without
unique iterates. Therefore, Theorem~\ref{th:GAMconvergence} holds for different
implementations, different hardware, different precisions of the
algorithm under consideration.

For a GAM method to converge, 
we have to ensure that the objective increases in both
the E-step and the M-step. 
$Q$ is from a 
constrained family of variational 
distributions, while the posterior and the full distribution
(observation and hidden units) are both derived from a model family. 
The model family is a parametrized family. 
For our models (i) the support of the density
models does not depend on the parameter and (ii) the density models are
continuous in their parameters. GAM convergence requires both (i) and (ii). 
Furthermore, both the E-step and the M-step must have 
unique maximizers and they increase the objective if they are not at a
maximum point.

The learning rules, that is, the E-step and the M-step
are closed maps as they are continuous functions.
The objective for the E-step is strict convex in all its parameters
for the variational distributions,
simultaneously \cite{Dredze:08,Dredze:12}. It is quadratic for the mean
vectors on which constraints are imposed.
The objective for the M-step is convex in both parameters $\BW$ and 
$\BPs^{-1}$ (we sometimes estimate $\BPs$ instead of $\BPs^{-1}$).
The objective is quadratic in the loading matrix $\BW$.
For rectifying only, 
we guarantee unique global maximizers by convex and compact 
sets for both the family of
desired distributions and the set of possible parameters. 
For this convex optimization problem with one {\em global} maximum.
For rectifying and normalizing,  
the family of
desired distributions is not convex due to equality constraints 
introduced by the normalization.
However we can guarantee 
{\em local} unique maximizers.

Summary of the requirements for GAM convergence Theorem~\ref{th:GAMconvergence}:
\begin{enumerate}
\item
the learning rules, that is, the E-step and the M-step, are closed
maps, 
\item
the parameter set is compact,
\item
the family of
variational distributions is compact (often described by the feasible
set of parameters of the variational distributions),
\item
the support of the density
models does not depend on the parameter,
\item
the density models are continuous in the parameters,
\item
the E-step has a
unique maximizer,
\item
the E-step increases the objective 
if not at the maximizer,
\item
the M-step has a
unique maximizer (not required by Theorem~\ref{th:GAMconvergence}),
\item
the M-step increases the objective 
if not at the maximizer.
\end{enumerate}

The resulting model from the GAM procedure 
is at a local maximum of
the objective given the model family and the family of 
variational distributions.
{\em The solution minimizes
the KL-distance between the family of full 
variational distributions and full model family}.
``Full'' means that both the observed and the hidden variables are
taken into account, where for the variational distributions the
probability of the observations is set to 1.
The {\em desired family} is defined as the set of 
all probability distributions that assign probability
one to the observation.
In our case the family of 
variational distributions is not the desired family since some
distributions are excluded by the constraints.
Therefore the solution of the GAM optimization 
does not guarantee stationary points 
in likelihood \cite{Gunawardana:05}. 
This means that we do not maximize the likelihood but minimize the 
KL-distance between variational distributions and model.

\section{Gradient-based M-step}
\label{sec:gradientM}

\subsection{Gradient Ascent}
\label{sec:gradienA}

The gradients in the M-step are:
\begin{align}\nonumber
\nabla_{\BW}  \Ecal \ &= \
\frac{1}{2 \ n} \sum_{i=1}^{n} \BPs^{-1} \ \Bv_i \ \EXP_Q^T
\left(\Bh_i\right) \ - \ 
  \frac{1}{2 \ n} \sum_{i=1}^{n}
\BPs^{-1} \ \BW \ \EXP_Q  \left(\Bh_i \
\Bh_i^T\right)
\end{align}
and
\begin{align}
\label{S_eq:optpi1}
\nabla_{\BPs}   \Ecal \ &= \
-\frac{1}{2} \BPs^{-1} \ + \  \frac{1}{2 \ n}
 \sum_{i=1}^{n} \EXP_Q \left(
\BPs^{-1} \left( \Bv_i \ - \  \BW \Bh_i \right) \ \left( \Bv_i
\ - \  \BW \Bh_i \right)^T \BPs^{-1} \right) \ .
\end{align}

Alternatively, we can estimate $\BPs^{-1}$ which leads to the derivatives:
\begin{align}
\label{S_eq:optpi2}
\nabla_{\BPs^{-1}}   \Ecal \ &= \
\frac{1}{2} \BPs \ - \  \frac{1}{2 \ n}
 \sum_{i=1}^{n} \EXP_Q \left(
\left( \Bv_i \ - \  \BW \Bh_i \right) \ \left( \Bv_i
\ - \  \BW \Bh_i \right)^T \right) \ .
\end{align}

Scaling the gradients leads to:
\begin{align}
2 \ \nabla_{\BW}  \Ecal \ &= \
\BPs^{-1} \
 \frac{1}{n} \sum_{i=1}^{n} \  \Bv_i \ \EXP_Q^{T}
\left(\Bh_i\right) \ - \ 
\BPs^{-1} \ \BW \  \frac{1}{n} \sum_{i=1}^{n} \EXP_Q  \left(\Bh_i \
\Bh_i^T\right)
\end{align}
and
\begin{align}
&2 \ \nabla_{\BPs}   \Ecal \ = \\\nonumber
&  - \ \BPs^{-1} \ + \ 
\BPs^{-1} \left(
 \frac{1}{n} \sum_{i=1}^n  \Bv_i \ \Bv_i^T
\ - \  \frac{1}{n} \sum_{i=1}^n
 \Bv_i  \ \EXP_Q^{T} \left(  \Bh_i \right)
\ \BW^T \right. \\\nonumber
& \left.  - \ \frac{1}{n} \sum_{i=1}^n
\BW \ \EXP_Q \left(  \Bh_i \right) \ \Bv_i^T
\ + \   \BW \ \frac{1}{n} \sum_{i=1}^n
 \EXP_Q \left(  \Bh_i \ \Bh_i^T \right)
\BW^T
\right) \BPs^{-1}    \ .
\end{align}
or
\begin{align}
&2 \ \nabla_{\BPs^{-1}}   \Ecal \ = \\\nonumber
&\BPs \ - \ 
\left(
 \frac{1}{n} \sum_{i=1}^n  \Bv_i \ \Bv_i^T
\ - \  \frac{1}{n} \sum_{i=1}^n
 \Bv_i  \ \EXP_Q^{T} \left(  \Bh_i \right)
\ \BW^T \right. \\\nonumber
& \left.  - \ \frac{1}{n} \sum_{i=1}^n
\BW \ \EXP_Q \left(  \Bh_i \right) \ \Bv_i^T
\ + \   \BW \ \frac{1}{n} \sum_{i=1}^n
 \EXP_Q \left(  \Bh_i \ \Bh_i^T \right)
\BW^T
\right)   \ .
\end{align}

Only the sums
\begin{align}
\BU \ &= \  \frac{1}{n} \ \sum_{i=1}^n
 \Bv_i  \EXP_Q^T \left(  \Bh_i \right)
\end{align}
and
\begin{align}
\BS \ &= \  \frac{1}{n} \ \sum_{i=1}^n
 \EXP_Q \left(  \Bh_i \ \Bh_i^T \right)
\end{align}
must be computed for both gradients.

\begin{align}
\BC \ &= \ \frac{1}{n} \ \sum_{i=1}^n  \Bv_i \ \Bv_i^T
\end{align}
is the estimated covariance matrix (matrix of second moments for zero mean).

{\bf The generalized EM algorithm update rules are:}

\begin{align}
\BC \ &= \ \frac{1}{n} \ \sum_{i=1}^n  \Bv_i \ \Bv_i^T \\\nonumber
\mbox{{\bf E-step:}}& ~ \\\nonumber
\Bmu_i \ &= \  \BW^T \
\left( \BW \ \BW^T \ + \ \BPs \right)^{-1} \ \Bv_i \ = \ \left( \BI \
  + \ \BW^T \BPs^{-1}\BW \right)^{-1} \ \BW^T
\BPs^{-1} \ \Bv_i  \ , \\\nonumber
\BSi \ &= \ \BI \ - \  \BW^T \ \left( \BW \
\BW^T \ + \ \BPs \right)^{-1} \BW \ = \ \left( \BI \ + \ \BW^T
\BPs^{-1}\BW \right)^{-1} \ , \\\nonumber
\EXP_Q
\left(\Bh_i\right) \ &= \ \Bmu_i \\\nonumber
\EXP_Q \left(\Bh_i \ \Bh_i^T\right) \ &= \
\Bmu_i \ \Bmu_i^T \ + \ \BSi \\\nonumber
\BU \ &= \ \frac{1}{n} \ \sum_{i=1}^n
 \Bv_i  \EXP_Q^T \left(  \Bh_i \right) \\\nonumber
\BS \ &= \  \frac{1}{n} \ \sum_{i=1}^n
 \EXP_Q \left(  \Bh_i \ \Bh_i^T \right)\\\nonumber
 ~ ~ & ~ \\
\mbox{{\bf M-step:}}& ~ \\\nonumber
\Delta \BW \ &= \
\BPs^{-1} \ \BU \ - \ \BPs^{-1} \ \BW \ \BS \\\nonumber
\Delta \BPs \ &= \ - \ \BPs^{-1} \ + \ 
\BPs^{-1} \left(
 \BC \ - \  \BU \ \BW^T  \  - \ \BW \  \BU 
\ + \   \BW \ \BS \ \BW^T
\right) \BPs^{-1}   \ .
\end{align}

\subsection{Newton Update}
\label{sec:newton}

Instead of gradient ascent, we now consider a Newton update step.
The Newton update for finding the roots of 
$\frac{\partial f}{\partial \Bv}$ is
\begin{align}
\Bv_{n+1} \ &= \ \Bv_n \ - \ \eta \ \BH^{-1} \ \nabla_{\Bv} f(\Bv_n) \ ,
\end{align}
where $\eta$ is a small step size and 
$\BH$ is the Hessian of $f$ with respect to $\Bv$ evaluated at $\Bv_n$.
We denote the update direction by
\begin{align}
\Delta \Bv \ &= \ - \ \BH^{-1} \ \nabla_{\Bv} f(\Bv_n) \ .
\end{align}

\subsubsection{Newton Update of the Loading Matrix}
\label{sec:NewtonW}

\begin{theorem}[Newton Update for Loading Matrix]
\label{th:NewtonLoading}
The M-step objective $\Ecal$
is quadratic in $\BW$, thus convex in $\BW$.
The Newton update direction for $\BW$
in the M-step is
\begin{align}
\Delta \BW \ &= \  \BU \ \BS^{-1} \ - \  \BW  \ .
\end{align}
\end{theorem}

\begin{proof}
The M-step objective is the {\em expected reconstruction error}
$\Ecal$, which is according to Eq.~\eqref{S_eq:recError}
\begin{align}
 \Ecal \ &= \  - \  \frac{1}{n} \ \sum_{i=1}^{n} \int_{\bbR^l} Q(\Bh_i) \ 
\log \left( p(\Bv_i \mid \Bh_i) \right) \  d\Bh_i  \ = \ \frac{1}{2} \Big(
m \ \log \left( 2 \pi \right) \  + \  \log
 \left| \BPs \right| \\ \nonumber 
&+ \   \TR \left(\BPs^{-1} \BC \right) \ - \
2 \  \TR  \left( \BPs^{-1} \BW \BU^T \right)
 \ + \ \TR  \left(\BW^T \BPs^{-1} \BW \BS \right)\Big) \ ,
\end{align}
where $\TR$ gives the trace of a matrix.
This is a quadratic function in $\BW$, as stated in the theorem.

The Hessian $\BH_{\BW}$ of $(2  \Ecal)$ with respect to $\BW$ as a vector is:
\begin{align}
\BH_{\BW} \ &= \
\frac{\partial \mathrm{vec} \left(2 \ 
\nabla_{\BW}  \Ecal \right)}{\partial \mathrm{vec}(\BW)^T}\ = \
\frac{\partial \mathrm{vec}\left(- \ \BPs^{-1} \ \BU \ + \ 
\BPs^{-1} \ \BW \ \BS\right) }{\partial \mathrm{vec}(\BW)^T} \\\nonumber
&= \  \BS \ \otimes \ \BPs^{-1}   \ ,
\end{align}
where $\otimes$ is the Kronecker product of matrices.  $\BH_{\BW}$ is
positive definite, thus the problem is convex in $\BW$.
The inverse of $\BH_{\BW}$ is
\begin{align}
\BH_{\BW}^{-1} \ &= \   \BS^{-1}  \ \otimes \ \BPs \ .
\end{align}

For the product of the  inverse Hessian with the gradient we have:
\begin{align}
& \BH_{\BW}^{-1} \ \mathrm{vec} \left(- \ \BPs^{-1} \ \BU \ + \
  \BPs^{-1} \ \BW \ \BS \right)
\ = \ \mathrm{vec} \left(  \BPs \  \left(- \ \BPs^{-1} \ \BU \ + \
  \BPs^{-1} \ \BW \ \BS \right) \BS^{-1} \right) \\\nonumber
&= \mathrm{vec} \left( - \ \BU \ \BS^{-1} \ + \  \BW \right) \ .
\end{align}

If we apply a Newton update, then the update direction for $\BW$
in the M-step is
\begin{align}
\Delta \BW \ &= \  \BU \ \BS^{-1} \ - \  \BW  \ .
\end{align}
\end{proof}

This is the exact EM update if the step-size $\eta$ is 1.
Since the objective is a quadratic function in $\BW$, one Newton
update would lead to the exact solution.

\subsubsection{Newton Update of the Noise Covariance}
\label{sec:newtonP}

We define the expected approximation error by
\begin{align}
\BE \ &= \
 \BC \ - \  \BU \ \BW^T  \  - \ \BW \  \BU 
\ + \   \BW \ \BS \ \BW^T\\\nonumber
&=  \ \frac{1}{n} \
 \sum_{i=1}^{n} \EXP_Q \left(
 \left( \Bv_i \ - \  \BW \Bh_i \right) \ \left( \Bv_i
\ - \  \BW \Bh_i \right)^T \right) \ .
\end{align}

\paragraph{$\BPs$ as parameter.}

\begin{theorem}[Newton Update for Noise Covariance]
\label{th:NewtonNoise}
The Newton update direction for $\BPs$ as parameter
in the M-step is
\begin{align}
\Delta \BPs \ &= \  \BE \ - \ \BPs \ .
\end{align}
An update with $\Delta \BPs$ ($\eta=1$) leads to the minimum of the M-step
objective $\Ecal$.
\end{theorem}

\begin{proof}
The M-step objective is the {\em expected reconstruction error}
$\Ecal$, which is according to Eq.~\eqref{S_eq:recError}
\begin{align}
 \Ecal \ &= \  - \  \frac{1}{n} \ \sum_{i=1}^{n} \int_{\bbR^l} Q(\Bh_i) \ 
\log \left( p(\Bv_i \mid \Bh_i) \right) \  d\Bh_i  \ = \ \frac{1}{2} \Big(
m \ \log \left( 2 \pi \right) \  + \  \log
 \left| \BPs \right| \\ \nonumber 
&+ \   \TR \left(\BPs^{-1} \BC \right) \ - \
2 \  \TR  \left( \BPs^{-1} \BW \BU^T \right)
 \ + \ \TR  \left(\BW^T \BPs^{-1} \BW \BS \right)\Big) \ ,
\end{align}
where $\TR$ gives the trace of a matrix.

Since
\begin{align}
2 \ \nabla_{\BPs}  \Ecal \ &= \
 \BPs^{-1} \ - \ 
\BPs^{-1} \BE \BPs^{-1} \ ,
\end{align}
is 
\begin{align}
\BPs \ = \ \BE  
\end{align}
the minimum of  $\Ecal$ with respect to $\BPs$.
Therefore an update with $\Delta \BPs=\BE-\BPs$ leads to the minimum.

The Hessian $\BH_{\BPs}$ of $(2 \Ecal)$ 
with respect to $\BPs$ as a vector is:
\begin{align}
\BH_{\BPs} \ &= \
\frac{\partial \mathrm{vec} \left(2 \ 
\nabla_{\BPs}  \Ecal \right)}{\partial \mathrm{vec}(\BPs)^T}
\ = \ \frac{\partial \mathrm{vec}\left( 
\BPs^{-1} \ - \ 
\BPs^{-1} \BE \BPs^{-1} 
\right) }{\partial \mathrm{vec}(\BPs)^T} \\\nonumber
&=
- \ \BPs^{-1} \ \otimes \ \BPs^{-1} \ + \
\BPs^{-1} \ \otimes \ \left( \BPs^{-1} \BE \BPs^{-1} \right)
  \ + \
 \left( \BPs^{-1} \BE \BPs^{-1}\right) \ \otimes \ \BPs^{-1} 
\ .
\end{align}

The expected approximation error $\BE$ is a sample estimate for 
$\BPs$, therefore we have $\BPs\approx \BE$.
The Hessian may not be positive definite for some values of $\BE$,
like for small values of $\BE$.
In order to guarantee a positive definite Hessian, more precisely an
approximation to it, for minmization,
we set 
\begin{align}
\BE \ &= \ \BPs 
\end{align}
and obtain
\begin{align}
\BH_{\BPs} \ &= \  \BPs^{-1} \ \otimes \ \BPs^{-1} \ .
\end{align}
We derive an approximate Newton update that is very close to the
Newton update.

The inverse of the approximated $\BH_{\BPs}$ is
\begin{align}
\BH_{\BPs}^{-1} \ &= \   \BPs  \ \otimes \ \BPs \ .
\end{align}

For the product of the inverse Hessian with the gradient we have:
\begin{align}
& \BH_{\BPs}^{-1} \ \mathrm{vec} \left(
\BPs^{-1} \ - \ 
\BPs^{-1} \BE \BPs^{-1}  \right)
\ = \ \mathrm{vec} \left(  \BPs \ \left(
\BPs^{-1} \ - \ 
\BPs^{-1} \BE \BPs^{-1}  \right)  \BPs \right) \\\nonumber
&= \mathrm{vec} \left(  \BPs \ - \ \BE  \right) \ .
\end{align}

If we apply a Newton update, then the update direction for $\BPs$
in the M-step is
\begin{align}
\Delta \BPs \ &= \  \BE \ - \ \BPs \ .
\end{align}
This is the exact EM update if the step-size $\eta$ is 1.
\end{proof}

\paragraph{$\BPs^{-1}$ as parameter.}

\begin{theorem}[Newton Update for Inverse Noise Covariance]
\label{th:NewtonInvNoise}
The M-step objective $\Ecal$
is convex in $\BPs^{-1}$.
The Newton update direction for $\BPs^{-1}$ as parameter
in the M-step is
\begin{align}
\Delta \BPs^{-1} \ &= \   \BPs^{-1} \ - \ \BPs^{-1} \ \BE  \ \BPs^{-1}  \ .
\end{align}
A first order approximation of this Newton direction for $\BPs$ 
in the M-step is
\begin{align}
\label{S_eq:psiUdate2}
\Delta \BPs \ &= \   \BE \ - \ \BPs \ .
\end{align}
An update with $\Delta \BPs$ ($\eta=1$) leads to the minimum of the M-step
objective $\Ecal$.
\end{theorem}

\begin{proof}
The M-step objective is the {\em expected reconstruction error}
$\Ecal$, which is according to Eq.~\eqref{S_eq:recError}
\begin{align}
 \Ecal \ &= \  - \  \frac{1}{n} \ \sum_{i=1}^{n} \int_{\bbR^l} Q(\Bh_i) \ 
\log \left( p(\Bv_i \mid \Bh_i) \right) \  d\Bh_i  \ = \ \frac{1}{2} \Big(
m \ \log \left( 2 \pi \right) \  + \  \log
 \left| \BPs \right| \\ \nonumber 
&+ \   \TR \left(\BPs^{-1} \BC \right) \ - \
2 \  \TR  \left( \BPs^{-1} \BW \BU^T \right)
 \ + \ \TR  \left(\BW^T \BPs^{-1} \BW \BS \right)\Big) \ ,
\end{align}
where $\TR$ gives the trace of a matrix.

Since
\begin{align}
2 \ \nabla_{\BPs^{-1}}  \Ecal \ &= \
- \ \BPs \ + \ \BE 
\end{align}
is 
\begin{align}
\BPs \ &= \ \BE  
\end{align}
the minimum of $\Ecal$ with respect to $\BPs^{-1}$.
Therefore an update with $\Delta \BPs=\BE-\BPs$ leads to the minimum.

The Hessian $\BH_{\BPs^{-1}}$ of $(2 \Ecal)$ 
with respect to $\BPs^{-1}$ as a vector is:
\begin{align}
\BH_{\BPs^{-1}} \ &= \
\frac{\partial \mathrm{vec} \left(2 \
\nabla_{\BPs^{-1}}  \Ecal \right)}{\partial
\mathrm{vec}(\BPs^{-1})^T}\ = \ \frac{\partial \mathrm{vec}\left( 
- \ \BPs \ + \ \BE 
\right) }{\partial \mathrm{vec}(\BPs^{-1})^T} \ =
\ \BPs \ \otimes \ \BPs \ .
\end{align}
Since the Hessian is positive definite, the E-step objective $\Ecal$
is convex in $\BPs^{-1}$, which is the first statement of the theorem.

The inverse of $\BH_{\BPs^{-1}}$ is
\begin{align}
\BH_{\BPs^{-1}}^{-1} \ &= \  \BPs^{-1}  \ \otimes \ \BPs^{-1} \ .
\end{align}

For the product of the inverse Hessian with the gradient we have:
\begin{align}
& \BH_{\BPs^{-1}}^{-1} \ \mathrm{vec} \left(
- \ \BPs \ + \  \BE   \right)
\ = \ \mathrm{vec} \left(  \BPs^{-1} \ \left(
- \ \BPs \ + \ \BE  \right)  \BPs^{-1} \right) \\\nonumber
&= \mathrm{vec} \left(  - \ \BPs^{-1} \ + \ \BPs^{-1} \ \BE  \ \BPs^{-1} \right) \ .
\end{align}

If we apply a Newton update, then the update direction for $\BPs^{-1}$
in the M-step is
\begin{align}
\Delta \BPs^{-1} \ &= \   \BPs^{-1} \ - \ \BPs^{-1} \ \BE  \ \BPs^{-1}  \ .
\end{align}

We now can approximate the update for $\BPs$ by the first terms of the
Taylor expansion:
\begin{align}
\BPs \ + \ \Delta \BPs \ &= \ \left(\BPs^{-1} \ + \ \Delta \BPs^{-1}
\right)^{-1} \ \approx \ \BPs \ - \ \BPs \ \Delta \BPs^{-1} \ \BPs \ .
\end{align}
We obtain for the update of  $\BPs$
\begin{align}
\Delta \BPs \ &= \  - \ \BPs \ \Delta \BPs^{-1} \ \BPs \ = \
\BE \ - \ \BPs \ .
\end{align}
This is the exact EM update if the step-size $\eta$ is 1.
\end{proof}

The Newton update derived from $\BPs^{-1}$ as parameter is 
the Newton update for $\BPs$.
Consequently, the Newton direction for both $\BPs$ and $\BPs^{-1}$ is
in the M-step
\begin{align}
\Delta \BPs \ &= \  \BE \ - \ \BPs \ .
\end{align}

\section{Gradient-based E-Step}
\label{sec:gradientE}

\subsection{Motivation for Rectifying and Normalization Constraints}
\label{sec:rectNorm}

The representation of data vector $\Bv$ by the model is the
variational mean vector 
$\Bmu_q$. In order to obtain sparse codes we want to
have non-negative $\Bmu_q$. We enforce non-negative mean
values by constraints and optimize by
projected Newton methods and by gradient projection methods. 
Non-negative constraints correspond to rectifying in the neural network field.
Therefore we aim to construct sparse codes 
in analogy to the
rectified linear units used for neural networks.

We constrain the variational distributions to 
the family of normal distributions with non-negative mean components.
Consequently we introduce non-negative or {\bf rectifying constraints}:
\begin{align}
\label{S_eq:constRect}
 \Bmu \ &\geq \ \BZe \ ,
\end{align}
where the inequality ``$\geq$'' holds component-wise.

However generative models with many coding units face a problem.
They tend to 
{\em explain away small and rare signals by noise}. 
For many coding units, model selection algorithms prefer models with 
coding units which do not have variation and, therefore,
are removed from the model. Other coding units hardly contribute to
explain the observations. The likelihood is larger if small and rare
signals are explained by noise, than the likelihood if coding units
are use to explain such signals. Coding units without variance are
kept on their default values, where they have maximal contribution to
the likelihood. If they are used for coding, they deviate from their
maximal values for each sample. In accumulation these deviations
decrease the likelihood more than it is increased by explaining small
or rare signals.
For our RFN models the problem can become severe, 
since we aim at models with up to several tens of
thousands of coding units. 
To avoid the explaining away problem, we enforce the selected models to
use all their coding units on an equal level. 
We do that by keeping the variation of each noise-free coding unit across the
training set at one.
Consequently, we 
introduce a {\bf normalization constraint} for each coding unit $1
\leq j \leq l$:
\begin{align}
\label{S_eq:constNorm}
 \frac{1}{n} \sum_{i=1}^{n} \mu_{ij}^2 \ &= \ 1 \ .
\end{align}
This constraint means that the noise-free part of each coding unit
has variance one across samples.

We will derive methods to increase the objective in the E-step
both for only rectifying constraints 
and for rectifying and normalization constraints.
These methods ensure to reduce 
the objective in the E-step to guarantee convergence via the GAM theory.
The resulting model from the GAM procedure 
is at a local maximum of
the objective given the model family and the family of 
variational distributions.
{\em The solution minimizes 
the KL-distance between the family of full 
variational distributions and full model family}.
``Full'' means that both the observed and the hidden variables are
taken into account.

\subsection{The Full E-step Objective}
\label{sec:fullE}

The E-step maximizes $\Fcal$ with respect to the variational
distribution $Q$, therefore the E-step minimizes
the Kullback-Leibler divergence (KL-divergence)
\cite{Kullback:51} 
$D_{\mathrm{KL}}(Q(\Bh) \parallel p(\Bh \mid \Bv) )$.
The KL-divergence between $Q$ and $p$ is
\begin{align}
D_{\mathrm{KL}}(Q \parallel p) \ &= \ \int  Q(\Bh) \ \log
\frac{ Q(\Bh)}{p(\Bh \mid \Bv)} \ d\Bh \ .
\end{align}

{\em Rectifying constraints} introduce non-negative constraints.
The minimization with respect to $Q(\Bh_i)$ gives the constraint
minimization problem:
\begin{align}
\label{S_eq:EobjectiveR}
\min_{Q(\Bh_i)}  {\mbox{\ ~} } & \frac{1}{n} 
\sum_{i=1}^n D_{\mathrm{KL}}(Q(\Bh_i) \parallel p(\Bh_i \mid \Bv_i) )\\ \nonumber 
\mbox{ s.t. }  {\mbox{\ ~} } &\forall_i: \ \Bmu_i \ \geq \ \BZe \ ,
\end{align}
where $\Bmu_i$ is the mean vector of $Q(\Bh_i)$.

{\em Rectifying and normalizing constraints} introduce non-negative
constraints and equality constraints.
The minimization with respect to $Q(\Bh_i)$ gives the constraint
minimization problem:
\begin{align}
\label{S_eq:EobjectiveRN}
\min_{Q(\Bh_i)}  {\mbox{\ ~} } & \frac{1}{n} 
\sum_{i=1}^n D_{\mathrm{KL}}(Q(\Bh_i) \parallel p(\Bh_i \mid \Bv_i) )\\ \nonumber 
\mbox{ s.t. }  {\mbox{\ ~} } &\forall_i: \ \Bmu_i \ \geq \ \BZe \ , \\ \nonumber 
&\forall_j: \ \frac{1}{n}  \ \sum_{i=1}^{n} \mu_{ij}^2 \ = \ 1 \ ,
\end{align}
where $\Bmu_i$ is the mean vector of $Q(\Bh_i)$.

First we consider the families from which the model and from which the
variational distributions stem.
The posterior of the model with Gaussian prior $p(\Bh)$ is Gaussian
(see Section~\ref{sec:mlfa}):
\begin{align}
p(\Bh \mid \Bv) \ &\sim \ (2\pi)^{-\frac{l}{2}} \
\left|\BSi_p \right|^{-\frac{1}{2}} \  \exp \left( - \ \frac{1}{2}\ (\Bh
  \ - \ \Bmu_p)^T \ \BSi_p^{-1} \ (\Bh \ - \ \Bmu_p) \right) \ .
\end{align}
To be as close as possible to the posterior distribution, 
we restrict $Q$ to be from a Gaussian family:
\begin{align}
Q(\Bh) \ &\sim \ (2\pi)^{-\frac{l}{2}} \
\left|\BSi_q \right|^{-\frac{1}{2}} \  \exp \left( - \ \frac{1}{2}\ (\Bh
  \ - \ \Bmu_q)^T \ \BSi_q^{-1} \ (\Bh \ - \ \Bmu_q) \right) \ .
\end{align}

For Gaussians, the Kullback-Leibler divergence between $Q$ and $p$ is
\begin{align}
\label{S_eq:KL1}
&D_{\mathrm{KL}}(Q \parallel p) \ = \\\nonumber 
& \frac{1}{2} \left\{ \TR \left(
     \BSi_p^{-1} \ \BSi_q \right) \ + \ \left( \Bmu_p \ - \
     \Bmu_q\right)^{T} \ \BSi_p^{-1} \ \left( \Bmu_p \ - \ \Bmu_q
   \right) \ - \ l \ - \ \ln \frac{| \BSi_q |}{| \BSi_p |} \right\}
 \ . 
\end{align}
This Kullback-Leibler divergence 
is convex in the mean vector $\Bmu_q$ and the
covariance matrix $\BSi_q$ of $Q$, 
simultaneously \cite{Dredze:08,Dredze:12}.

We now minimize Eq.~\eqref{S_eq:KL1} with respect to $Q$. For the moment
we do not care about the constraints introduced by non-negativity and
by normalization.
Eq.~\eqref{S_eq:KL1} has a quadratic form in $\Bmu_q$, where $\BSi_q$
does not enter, and terms in $\BSi_q$, where $\Bmu_q$ does not enter.
Therefore we can separately minimize for $\BSi_q$ and for $\Bmu_q$.

For the minimization with respect to $\BSi_q$, we require
\begin{align}
\frac{\partial}{\partial \BSi_q} \TR \left( \BSi_p^{-1} \ \BSi_q
\right) \ &= \ \BSi_p^{-T}
\end{align}
and
\begin{align}
\frac{\partial}{\partial \BSi_q} \ln \left| \BSi_q
\right| \ &= \ \BSi_q^{-T} \ .
\end{align}
For optimality the derivative of the objective  
$D_{\mathrm{KL}}(Q \parallel p)$ with respect to $\BSi_q$ must be zero:
\begin{align}
\frac{\partial}{\partial \BSi_q} D_{\mathrm{KL}}(Q \parallel p)
 \ &= \ \frac{1}{2} \ \BSi_p^{-T} \ - \  \frac{1}{2} \  
\BSi_q^{-T} \ = \ \BZe \ .
\end{align}
This gives 
\begin{align}
\BSi \ &= \ \BSi_q \ = \ \BSi_p \ .
\end{align}
We often drop the index $q$ since for $1 \leq i \leq n$ all covariance
matrices $\BSi_q$ are equal to $\BSi_p$.

The mean vector $\Bmu_q$ of $Q$ is the solution 
of the minimization problem:
\begin{align}
\min_{\Bmu}  {\mbox{\ ~} } &\frac{1}{2} \ \left( \Bmu_p \ - \
\Bmu\right)^{T} \ \BSi_p^{-1} \ \left( \Bmu_p \ - \ \Bmu \right)
\end{align}
which is equivalent to 
\begin{align}
\min_{\Bmu}  {\mbox{\ ~} } &\frac{1}{2} \ \Bmu^T \BSi_p^{-1}
\Bmu \ - \  \Bmu_p^T \BSi_p^{-1} \Bmu \ .
\end{align}

The derivative and the Hessian of this objective is:
\begin{align}
\frac{\partial}{\partial \Bmu} D_{\mathrm{KL}}(Q \parallel p)
 \ &= \ \BSi_p^{-1} ( \Bmu \ - \  \Bmu_p) \ , \\
\frac{\partial^2}{\partial^2 \Bmu} D_{\mathrm{KL}}(Q \parallel p)
 \ &= \ \BSi_p^{-1} \ .
\end{align}

\subsection{E-step for Mean with Rectifying Constraints}
\label{sec:Erect}

\subsubsection{The E-Step Minimization Problem}
\label{sec:ErectP}

Rectifying is realized by non-negative constraints.
The mean vector $\Bmu_q$ of $Q$ is the solution of the
minimization problem:
\begin{align}
\min_{\Bmu}  {\mbox{\ ~} } &\frac{1}{2} \ (\Bmu \ - \  \Bmu_p)^T \ \BSi_p^{-1}
\ (\Bmu \ - \  \Bmu_p) \\ \nonumber 
\mbox{ s.t. }  {\mbox{\ ~} } &\Bmu \ \geq \ \BZe \ .
\end{align}
This is a convex quadratic minimization 
problem with non-negativity constraints (convex
feasible set).

If $\Bla$ is the Lagrange multiplier for the constraints, then the dual
is
\begin{align}
\min_{\Bla}  {\mbox{\ ~} } &\frac{1}{2} \ \Bla^T \BSi_p
\Bla \ + \  \Bmu_p^T \Bla  \\ \nonumber 
\mbox{ s.t. }  {\mbox{\ ~} } &\Bla \ \geq \ \BZe \ .
\end{align}
The Karush-Kuhn-Tucker conditions require for the optimal solution for
each component $1\leq j\leq l$:
\begin{align}
\lambda_j \ \mu_{j} \ &= \ 0 \ .
\end{align}
Further the derivative of the Lagrangian with respect to $\Bmu$
gives
\begin{align}
\BSi_p^{-1} \Bmu \ - \ \BSi_p^{-1} \Bmu_p \ - \ \Bla \ = \ \BZe 
\end{align}
which can be written as
\begin{align}
\label{S_eq:ZeroD1}
\Bmu \ - \ \Bmu_p \ - \ \BSi_p \ \Bla \ = \ \BZe \ .
\end{align}

This minimization problem cannot be solved directly. 
Therefore we perform a gradient projection or projected Newton step to
decrease the objective.

\subsubsection{The Projection onto the Feasible Set}
\label{sec:ErectProj}

To decrease the objective, we perform a gradient projection or
a projected Newton step. 
We will base our algorithms 
on {\em Euclidean least distance projections}.
If projected onto convex sets, these projections do not increase
distances. 
The Euclidean projection onto the feasible set is denoted by $\PP$, 
that is, the map that takes $\Bmu_p$ to its nearest point $\Bmu$ (in
the $L^2$-norm) in the feasible set.

For rectifying constraints, the projection $\PP$
(Euclidean least distance projection) of $\Bmu_p$ onto the
convex feasible set is given by the solution of the convex 
optimization problem:
\begin{align}
\label{S_eq:optmu}
\min_{\Bmu}  {\mbox{\ ~} } &\frac{1}{2} \ \left( \Bmu  \ - \
  \Bmu_p \right)^T \left( \Bmu  \ - \
  \Bmu_p \right) \\ \nonumber
\mbox{ s.t. }  {\mbox{\ ~} } &\Bmu \ \geq \ \BZe \ .
\end{align}
The
following Theorem~\ref{th:rectifying} shows that update Eq.~\eqref{S_eq:Proj1} is the
projection $\PP$ defined by optimization problem  Eq.~\eqref{S_eq:optmu}.

\begin{theorem}[Projection: Rectifying]
\label{th:rectifying}
The solution to optimization problem  Eq.~\eqref{S_eq:optmu}, which
defines the Euclidean least distance projection, is
\begin{align}
\label{S_eq:Proj1}
\mu_{j} \ &= \ \left[\PP (\Bmu_p) \right]_j \ = \  \left\{ 
\begin{array}{lcl}
0 & \mathrm{for} & (\mu_p)_j \ \leq \ 0 \\
(\mu_p)_j & \mathrm{for} &  (\mu_p)_j \ > \ 0
\end{array} \right. 
\end{align}
\end{theorem}

\begin{proof}
For the projection we have the minimization problem:
\begin{align}
\min_{\Bmu}  {\mbox{\ ~} } &\frac{1}{2} \ \left( \Bmu  \ - \
  \Bmu_p \right)^T \left( \Bmu  \ - \
  \Bmu_p \right) \\ \nonumber
\mbox{ s.t. }  {\mbox{\ ~} } &\Bmu \ \geq \ \BZe \ .
\end{align}
The Lagrangian $L$ with multiplier $\Bla \geq \BZe$ is
\begin{align}
L \ &= \ \frac{1}{2} \  \left( \Bmu  \ - \
  \Bmu_p \right)^T \left( \Bmu  \ - \
  \Bmu_p \right) \ - \ \Bla^T \ \Bmu \ .
\end{align}
The derivative with respect to  $\Bmu$ is
\begin{align}
\label{S_eq:LagFre}
\frac{\partial L}{\partial \Bmu} \ &= \ 
\Bmu\ - \ \Bmu_p \ - \ \Bla \ = \ \BZe \ . 
\end{align}
The Karush-Kuhn-Tucker (KKT) conditions require for the optimal
solution that for
each constraint $j$:
\begin{align}
\label{S_eq:KKT}
\lambda_j \ \mu_{j} \ &= \ 0 \ .
\end{align}

If $0 < (\mu_p)_{j}$ then Eq.~\eqref{S_eq:LagFre} requires $0 < \mu_{j}$
because the Lagrangian $\lambda_j$ is larger than or equal to zero:
$0 \leq \lambda_j$. From the KKT conditions Eq.~\eqref{S_eq:KKT}
follows that $\lambda_j=0$
and, therefore, $0 < \mu_{j}=(\mu_p)_{j}$.
If $(\mu_p)_{j}<0$ then $0 < \mu_{j} - (\mu_p)_{j}$,
because the constraints of the primal problem require $0\leq \mu_{j}$.
From Eq.~\eqref{S_eq:LagFre} follows that $0< \lambda_j$.
From the KKT conditions Eq.~\eqref{S_eq:KKT} follows that $(\mu_p)_{j}=0$
and $0 < \lambda_j= -(\mu_p)_{j}$. If $(\mu_p)_{j}=0$, then
Eq.~\eqref{S_eq:LagFre} and the KKT conditions Eq.~\eqref{S_eq:KKT} lead to 
$(\mu_p)_{j}=\mu_{j}=\lambda_j=0$.

Therefore the solution of problem Eq.~\eqref{S_eq:optmu} is
\begin{align}
\mu_{j} \ &= \ \left\{ 
\begin{array}{lcl}
(\mu_p)_{j} & \mathrm{for} & (\mu_p)_{j} > 0 \mbox{~~and~~} \lambda_j = 0 \\
0 & \mathrm{for} & (\mu_p)_{j} \leq 0 \mbox{~~and~~} \lambda_j = - (\mu_p)_{j} 
\end{array} \right. \ .
\end{align}

This finishes the proof.
\end{proof}

\subsection{E-step for Mean with Rectifying and Normalizing Constraints}
\label{sec:ErectNorm}

\subsubsection{The E-Step Minimization Problem}
\label{sec:ErectNormP}

If we also consider normalizing constraints, then we have to minimize
all KL-divergences simultaneously. The normalizing constraints connect
the single optimization problems for each sample $\Bv_i$. For the
E-step, we obtain the minimization problem:
\begin{align}
\label{S_eq:normP}
\min_{\Bmu_i}  {\mbox{\ ~} } &\frac{1}{ n} \ \sum_{i=1}^{n}
(\Bmu_i \ - \  (\Bmu_p)_i)^T \ \BSi_p^{-1} \ (\Bmu_i \ - \  (\Bmu_p)_i)\\ \nonumber 
\mbox{ s.t. }  {\mbox{\ ~} } &\forall_i: \ \Bmu_i \ \geq \ \BZe \quad , \quad
\forall_j: \ \frac{1}{n}  \ \sum_{i=1}^{n} \mu_{ij}^2 \ = \ 1 \ .
\end{align}
The ``$\geq$''-sign is meant component-wise.
The $l$ equality constraints lead to non-convex feasible sets.
The solution to this optimization problem are the means vectors
$\Bmu_i$ of $Q(\Bh_i)$.

\paragraph{Generalized Reduced Gradient.}

The equality constraints can be solved for one variable which is then 
inserted into the objective.
The equality constraint gives for each $1\leq j \leq l$:
\begin{align}
\label{S_eq:solveIt}
\mu_{1j}^2 \ &= \ n \ - \ \sum_{i=2}^{n}
\mu_{ij}^2  \quad \mbox{or} \quad
\mu_{1j} \ = \ \sqrt{n \ - \ \sum_{i=2}^{n}
\mu_{ij}^2} \ .
\end{align}
These equations can be inserted into the objective and, thereby, we
remove the variables $\mu_{1j}$.
We have to ensure that the $\mu_{1j}$ exist by
\begin{align}
\label{S_eq:solveIta}
\sum_{i=2}^{n} \mu_{ij}^2 \ \leq \ n  \ .
\end{align}
These constraints define a convex set feasible set.
To solve the each equality constraints for a variable and insert it into
the objective is called
{\em generalized reduced gradient} method \cite{Abadie:69}. 
For solving the reduced problem, we can use methods for constraint
optimization were we now ensure a convex feasible set. 
These methods solve the original problem Eq.~\eqref{S_eq:normP}. 
We only require an improvement of the objective with a feasible value.
For the reduced problem,
we perform one step of a {\em gradient projection method}.

\paragraph{Gradient Projection Methods.}

Also for the original problem Eq.~\eqref{S_eq:normP}, 
{\em gradient projection methods} can be used.
The gradient  projection method has been generalized by Rosen to 
{\em non-linear constraints} \cite{Rosen:61} and was later improved by
\cite{Haug:79}. 
The gradient projection algorithm 
of Rosen works for {\em non-convex feasible sets}.
The idea is to linearize the nonlinear constraints and solve the
problem. Subsequently a
restoration move brings the solution back to the constraint
boundaries.

\subsubsection{The Projection onto the Feasible Set}
\label{sec:ErectNomrProj}

To decrease the objective, we perform a gradient projection,
a projected Newton step, or a step of the generalized reduced method.
We will base our algorithms 
on {\em Euclidean least distance projections}.
If projected onto convex sets, these projections do not increase
distances. 
The Euclidean projection onto the feasible set is denoted by $\PP$, 
that is, the map that simultaneously takes $\{(\Bmu_p)_i\}$ to the nearest points
$\{\Bmu_i\}$  (in the $L^2$-norm) in the feasible set.

For rectifying and normalizing constraints the projection 
(Euclidean least distance projection)
of
$\{(\Bmu_p)_i\}$ onto the {\bf non-convex} feasible set leads to
the optimization problem
\begin{align}
\label{S_eq:optmu2}
\min_{\Bmu_i}  {\mbox{\ ~} } &\frac{1}{n} \ \sum_{i=1}^{n} \left( \Bmu_i  \ - \
  (\Bmu_p)_i \right)^T \left( \Bmu_i  \ - \
  (\Bmu_p)_i \right) \\ \nonumber
\mbox{ s.t. }  {\mbox{\ ~} } &\forall_i: \ \Bmu_i \ \geq \ \BZe \ , \\\nonumber
&\forall_j: \ \frac{1}{n} \ \sum_{i=1}^{n} \mu_{ij}^2 \ = \ 1 \ .
\end{align}
By using 
$\left( \Bmu_i  -
  (\Bmu_p)_i \right)^T \left( \Bmu_i   - 
  (\Bmu_p)_i \right)= \Bmu_i^T \Bmu_i  -  2  \Bmu_i^T (\Bmu_p)_i 
+  (\Bmu_p)_i^T (\Bmu_p)_i$,
we see that the objective contains the sum
$\sum_{ij} \mu_{ij}^2$.
The constraints enforce this sum to be constant.
Therefore inserting the equality constraints into the objective,
optimization problem Eq.~\eqref{S_eq:optmu2} is 
equivalent to 
\begin{align}
\label{S_eq:optmu2a}
\min_{\Bmu_i}  {\mbox{\ ~} } & - \ \frac{1}{n} \ \sum_{i=1}^{n}
\Bmu_i^T  \ (\Bmu_p)_i \\\nonumber
\mbox{ s.t. }  {\mbox{\ ~} } &\forall_i: \ \Bmu_i \ \geq \ \BZe \ , \\\nonumber
&\forall_j: \ \frac{1}{n} \ \sum_{i=1}^{n} \mu_{ij}^2 \ = \ 1 \ .
\end{align}

The
following Theorem~\ref{S_th:rectNorm} shows that updates
Eq.~\eqref{S_eq:Proj2} 
and Eq.~\eqref{S_eq:Proj2A} form the
projection defined by optimization problem  Eq.~\eqref{S_eq:optmu2}.

\begin{theorem}[Projection: Rectifying and Normalizing]
\label{S_th:rectNorm}
If at least one $(\mu_p)_{ij}$ is positive for $1 \leq j \leq l$, 
then the solution to optimization problem  Eq.~\eqref{S_eq:optmu2}, which
defines the Euclidean least distance projection, is
\begin{align}
\label{S_eq:Proj2}
\hat{\mu}_{ij} \ &= \  \left\{ 
\begin{array}{lcl}
0 & \mathrm{for} & (\mu_p)_{ij} \ \leq \ 0 \\
(\mu_p)_{ij} & \mathrm{for} &  (\mu_p)_{ij} \ > \ 0
\end{array} \right. \\ \nonumber
\mu_{ij} \ &= \ \left[\PP ((\Bmu_p)_i) \right]_j \ = \  
\frac{\hat{\mu}_{ij}}{\sqrt{\frac{1}{n} \ \sum_{i=1}^{n} \hat{\mu}_{ij}^2}} \ .
\end{align}
If all $(\mu_p)_{ij}$ are non-positive for $1 \leq j \leq l$,
then the optimization problem Eq.~\eqref{S_eq:optmu2} has the solution
\begin{align}
\label{S_eq:Proj2A}
\mu_{ij} \ &= \  \left\{ 
\begin{array}{lcl}
\sqrt{n} & \mathrm{for} & j \ = \  \arg \max_{\hat{j}} \{(\mu_p)_{i\hat{j}}\}\\
0 & \mathrm{otherwise} &
\end{array} \right. \ .
\end{align}
\end{theorem}

\begin{proof}
In the
following we show that updates Eq.~\eqref{S_eq:Proj2} 
and Eq.~\eqref{S_eq:Proj2} are the 
projection onto the feasible set.
For the projection of $\{(\Bmu_p)_i\}$ onto the feasible set,
we have the minimization problem:
\begin{align}
\min_{\Bmu_i}  {\mbox{\ ~} } &\frac{1}{n} \ \sum_{i=1}^{n} \left( \Bmu_i  \ - \
  (\Bmu_p)_i \right)^T \left( \Bmu_i  \ - \
  (\Bmu_p)_i \right) \\\nonumber
\mbox{ s.t. }  {\mbox{\ ~} } &\forall_i: \ \Bmu_i \ \geq \ \BZe \ , \\\nonumber
&\forall_j: \ \frac{1}{n} \ \sum_{i=1}^{n} \mu_{ij}^2 \ = \ 1 \ .
\end{align}
The feasible set is non-convex because of the quadratic equality
constraint.
The Lagrangian with multiplier $\Bla \geq \BZe$ is
\begin{align}
L \ &= \ \frac{1}{n} \  \sum_{i=1}^{n} \left( \Bmu_i  \ - \
   (\Bmu_p)_i  \right)^T \left( \Bmu_i  \ - \
   (\Bmu_p)_i  \right) \ - \  \sum_{i=1}^{n} \Bla_i^T \ 
\Bmu_i \\ \nonumber
 &+ \ 
\sum_j \tau_j \ \left(\frac{1}{n} \ \sum_{i=1}^{n} \mu_{ij}^2 \ - \ 1 \right) \ .
\end{align}
The Karush-Kuhn-Tucker (KKT) conditions require for the optimal
solution:
\begin{align}
\label{S_eq:KKT1}
\lambda_{ij} \ \mu_{ij} \ &= \ 0 \quad \mbox{and} \quad 
 \tau_j \ \left(\frac{1}{n} \ \sum_{i=1}^{n} \mu_{ij}^2 \ - \ 1 \right) \ = \ 0 \ . 
\end{align}

The derivative of $L$ with respect to  $\mu_{ij}$ is
\begin{align}
\label{S_eq:deriv1}
\frac{\partial L}{\partial \mu_{ij}} \ &= \ 
\frac{2}{ n} \ (\mu_{ij} \ - \ (\mu_p)_{ij}) \ - \ \lambda_{ij} \ + \ 
\frac{2}{ n} \ \tau_j \ \mu_{ij} \ = \ 0 \ . 
\end{align}
We multiply this equation by $\mu_{ij}$ and obtain:
\begin{align}
\frac{2}{ n} \ (\mu_{ij}^2 \ - \ (\mu_p)_{ij} \ \mu_{ij}) \ - \ \lambda_{ij} \
\mu_{ij} \ + \ 
\frac{2}{ n} \ \tau_j \ \mu_{ij}^2 \ = \ 0 \ . 
\end{align}
The KKT conditions give $ \lambda_{ij} \mu_{ij} = 0$, therefore this
term can be removed from the equation.
Next we sum over $i$:
\begin{align}
\frac{2}{ n} \ \sum_{i=1}^{n} \left(\mu_{ij}^2 \ - \ (\mu_p)_{ij} \
  \mu_{ij} \right)  \ + \ 
\frac{2}{ n} \ \sum_{i=1}^{n} \tau_j \ \mu_{ij}^2 \ = \ 0 \ . 
\end{align}
Using the equality constraint $1/n \sum_{i=1}^{n} \mu_{ij}^2=1$ and
dividing by 2 and gives: 
\begin{align}
1 \ - \ \frac{1}{ n} \
\sum_{i=1}^{n} (\mu_p)_{ij} \ \mu_{ij}  \ + \ 
\tau_j  \ = \ 0 \ . 
\end{align}
Solving for $\tau_j$ leads to:
\begin{align}
\label{S_eq:tau}
\tau_j  \ &= \ \frac{1}{ n} \
\sum_{i=1}^{n} (\mu_p)_{ij} \ \mu_{ij} \ - \ 1 \ . 
\end{align}

We insert $\tau_j$ into Eq.~\eqref{S_eq:deriv1}
\begin{align}
\label{S_eq:oft}
- \ (\mu_p)_{ij} \ - \ \frac{n}{2} \lambda_{ij} \ + \ 
\left(\frac{1}{ n} \
\sum_{s=1}^{n} (\mu_p)_{sj} \ \mu_{sj} \right) \ \mu_{ij} \ = \ 0 \ . 
\end{align}
We immediately see, that if $\mu_{ij} =  0$ 
then $(\mu_p)_{ij} = - \frac{n}{2} \lambda_{ij}<0$. 
Therefore we can assume $\mu_{ij} >  0$.
Multiplying  Eq.~\eqref{S_eq:oft} with $\mu_{ij}$ and using the KKT conditions gives 
\begin{align}
- \ (\mu_p)_{ij} \ \mu_{ij} \ + \ 
\left(\frac{1}{ n} \
\sum_{s=1}^{n} (\mu_p)_{sj} \ \mu_{sj} \right) \ \mu_{ij}^2 \ = \ 0 \ . 
\end{align}
Therefore $(\mu_p)_{ij} \mu_{ij}$ and $\frac{1}{ n} 
\sum_{s=1}^{n} (\mu_p)_{sj}  \mu_{sj}$ have the same sign or $\mu_{ij}=0$.
Since  $0 \leq \mu_{ij}$, we deduce that $(\mu_p)_{ij}$ and $\frac{1}{n} 
\sum_{s=1}^{n} (\mu_p)_{sj}  \mu_{sj}$ have the same sign or $\mu_{ij}=0$.
Since the sum is independent of $i$, 
all $(\mu_p)_{ij}$ with $\mu_{ij}>0$ have the same sign for $1 \leq i
\leq n$.
Solving Eq.~\eqref{S_eq:oft} for $\mu_{ij}$ gives
\begin{align}
\mu_{ij} \ &= \ 
\frac{(\mu_p)_{ij} \ + \ \frac{n}{2} \lambda_{ij}}{\frac{1}{ n} \
\sum_{s=1}^{n} (\mu_p)_{sj} \ \mu_{sj}} \ .
\end{align}

{\bf I.} If all $(\mu_p)_{ij}$ are non-positive for $1 \leq j \leq l$, 
then the sum $\frac{1}{n} \sum_{s=1}^{n} (\mu_p)_{sj}  \mu_{sj}$
is negative. From the first order derivative of the Lagrangian in 
Eq.~\eqref{S_eq:deriv1}, we can compute the second order derivative
\begin{align}
\label{S_eq:nec}
\frac{\partial^2 L}{\partial \mu_{ij} \partial \mu_{ij}} \ &= \ 
\frac{2}{ n}  \ + \ \frac{2}{ n} \ \tau_j \
= \  2 \ \sum_{i=1}^{n} (\mu_p)_{ij} \ \mu_{ij} \ < \ 0 \ .
\end{align}
We inserted the expression of Eq.~\eqref{S_eq:tau} for $\tau_j$.
Since all mixed second order derivatives are zero, the
(projected) Hessian of the Lagrangian is diagonal with negative
entries. Therefore it is strict negative definite.
Thus, the second order necessary conditions cannot be fulfilled.
The minimum is a border point of the constraints.

For each $j$ for which all $(\mu_p)_{ij}$ are non-positive for $1 \leq j \leq l$, 
optimization problem Eq.~\eqref{S_eq:optmu2a} defines 
a plane that has a
normal vector in the positive orthant (hyperoctant). For such a $j$ the
corresponding equality constraint defines a hypersphere. 
Minimization means that the plane containing the solution is parallel to the
original plane and should be as close to the origin as
possible. If we move the plane parallel from the origin into the
positive orthant, then the first intersection with the hypersphere is
\begin{align}
\mu_{ij} \ &= \  \left\{ 
\begin{array}{lcl}
\sqrt{n} & \mathrm{for} & j \ = \  \arg \max_{\hat{j}} \{(\mu_p)_{i\hat{j}}\}\\
0 &  \mathrm{otherwise} &
\end{array} \right. \ .
\end{align}
This is the solution for  $\mu_{ij}$ with 
$1 \leq j \leq l$ to our minimization problem.

{\bf II.} If one $(\mu_p)_{ij}$ is positive, then from
Eq.~\eqref{S_eq:oft} with this $(\mu_p)_{ij}$ follows  that 
$\frac{1}{ n} \sum_{s=1}^{n} (\mu_p)_{sj}  \mu_{sj}$ is positive,
otherwise Eq.~\eqref{S_eq:oft} has only negative terms on the left hand
side. In particular, the second order necessary conditions are always
fulfilled as Eq.~\eqref{S_eq:nec} is positive.
For $(\mu_p)_{ij} <0$ it follows from Eq.~\eqref{S_eq:oft}
that $\lambda_{ij}>0$ and from the KKT conditions that $\mu_{ij} =  0$.
For $(\mu_p)_{ij} >0$ it follows from Eq.~\eqref{S_eq:oft} that 
$\mu_{ij}>0$ and from the KKT conditions that $\lambda_{ij}=0$.
Therefore we define:
\begin{align}
\hat{\mu}_{ij} \ &= \  \left\{ 
\begin{array}{lcl}
0 & \mathrm{for} & (\mu_p)_{ij} \ \leq \ 0 \\
(\mu_p)_{ij} & \mathrm{for} &  (\mu_p)_{ij} \ > \ 0
\end{array} \right. \ ,
\end{align}
We write the solution as
\begin{align}
\mu_{ij} \ &= \ 
\frac{\hat{\mu}_{ij}}{\frac{1}{ n} \
\sum_{s=1}^{n} (\mu_p)_{sj} \ \mu_{sj}} \ = \ \alpha_j \
\hat{\mu}_{ij} \ .
\end{align}
We now use the equality constraint:
\begin{align}
\frac{1}{ n} \sum_{i=1}^{n} \mu_{ij}^2 
\ = \  \alpha_j^2 \ \frac{1}{ n} \sum_{i=1}^{n} \hat{\mu}_{ij}^2  \ =
\ 1 \ .
\end{align}
Solving for $\alpha_j$ gives:
\begin{align}
\alpha_j \ &= \  
\frac{1}{\sqrt{\frac{1}{n} \ \sum_{i=1}^{n} \hat{\mu}_{ij}^2}} \ .
\end{align}
Therefore the solution is
\begin{align}
\mu_{ij} \ &= \  
\frac{\hat{\mu}_{ij}}{\sqrt{\frac{1}{n} \ \sum_{i=1}^{n} \hat{\mu}_{ij}^2}} \ .
\end{align}

This finishes the proof.

\end{proof}

\subsection{Gradient and Scaled Gradient Projection and Projected Newton}
\label{sec:gradProjNew}

\subsubsection{Gradient Projection Algorithm}
\label{sec:gradProj}

The
{\em projected gradient descent} or 
{\em gradient projection algorithm} \cite{Bertsekas:76,Kelley:99}
performs first a gradient step and then projects the result to 
the {\em feasible set}.
The projection onto the feasible set is denoted by $\PP$, 
that is, the map that takes $\Bmu$ into the nearest point (in
the $L^2$-norm) in the feasible set to $\Bmu$.
The feasible set must be convex, however later we will introduce
gradient projection methods for non-convex feasible sets.

The gradient projection method is in our case
\begin{align}
\Bmu_{k+1} \ &= \ \PP \left( 
\Bmu_{k} \ + \ \lambda \  \BSi_p^{-1} (  \Bmu_p \ - \ \Bmu_{k}) 
\right) \ .
\end{align}

The Lipschitz constant for the gradient 
is $\|\BSi_p^{-1}\|_s= e_{\mathrm{max}}(\BSi_p^{-1})$, the largest 
eigenvalue of $\BSi_p^{-1}$.
The following statement is Theorem 5.4.5 in \cite{Kelley:99}.

\begin{theorem}[Theorem 5.4.5 in \cite{Kelley:99}]
\label{th:gradProj}
The {\em sufficient decrease} condition
\begin{align}
D_{\mathrm{KL}}(Q(\Bmu_{k+1}) \parallel p) \ - \
D_{\mathrm{KL}}(Q(\Bmu_{k}) \parallel p) \ &\leq \
\frac{- \alpha}{\lambda} \|\Bmu_{k} \ - \ \Bmu_{k+1}  \|^2
\end{align}
(e.g.\ with $\alpha=10^{-4}$) holds for all $\lambda$ such that
\begin{align}
0 \ &< \ \lambda \ \leq \ \frac{2 \
  (1-\alpha)}{e_{\mathrm{max}}(\BSi_p^{-1})} \ .
\end{align}
\end{theorem}

\begin{proof}
See \cite{Kelley:99}.
\end{proof}

{\em Theorem~\ref{th:gradProj} guarantees that we can increase the
  objective by gradient projection in 
the E-step, except the case where we already reached the maximum.}

For a fast upper bound on the maximal eigenvalue we use 
\begin{align}
e_{\mathrm{max}}(\BSi_p^{-1}) \ &\leq \ \TR (\BSi_p^{-1} ) 
\end{align}
and 
\begin{align}
e_{\mathrm{max}}(\BSi_p^{-1}) \ &\leq \ \|\BW\|_s^2 \ \|\BPs^{-1}\|_s \ -
1 \ ,
\end{align}
where the latter follows from
\begin{align}
\BSi_p^{-1} \ &= \ \BI \ + \ \BW^T \BPs^{-1}\BW \ .
\end{align}
Improved methods for finding an appropriate $\lambda$ by line search
methods have been
proposed \cite{Birgin:00,Serafini:05}.
We use a search with $\lambda=\beta^{t}$ with $t=0,1,2,\ldots$ and
$\beta=2^{-1}$ or $\beta=10^{-1}$.

A special version of the gradient projection method 
is the {\em generalized reduced method} \cite{Abadie:69}.
This method is able to solve our optimization problem with equality
constraints. 
The gradient projection method has been generalized by Rosen to 
non-linear constraints \cite{Rosen:61}.
The gradient projection algorithm 
of Rosen can also be used for a region which is not convex.
The idea is to linearize the nonlinear constraints and solve the
problem. Subsequently a
restoration move brings the solution back to the constraint
boundaries.
Rosen's gradient projection method was improved by \cite{Haug:79}.
{\em These methods guarantee that we can increase the objective in 
the E-step for non-convex feasible sets, 
except the case where we already reached the maximum.}
These algorithms for non-convex feasible sets 
will only give a local maximum.
Also the GAM algorithm will only find a local
maximum.

\subsubsection{Scaled Gradient Projection and Projected Newton Method}
\label{sec:projNewton}

Both the {\em scaled gradient projection algorithm}
and the {\em projected Newton method} 
were proposed in \cite{Bertsekas:82}.
We follow \cite{Kelley:99}.

The idea is to use a Newton update instead of the
a gradient update:
\begin{align}
\Bmu_{k+1} \ &= \ \PP \left( 
\Bmu_{k} \ + \ \lambda \  \BH^{-1} \ \BSi_p^{-1} (  \Bmu_p \ - \ \Bmu_{k}) 
\right) \ .
\end{align}

$\BH^{-1}$ can be an arbitrary strict positive definite matrix.
If we set $\BH^{-1}=\BSi_p$, then we have a Newton update of the
{\em projected Newton method} \cite{Bertsekas:82}.
For $\lambda=1$ we obtain
\begin{align}
\label{S_eq:update1}
\Bmu_{k+1} \ &= \ \PP \left(  \Bmu_p \right) \ .
\end{align}
otherwise
\begin{align}
\label{S_eq:update2}
\Bmu_{k+1} \ &= \ \PP \left( 
(1 \ - \ \lambda) \Bmu_{k} \ + \ \lambda  \Bmu_p \right) \ .
\end{align}

The search direction for the unconstrained
problem can be rotated by $\BH^{-1}$
to be orthogonal to the direction of decrease in the inactive
directions for the constrained problem.

To escape this possible problem,
an $\epsilon$-active set is introduced which contains all $j$
with $\mu_j\leq \epsilon$.
All columns and rows 
of the Hessian having an index in 
the $\epsilon$-active set are fixed to 
$\Be_j$. After sorting the indices of the 
$\epsilon$-active set together, they form a block
which is the sub-identity matrix.
$\BH$ is set to the Hessian $\BSi_p$ where the 
$\epsilon$-active set columns and rows are replaced by unit vectors. 

The following Theorem~\ref{th:projNewton}
is Lemma 5.5.1 in \cite{Kelley:99}.
{\em Theorem~\ref{th:projNewton} states that 
the objective decreases using the reduced Hessian in the projected
Newton method for convex feasible sets}.

\begin{theorem}[Lemma 5.5.1 in \cite{Kelley:99}]
\label{th:projNewton}
The {\em sufficient decrease} condition
\begin{align}
D_{\mathrm{KL}}(Q(\Bmu_{k+1}) \parallel p) \ - \
D_{\mathrm{KL}}(Q(\Bmu_{k}) \parallel p) \ &\leq \
- \ \alpha \  ( \Bmu_{k}  \ - \  \Bmu_p)^T \BSi_p^{-1} ( \Bmu_{k}  \ - \  \Bmu_{k+1})
\end{align}
holds for all $\lambda$ smaller than a
bound depending on $\BH$ and $\epsilon$.
\end{theorem}

\begin{proof}
See \cite{Kelley:99}.
\end{proof}

In practical applications, a proper 
$\lambda$ is found by line search.
The {\em projected Newton method}
uses $\lambda=1$ to set $\epsilon$ \cite{Bertsekas:82}:
\begin{align}
\epsilon \ &= \ \| \Bmu_{k} \ - \ \PP \left(  \Bmu_p \right) \| \ .
\end{align}

\subsubsection{Combined Method}
\label{sec:combined}

Following \cite{Kim:06,Serafini:05} we use the following very
general update rule, which includes 
the gradient projection algorithm, the scaled gradient projection algorithm,
and the projected Newton method.

We use following update for the E-step:
\begin{align}
\Bd_{k+1} \ &= \ \PP \left( 
\Bmu_{k} \ + \ \lambda \  \BH^{-1} \ \BSi_p^{-1} ( \Bmu_p \ - \ \Bmu_{k}) 
\right) \ , \\ \nonumber
\Bmu_{k+1} \ &= \ \PP \left( \Bmu_{k} \ + \ \gamma \left(\Bd_{k+1} \ - \
  \Bmu_{k} \right)\right) \ .
\end{align}
We have to project twice since the equality constraint produces a
manifold in the parameter space.

We iterate this update until we see a decrease of the objective in the 
E-step:
\begin{align}
D_{\mathrm{KL}}(Q_{k+1} \parallel p) \ - \
D_{\mathrm{KL}}(Q_{k} \parallel p) \ &< \ 0 \ .
\end{align}
For the constraints we have only to optimize the mean vector $\Bmu$ to ensure
\begin{align}
D_{\mathrm{KL}}(Q(\Bmu_{k+1}) \parallel p) \ - \
D_{\mathrm{KL}}(Q(\Bmu_{k}) \parallel p) \ &< \ 0 \ .
\end{align}
Even 
\begin{align}
D_{\mathrm{KL}}(Q(\Bmu_{k+1}) \parallel p) \ &= \
D_{\mathrm{KL}}(Q(\Bmu_{k}) \parallel p) 
\end{align}
can be sufficient if minimizing $\BSi_{k+1}=\BSi_p$ ensures
\begin{align}
D_{\mathrm{KL}}(Q_{k+1} \parallel p) \ &< \
D_{\mathrm{KL}}(Q_{k} \parallel p)  \ .
\end{align}

We use following schedule:
\begin{enumerate}
\item
\begin{itemize}
\item
$\BH^{-1} = \BSi_p$
\item
$\lambda=1$
\item
$\gamma=1$
\end{itemize}
That is 
\begin{align}
\Bmu_{k+1} \ &=  \ \PP \left(  \Bmu_p \right) \ .
\end{align}
\item
\begin{itemize}
\item
$\BH^{-1} = \BSi_p$
\item
$\lambda=1$
\item
$\gamma \in (0,1]$
\end{itemize}
That is 
\begin{align}
\Bmu_{k+1} \ &=  \ \PP \left( (1 \ - \ \gamma) \ 
\Bmu_{k} \ + \ \gamma \ \PP \left(  \Bmu_p \right) \right) \ .
\end{align}
\item
\begin{itemize}
\item
$\BH^{-1} = \BSi_p$
\item
$\lambda \in (0,1]$
\item
$\gamma=1$
\end{itemize}
That is 
\begin{align}
\Bmu_{k+1} \ &= \ \PP \left( 
(1 \ - \ \lambda) \Bmu_{k} \ + \ \lambda  \Bmu_p \right) \ .
\end{align}
\item
\begin{itemize}
\item
$\BH^{-1} = \BSi_p$
\item
$\lambda \in (0,1]$
\item
$\gamma=\in (0,1]$
\end{itemize}
That is 
\begin{align}
\Bmu_{k+1} \ &=  \ \PP \left( (1 \ - \ \gamma) \ 
\Bmu_{k} \ + \ \gamma \ \PP \left( 
(1 \ - \ \lambda) \Bmu_{k} \ + \ \lambda  \Bmu_p \right) \right) \ .
\end{align}
\item
\begin{itemize}
\item
$\BH^{-1} = \mathrm{R}(\BSi_p)$
\item
$\lambda \in (0,1]$
\item
$\gamma=\in (0,1]$
\end{itemize}
$\mathrm{R}(\BSi_p)$ denotes the reduced matrix (Hessian or a positive
definite) according to
the projected Newton method or the scaled gradient projection
algorithm.
For convex feasible sets we can guarantee at this level already an
increase of the objective at the E-step.
\item
\begin{itemize}
\item
$\BH^{-1} = \BI$
\item
$\lambda \in (0,1]$
\item
$\gamma=\in (0,1]$
\end{itemize}
This is the 
gradient projection algorithm. In particular we include the 
generalized reduced method and Rosen's gradient projection method.
At this step we guarantee an increase of the objective at the E-step 
even for non-convex feasible sets because we also use
complex methods for constraint optimization.
\end{enumerate}

Step 5. ensures an improvement if only using
rectifying constraints according to the theory of
projected Newton methods \cite{Kelley:99}.
Step 6. ensures an improvement if using both
rectifying constraints and normalizing constraints,
because we use known methods for constraint optimization.
To set $\Bmu_{k+1}=\Bmu_k$ is sufficient to increase the objective at
the E-step if $\BSi_{k+1}=\BSi_p$ decreases the KL divergence. 
However we will not always set $\Bmu_{k+1}=\Bmu_k$ to avoid 
accumulation points outside the solution set.

\section{Alternative Gaussian Prior}
\label{sec:prior}

We assume $\Bh$ is Gaussian with covariance $\BM$ and mean
$\Bxi$ 
\begin{align}
\Bh \ &\sim \ \Ncal\left(\Bxi,\BM\right) \ .
\end{align}
We derive the posterior for this prior. 

The likelihood is Gaussian since a affine transformation of a Gaussian random
variable is again a Gaussian random variable and the convolution of two Gaussians
is Gaussian, too. 
Thus, $\Bv=\BW \Bh+\Bep$ is Gaussian if $\Bh$ and $\Bep$
are both Gaussian.
For the prior moments we have
\begin{align}
\EXP(\Bh) \ &= \ \Bxi \ , \\
\EXP(\Bh \Bh^T) \ &= \ \BM  \ + \ \Bxi \ \Bxi^T \ , \\
\mathrm{var}(\Bh) \  &= \ \BM
\end{align}
and for the likelihood of $\Bv$ we obtain the moments
\begin{align}
\EXP(\Bv) \ &= \ \BW \Bxi \ , \\ 
\EXP(\Bv \Bv^T)  \ &= \ \BW \ \EXP(\Bh \Bh^T) \ \BW^T \ + \ \BPs \\\nonumber
&= \ \BW \ \BM \ \BW^T \ + \ \BPs  \ + \ \BW
\ \Bxi \ \Bxi^T \ \BW^T \ , \\
\mathrm{var}(\Bv) \  &= \  \BW \ \BM \ \BW^T \ + \ \BPs \ .
\end{align}

We need some algebraic identities to derive the posterior.
The Woodbury matrix identity gives
\begin{align}
\BM \ - \  \BM \ \BW^T \ \left( \BW \ \BM \
\BW^T \ + \ \BPs \right)^{-1} \BW \ \BM \ = \ \left( \BM^{-1} \ + \ \BW^T
\BPs^{-1}\BW \right)^{-1} \ .
\end{align}
Multiplying this equation from the left hand side with $\BPs^{-1} \BW$
gives
\begin{align}
&\BPs^{-1} \ \BW \ \left( \BM^{-1} \ + \ \BW^T
\BPs^{-1}\BW \right)^{-1} \\\nonumber
&= \
\BPs^{-1} \ \BW \ \BM \ - \  \BPs^{-1} \ \BW \ \BM \ \BW^T \ \left( \BW \ \BM \
\BW^T \ + \ \BPs \right)^{-1} \BW \ \BM  \\\nonumber
&= \
\BPs^{-1}  \left( \BW \ \BM \
\BW^T \ + \ \BPs \right) \ \left( \BW \ \BM \
\BW^T \ + \ \BPs \right)^{-1} \ \BW \ \BM \ - \\\nonumber
&~~~~\BPs^{-1} \ \BW \ \BM \ \BW^T \ \left( \BW \ \BM \
\BW^T \ + \ \BPs \right)^{-1} \BW \ \BM \\\nonumber
&= \ \left( \BPs^{-1}  \left( \BW \ \BM \
\BW^T \ + \ \BPs \right) \ - \  \BPs^{-1} \ \BW \ \BM \ \BW^T
\right) \  \left( \BW \ \BM \
\BW^T \ + \ \BPs \right)^{-1} \BW \ \BM \\\nonumber
&= \ \left( \BW \ \BM \
\BW^T \ + \ \BPs \right)^{-1} \BW \ \BM \ .
\end{align}
It follows that 
\begin{align}
 \BM \ \BW^T \ \left( \BW \ \BW^T \ + \ \BPs \right)^{-1} \ \Ba 
\ = \ \left(\BM^{-1} \
  + \ \BW^T \BPs^{-1}\BW \right)^{-1} \ \BW^T
\BPs^{-1} \ \Ba  \ .
\end{align}

The posterior $p(\Bh \mid \Bv)$ 
is derived from Gaussian conditioning because 
both the likelihood $p(\Bv)$ and the prior $p(\Bh)$ are 
Gaussian distributed.
The conditional distribution $p(\Ba \mid \Bb)$ of two random variables
$\Ba$ and $\Bb$ that both follow a
Gaussian distribution is a Gaussian:
\begin{align}
\Ba  \ &\sim \ \Ncal\left( \Bmu_a ,  \BSi_{aa} \right) \ , \\
\Bb \ &\sim \ \Ncal\left( \Bmu_b ,  \BSi_{bb}
  \right) \ , \\
\BSi_{ba} \ &= \ \COV(\Bb,\Ba) \ , \\
\BSi_{ab} \ &= \ \COV(\Ba,\Bb)  \ , \\
\Ba \mid \Bb  \ &\sim \  \Ncal\left( \Bmu_a \ + \ \BSi_{ab}
  \BSi_{bb}^{-1} \left( \Bb \ - \ \Bmu_b \right) \ , \
\BSi_{aa} \ - \  \BSi_{ab}  \BSi_{bb}^{-1} \BSi_{ba} \right) \ .
\end{align}

Therefore we need 
the second moments between $\Bv$ and $\Bh$:
\begin{align}
\EXP(\Bv \Bh^T) \ &= \ 
\EXP(\BW \Bh \Bh^T) \ + \ \EXP(\Bep \Bh^T)
\ = \ \BW  \ ( \BM \ + \ \Bxi \ \Bxi^T ) \ .
\end{align}
The covariances between  $\Bv$ and $\Bh$ are
\begin{align}
\COV(\Bv,\Bh) \ & = \ \EXP(\Bv \Bh^T) \ - \ \EXP(\Bv) \EXP(\Bh^T)
\\\nonumber
&= \  \BW  \ \BM \ + \ \BW \Bxi \ \Bxi^T \ - \ \BW \Bxi \ \Bxi^T \ = \   \BW  \ \BM \ , \\
\COV(\Bh,\Bv) \ &= \ \EXP(\Bh \Bv^T) \ - \ \EXP(\Bh) \EXP(\Bv^T)= \   \ \BM  \ \BW^T \ .
\end{align}

Thus,
the mean of $p(\Bh \mid \Bv)$ is
\begin{align}
\Bmu_{\Bh \mid \Bv} \ &= \ \Bxi \ + \ \BM  \ \BW^T \ (\BW \ \BM \ \BW^T
\ + \ \BPs)^{-1} (\Bv \ - \ \BW \Bxi) \\\nonumber
&= \ \Bxi \ + \  \left(\BM^{-1} \
  + \ \BW^T \BPs^{-1}\BW \right)^{-1} \ \BW^T
\BPs^{-1} \ (\Bv \ + \ \BW \Bxi) \\\nonumber
&=  \ \left(\BM^{-1} \
  + \ \BW^T \BPs^{-1}\BW \right)^{-1} \ \left(\BM^{-1} \
  + \ \BW^T \BPs^{-1}\BW \right) \Bxi \\\nonumber 
&~~~~+ \  \left(\BM^{-1} \
  + \ \BW^T \BPs^{-1}\BW \right)^{-1} \ \BW^T
\BPs^{-1} \ (\Bv \ - \ \BW \Bxi)\\\nonumber
&= \ \left(\BM^{-1} \
  + \ \BW^T \BPs^{-1}\BW \right)^{-1}\\\nonumber
&~~~~\left( \BM^{-1} \ \Bxi \ 
  + \ \BW^T \BPs^{-1}\BW \ \Bxi 
 \ + \ \BW^T
\BPs^{-1} \ \Bv \ - \  \BW^T
\BPs^{-1} \BW \Bxi \right) \\\nonumber
&= \ \left(\BM^{-1} \
  + \ \BW^T \BPs^{-1}\BW \right)^{-1}
\left(  \BW^T
\BPs^{-1} \ \Bv \ + \ \BM^{-1} \ \Bxi  \right) \ .
\end{align}
The covariance matrix of $p(\Bh \mid \Bv)$ is
\begin{align}
\BSi_{\Bh \mid \Bv} \ &= \ \BM \ - \   \BM \ \BW^T \ \left( \BW \ \BM \
\BW^T \ + \ \BPs \right)^{-1} \BW \  \BM  \\\nonumber
&= \ \left(  \BM^{-1} \ + \ \BW^T
\BPs^{-1}\BW \right)^{-1} \ .
\end{align}

In particular, the variable $\Bxi$ may be used to enforce more
sparseness by setting its components to negative values.
Since the covariance matrix $\BSi_{\Bh \mid \Bv}$ is positive
semi-definite, we ensure that
\begin{align}
\Bxi^T \ \left(\BM^{-1} \
  + \ \BW^T \BPs^{-1}\BW \right)^{-1} \ \Bxi \ \geq \ 0 \ .
\end{align}
If $\Bxi=-\rho \BOn$ ($\BOn$ is the vector with all components being
one), then the largest absolute components of $\BSi_{\Bh \mid \Bv}\Bxi$ 
must be negative.
Thus, $\Bxi=-\rho \BOn$ leads to sparser solutions.

\section{Hyperparameters Selected for Method Assessment}
\label{S_sec:hyperpars}

The performance of rectified factor
networks (RFNs) as unsupervised methods for data representation was
compared with:\newline
(1) \textbf{RFN}: rectified factor networks,\newline
(2) \textbf{RFNn}: RFNs without normalization,\newline
(3) \textbf{DAE}: denoising autoencoders with rectified linear units,\newline
(4) \textbf{RBM}: restricted Boltzmann machines with Gaussian visible units and
hidden binary units,\newline
(5) \textbf{FAsp}: factor analysis with Jeffrey's prior ($p(z)\propto
1/z$) on the hidden units which is
sparser than a Laplace prior,\newline
(6) \textbf{FAlap}: factor analysis with Laplace prior on the
hidden units,\newline
(7) \textbf{ICA}: independent component analysis by FastICA \cite{Hyvarinen:99a},\newline
(8) \textbf{SFA}: sparse factor analysis with a Laplace prior on the parameters,\newline
(9) \textbf{FA}: standard factor analysis,\newline
(10) \textbf{PCA}: principal component analysis.\newline
The number of components are fixed to 50, 100, or 150 for each method.
The used hyperparameters are listed in Tab.~\ref{S_table:hyperparams}.

\begin{table}[h]
  \caption{Hyperparameters of all methods that were used to assess the performance of rectified factor networks (RFNs) as unsupervised methods for data representation.}
  \label{S_table:hyperparams}
\begin{center}
\begin{tabular}{l*{1}{>{\raggedright\arraybackslash}p{25em}}}
\bf Method &  \bf Used hyperparameters\\
\toprule
RFN & \{learning rate=0.1, iterations=1000\} \\
RFNn &  \{learning rate=0.1, iterations=1000\} \\
DAE &  \{corruption level=0.2, learning rate=1e-04, iterations=1000\} \\
RBM &  \{learning rate=0.01, iterations=1000\} \\
FAsp & \{iterations=500\} \\
FAlap &  \{iterations=500\} \\
SFA &  \{Laplace weight decay factor=5e-05, iterations=500\} \\
  \addlinespace[1pt]
\bottomrule 
\end{tabular}
\end{center}

\end{table}

\vfill
\clearpage

\section{Data Set I}
\label{S_sec:data1}

The number of components are fixed to 50, 100 or 150.

We generated nine different benchmark data sets (D1 to D9), where each data set consists of
100 instances for averaging the results.
Each instance consists of 100 samples and 100 features resulting in a
100$\times$100 data matrix. 
Into these data matrices, structures are implanted as biclusters \cite{Hochreiter:10s}. 
A bicluster is a pattern consisting of a particular number of features
which is found in a particular number of samples. 
The size of the bicluster is given by the number of features that form
the pattern and by the number of samples in which the pattern is found.
The data sets had different noise levels and different bicluster sizes.
We considered large and small bicluster sizes, 
where large biclusters have 20--30 samples and 20--30 features, 
while small biclusters have 3--8 samples and 3--8 features.
The signal strength (scaling factor) of a pattern in a sample
was randomly chosen according to the Gaussian $\Ncal\left(1,1\right)$.
Finally, to each data matrix background noise was added, where the
noise is distributed according to a zero-mean 
Gaussian with standard deviation 1, 5, or 10.
The data sets are described in Tab.~\ref{tab:data}.
The remaining components of the spanning outer product vectors were drawn by  $\Ncal\left(0,0.01\right)$.

\begin{table}[ht]
\centering
\caption{Overview over the datasets. 
Shown is the background noise (``noise''), the number of large
biclusters ($n_1$), and the number of small biclusters ($n_2$).\label{tab:data}}
\begin{tabular}{l*{9}{>{\raggedleft\arraybackslash}p{2em}}}
\toprule
&D1&D2&D3&D4&D5&D6&D7&D8&D9 \\
\midrule
noise&1&5&10&1&5&10&1&5&10 \\
$n_1$&10&10&10&15&15&15&5&5&5 \\
$n_2$&10&10&10&5&5&5&15&15&15\\
\bottomrule
\end{tabular}%
\end{table}%

\begin{sidewaystable}
\centering
\caption{ Comparison for 50 factors / hidden units extracted by RFN, RFN without normalization (RFNn), denoising
  autoencoder (DAE), restricted Boltzmann machines (RBM), factor
  analysis with a very sparse prior (FAsp), factor analysis with a Laplace
prior (FAlap), independent component analysis (ICA), sparse factor
analysis (SFA), factor analysis (FA), and principal component analysis
(PCA) on nine data sets. Criteria are: sparseness of the coding units (SP),
reconstruction error (ER), and  the difference between the empirical and the model covariance
matrix (CO). The lower right column block gives the average 
SP (\%), ER and CO.  Results reported here, are the mean together with the standard deviation of 100 instances. The maximal value in the table and the maximal standard deviation was set to 999 and to 99, respectively.  
 \label{S_tab:compare_n50_II}%
}
\begin{tabular*}{\textwidth}{l*{3}{>{\columncolor{mColor1} \raggedleft\arraybackslash}p{2.9em}}*{3}{>{\raggedleft\arraybackslash}p{2.9em}}*{3}{>{\columncolor{mColor1} \raggedleft\arraybackslash}p{2.9em}}*{3}{>{\raggedleft\arraybackslash}p{2.9em}}*{3}{>{\columncolor{mColor1} \raggedleft\arraybackslash}p{2.9em}}}
\toprule

&\multicolumn{3}{c}{D1} 
&\multicolumn{3}{c}{D2} 
&\multicolumn{3}{c}{D3} 
&\multicolumn{3}{c}{D4} 
&\multicolumn{3}{c}{D5} \\
\cmidrule(r){2-4} 
\cmidrule(lr){5-7} 
\cmidrule(lr){8-10} 
\cmidrule(lr){11-13} 
\cmidrule(l){14-16} 
& SP & ER & CO
& SP & ER & CO
& SP & ER & CO
& SP & ER & CO
& SP & ER & CO\\
RFN&\footnotesize74\tiny$\pm$0&\footnotesize58\tiny$\pm$1&\footnotesize5\tiny$\pm$0&\footnotesize75\tiny$\pm$0&\footnotesize233\tiny$\pm$3&\footnotesize66\tiny$\pm$1&\footnotesize75\tiny$\pm$0&\footnotesize456\tiny$\pm$5&\footnotesize253\tiny$\pm$6&\footnotesize74\tiny$\pm$0&\footnotesize63\tiny$\pm$1&\footnotesize6\tiny$\pm$1&\footnotesize75\tiny$\pm$0&\footnotesize236\tiny$\pm$3&\footnotesize68\tiny$\pm$2\\ 
RFNn&\footnotesize73\tiny$\pm$0&\footnotesize85\tiny$\pm$3&\footnotesize13\tiny$\pm$2&\footnotesize75\tiny$\pm$0&\footnotesize272\tiny$\pm$3&\footnotesize85\tiny$\pm$2&\footnotesize75\tiny$\pm$0&\footnotesize531\tiny$\pm$6&\footnotesize321\tiny$\pm$7&\footnotesize72\tiny$\pm$0&\footnotesize95\tiny$\pm$4&\footnotesize17\tiny$\pm$2&\footnotesize74\tiny$\pm$0&\footnotesize276\tiny$\pm$4&\footnotesize89\tiny$\pm$3\\ 
DAE&\footnotesize65\tiny$\pm$0&\footnotesize65\tiny$\pm$2&\footnotesize ---&\footnotesize66\tiny$\pm$0&\footnotesize233\tiny$\pm$2&\footnotesize ---&\footnotesize66\tiny$\pm$0&\footnotesize456\tiny$\pm$4&\footnotesize ---&\footnotesize65\tiny$\pm$1&\footnotesize71\tiny$\pm$2&\footnotesize ---&\footnotesize66\tiny$\pm$0&\footnotesize237\tiny$\pm$2&\footnotesize ---\\ 
RBM&\footnotesize25\tiny$\pm$2&\footnotesize86\tiny$\pm$3&\footnotesize ---&\footnotesize11\tiny$\pm$1&\footnotesize287\tiny$\pm$3&\footnotesize ---&\footnotesize10\tiny$\pm$1&\footnotesize558\tiny$\pm$5&\footnotesize ---&\footnotesize25\tiny$\pm$2&\footnotesize94\tiny$\pm$3&\footnotesize ---&\footnotesize11\tiny$\pm$1&\footnotesize292\tiny$\pm$3&\footnotesize ---\\ 
FAsp&\footnotesize39\tiny$\pm$1&\footnotesize232\tiny$\pm$31&\footnotesize654\tiny$\pm$99&\footnotesize40\tiny$\pm$1&\footnotesize999\tiny$\pm$41&\footnotesize999\tiny$\pm$99&\footnotesize41\tiny$\pm$1&\footnotesize999\tiny$\pm$99&\footnotesize999\tiny$\pm$99&\footnotesize38\tiny$\pm$1&\footnotesize318\tiny$\pm$33&\footnotesize999\tiny$\pm$99&\footnotesize40\tiny$\pm$1&\footnotesize999\tiny$\pm$48&\footnotesize999\tiny$\pm$99\\ 
FAlap&\footnotesize4\tiny$\pm$0&\footnotesize53\tiny$\pm$2&\footnotesize144\tiny$\pm$36&\footnotesize4\tiny$\pm$0&\footnotesize224\tiny$\pm$5&\footnotesize185\tiny$\pm$5&\footnotesize5\tiny$\pm$0&\footnotesize439\tiny$\pm$9&\footnotesize692\tiny$\pm$16&\footnotesize4\tiny$\pm$0&\footnotesize55\tiny$\pm$2&\footnotesize180\tiny$\pm$39&\footnotesize4\tiny$\pm$0&\footnotesize226\tiny$\pm$5&\footnotesize192\tiny$\pm$6\\ 
ICA&\footnotesize2\tiny$\pm$0&\footnotesize34\tiny$\pm$0&\footnotesize ---&\footnotesize2\tiny$\pm$0&\footnotesize164\tiny$\pm$2&\footnotesize ---&\footnotesize2\tiny$\pm$0&\footnotesize324\tiny$\pm$4&\footnotesize ---&\footnotesize2\tiny$\pm$0&\footnotesize35\tiny$\pm$0&\footnotesize ---&\footnotesize2\tiny$\pm$0&\footnotesize166\tiny$\pm$2&\footnotesize ---\\ 
SFA&\footnotesize1\tiny$\pm$0&\footnotesize42\tiny$\pm$1&\footnotesize11\tiny$\pm$2&\footnotesize1\tiny$\pm$0&\footnotesize206\tiny$\pm$4&\footnotesize56\tiny$\pm$2&\footnotesize1\tiny$\pm$0&\footnotesize406\tiny$\pm$9&\footnotesize215\tiny$\pm$7&\footnotesize1\tiny$\pm$0&\footnotesize42\tiny$\pm$1&\footnotesize13\tiny$\pm$2&\footnotesize1\tiny$\pm$0&\footnotesize208\tiny$\pm$4&\footnotesize58\tiny$\pm$2\\ 
FA&\footnotesize1\tiny$\pm$0&\footnotesize42\tiny$\pm$1&\footnotesize6\tiny$\pm$1&\footnotesize1\tiny$\pm$0&\footnotesize206\tiny$\pm$4&\footnotesize54\tiny$\pm$2&\footnotesize1\tiny$\pm$0&\footnotesize407\tiny$\pm$8&\footnotesize210\tiny$\pm$6&\footnotesize1\tiny$\pm$0&\footnotesize42\tiny$\pm$1&\footnotesize8\tiny$\pm$1&\footnotesize1\tiny$\pm$0&\footnotesize208\tiny$\pm$4&\footnotesize56\tiny$\pm$2\\ 
PCA&\footnotesize1\tiny$\pm$0&\footnotesize34\tiny$\pm$0&\footnotesize ---&\footnotesize0\tiny$\pm$0&\footnotesize164\tiny$\pm$2&\footnotesize ---&\footnotesize0\tiny$\pm$0&\footnotesize324\tiny$\pm$4&\footnotesize ---&\footnotesize1\tiny$\pm$0&\footnotesize35\tiny$\pm$0&\footnotesize ---&\footnotesize0\tiny$\pm$0&\footnotesize166\tiny$\pm$2&\footnotesize ---\\ 
\midrule &\multicolumn{3}{c}{D6} 
 &\multicolumn{3}{c}{D7} 
 &\multicolumn{3}{c}{D8} 
 &\multicolumn{3}{c}{D9} 
 &\multicolumn{3}{c}{average } \\ 
 \cmidrule(r){2-4} 
 \cmidrule(lr){5-7} 
 \cmidrule(lr){8-10} 
 \cmidrule(lr){11-13} 
 \cmidrule(l){14-16} 
 & SP & ER & CO 
 & SP & ER & CO 
 & SP & ER & CO 
 & SP & ER & CO 
 & SP & ER & CO\\ 
RFN&\footnotesize75\tiny$\pm$0&\footnotesize458\tiny$\pm$5&\footnotesize256\tiny$\pm$6&\footnotesize75\tiny$\pm$0&\footnotesize53\tiny$\pm$1&\footnotesize4\tiny$\pm$1&\footnotesize75\tiny$\pm$0&\footnotesize230\tiny$\pm$3&\footnotesize64\tiny$\pm$1&\footnotesize75\tiny$\pm$0&\footnotesize454\tiny$\pm$5&\footnotesize251\tiny$\pm$5&\footnotesize75\tiny$\pm$0&\footnotesize249\tiny$\pm$3&\footnotesize108\tiny$\pm$3\\ 
RFNn&\footnotesize75\tiny$\pm$0&\footnotesize532\tiny$\pm$6&\footnotesize323\tiny$\pm$7&\footnotesize73\tiny$\pm$0&\footnotesize73\tiny$\pm$3&\footnotesize10\tiny$\pm$2&\footnotesize75\tiny$\pm$0&\footnotesize268\tiny$\pm$3&\footnotesize82\tiny$\pm$2&\footnotesize75\tiny$\pm$0&\footnotesize528\tiny$\pm$6&\footnotesize317\tiny$\pm$7&\footnotesize74\tiny$\pm$0&\footnotesize295\tiny$\pm$4&\footnotesize140\tiny$\pm$4\\ 
DAE&\footnotesize66\tiny$\pm$0&\footnotesize458\tiny$\pm$4&\footnotesize ---&\footnotesize65\tiny$\pm$0&\footnotesize58\tiny$\pm$1&\footnotesize ---&\footnotesize66\tiny$\pm$0&\footnotesize230\tiny$\pm$2&\footnotesize ---&\footnotesize66\tiny$\pm$0&\footnotesize453\tiny$\pm$5&\footnotesize ---&\footnotesize66\tiny$\pm$0&\footnotesize251\tiny$\pm$3&\footnotesize ---\\ 
RBM&\footnotesize10\tiny$\pm$1&\footnotesize561\tiny$\pm$5&\footnotesize ---&\footnotesize23\tiny$\pm$2&\footnotesize76\tiny$\pm$2&\footnotesize ---&\footnotesize11\tiny$\pm$1&\footnotesize282\tiny$\pm$3&\footnotesize ---&\footnotesize10\tiny$\pm$1&\footnotesize555\tiny$\pm$5&\footnotesize ---&\footnotesize15\tiny$\pm$1&\footnotesize310\tiny$\pm$4&\footnotesize ---\\ 
FAsp&\footnotesize40\tiny$\pm$2&\footnotesize999\tiny$\pm$99&\footnotesize999\tiny$\pm$99&\footnotesize39\tiny$\pm$1&\footnotesize152\tiny$\pm$26&\footnotesize345\tiny$\pm$99&\footnotesize40\tiny$\pm$1&\footnotesize999\tiny$\pm$31&\footnotesize999\tiny$\pm$99&\footnotesize41\tiny$\pm$1&\footnotesize999\tiny$\pm$99&\footnotesize999\tiny$\pm$99&\footnotesize40\tiny$\pm$1&\footnotesize999\tiny$\pm$63&\footnotesize999\tiny$\pm$99\\ 
FAlap&\footnotesize5\tiny$\pm$0&\footnotesize443\tiny$\pm$9&\footnotesize701\tiny$\pm$15&\footnotesize4\tiny$\pm$0&\footnotesize50\tiny$\pm$2&\footnotesize110\tiny$\pm$37&\footnotesize4\tiny$\pm$0&\footnotesize221\tiny$\pm$5&\footnotesize177\tiny$\pm$4&\footnotesize5\tiny$\pm$0&\footnotesize439\tiny$\pm$10&\footnotesize686\tiny$\pm$15&\footnotesize4\tiny$\pm$0&\footnotesize239\tiny$\pm$6&\footnotesize341\tiny$\pm$19\\ 
ICA&\footnotesize2\tiny$\pm$0&\footnotesize325\tiny$\pm$4&\footnotesize ---&\footnotesize2\tiny$\pm$0&\footnotesize34\tiny$\pm$0&\footnotesize ---&\footnotesize2\tiny$\pm$0&\footnotesize163\tiny$\pm$2&\footnotesize ---&\footnotesize2\tiny$\pm$0&\footnotesize322\tiny$\pm$4&\footnotesize ---&\footnotesize2\tiny$\pm$0&\footnotesize174\tiny$\pm$2&\footnotesize ---\\ 
SFA&\footnotesize1\tiny$\pm$0&\footnotesize408\tiny$\pm$9&\footnotesize217\tiny$\pm$7&\footnotesize1\tiny$\pm$0&\footnotesize42\tiny$\pm$1&\footnotesize8\tiny$\pm$2&\footnotesize1\tiny$\pm$0&\footnotesize204\tiny$\pm$4&\footnotesize54\tiny$\pm$2&\footnotesize1\tiny$\pm$0&\footnotesize405\tiny$\pm$9&\footnotesize213\tiny$\pm$7&\footnotesize1\tiny$\pm$0&\footnotesize218\tiny$\pm$5&\footnotesize94\tiny$\pm$3\\ 
FA&\footnotesize1\tiny$\pm$0&\footnotesize409\tiny$\pm$9&\footnotesize212\tiny$\pm$7&\footnotesize1\tiny$\pm$0&\footnotesize42\tiny$\pm$1&\footnotesize4\tiny$\pm$1&\footnotesize1\tiny$\pm$0&\footnotesize205\tiny$\pm$4&\footnotesize53\tiny$\pm$2&\footnotesize1\tiny$\pm$0&\footnotesize405\tiny$\pm$8&\footnotesize208\tiny$\pm$6&\footnotesize1\tiny$\pm$0&\footnotesize218\tiny$\pm$4&\footnotesize90\tiny$\pm$3\\ 
PCA&\footnotesize0\tiny$\pm$0&\footnotesize325\tiny$\pm$4&\footnotesize ---&\footnotesize1\tiny$\pm$0&\footnotesize34\tiny$\pm$0&\footnotesize ---&\footnotesize0\tiny$\pm$0&\footnotesize163\tiny$\pm$2&\footnotesize ---&\footnotesize0\tiny$\pm$0&\footnotesize322\tiny$\pm$4&\footnotesize ---&\footnotesize0\tiny$\pm$0&\footnotesize174\tiny$\pm$2&\footnotesize ---\\ 
\bottomrule
\end{tabular*}%
\end{sidewaystable}

\begin{sidewaystable}
\centering
\caption{ Comparison for 100 factors / hidden units extracted by RFN, RFN without normalization (RFNn), denoising
  autoencoder (DAE), restricted Boltzmann machines (RBM), factor
  analysis with a very sparse prior (FAsp), factor analysis with a Laplace
prior (FAlap), independent component analysis (ICA), sparse factor
analysis (SFA), factor analysis (FA), and principal component analysis
(PCA) on nine data sets. Criteria are: sparseness of the coding units (SP),
reconstruction error (ER), and the difference between the empirical and the model covariance
matrix (CO). The lower right column block gives the average 
SP (\%), ER and CO.  Results reported here, are the mean together with the standard deviation of 100 instances.  The maximal value in the table and the maximal standard deviation was set to 999 and to 99, respectively.  
 \label{S_tab:compare_n100_II}%
}
\begin{tabular*}{\textwidth}{l*{3}{>{\columncolor{mColor1} \raggedleft\arraybackslash}p{2.9em}}*{3}{>{\raggedleft\arraybackslash}p{2.9em}}*{3}{>{\columncolor{mColor1} \raggedleft\arraybackslash}p{2.9em}}*{3}{>{\raggedleft\arraybackslash}p{2.9em}}*{3}{>{\columncolor{mColor1} \raggedleft\arraybackslash}p{2.9em}}}
\toprule

&\multicolumn{3}{c}{D1} 
&\multicolumn{3}{c}{D2} 
&\multicolumn{3}{c}{D3} 
&\multicolumn{3}{c}{D4} 
&\multicolumn{3}{c}{D5} \\
\cmidrule(r){2-4} 
\cmidrule(lr){5-7} 
\cmidrule(lr){8-10} 
\cmidrule(lr){11-13} 
\cmidrule(l){14-16} 
& SP & ER & CO
& SP & ER & CO
& SP & ER & CO
& SP & ER & CO
& SP & ER & CO\\
RFN&\footnotesize79\tiny$\pm$1&\footnotesize23\tiny$\pm$3&\footnotesize2\tiny$\pm$0&\footnotesize82\tiny$\pm$1&\footnotesize63\tiny$\pm$9&\footnotesize16\tiny$\pm$3&\footnotesize82\tiny$\pm$1&\footnotesize120\tiny$\pm$17&\footnotesize61\tiny$\pm$15&\footnotesize78\tiny$\pm$1&\footnotesize27\tiny$\pm$3&\footnotesize2\tiny$\pm$1&\footnotesize82\tiny$\pm$1&\footnotesize62\tiny$\pm$7&\footnotesize16\tiny$\pm$3\\ 
RFNn&\footnotesize77\tiny$\pm$0&\footnotesize61\tiny$\pm$4&\footnotesize6\tiny$\pm$1&\footnotesize80\tiny$\pm$0&\footnotesize169\tiny$\pm$4&\footnotesize36\tiny$\pm$2&\footnotesize80\tiny$\pm$0&\footnotesize326\tiny$\pm$8&\footnotesize135\tiny$\pm$6&\footnotesize76\tiny$\pm$1&\footnotesize73\tiny$\pm$4&\footnotesize9\tiny$\pm$2&\footnotesize79\tiny$\pm$0&\footnotesize171\tiny$\pm$5&\footnotesize37\tiny$\pm$2\\ 
DAE&\footnotesize67\tiny$\pm$0&\footnotesize48\tiny$\pm$2&\footnotesize ---&\footnotesize70\tiny$\pm$0&\footnotesize134\tiny$\pm$1&\footnotesize ---&\footnotesize70\tiny$\pm$0&\footnotesize260\tiny$\pm$2&\footnotesize ---&\footnotesize67\tiny$\pm$0&\footnotesize54\tiny$\pm$2&\footnotesize ---&\footnotesize70\tiny$\pm$0&\footnotesize137\tiny$\pm$1&\footnotesize ---\\ 
RBM&\footnotesize14\tiny$\pm$1&\footnotesize81\tiny$\pm$3&\footnotesize ---&\footnotesize4\tiny$\pm$0&\footnotesize266\tiny$\pm$3&\footnotesize ---&\footnotesize4\tiny$\pm$0&\footnotesize514\tiny$\pm$6&\footnotesize ---&\footnotesize15\tiny$\pm$1&\footnotesize88\tiny$\pm$2&\footnotesize ---&\footnotesize4\tiny$\pm$0&\footnotesize270\tiny$\pm$3&\footnotesize ---\\ 
FAsp&\footnotesize72\tiny$\pm$0&\footnotesize233\tiny$\pm$32&\footnotesize499\tiny$\pm$99&\footnotesize62\tiny$\pm$0&\footnotesize999\tiny$\pm$43&\footnotesize999\tiny$\pm$99&\footnotesize56\tiny$\pm$0&\footnotesize999\tiny$\pm$99&\footnotesize999\tiny$\pm$99&\footnotesize71\tiny$\pm$0&\footnotesize320\tiny$\pm$34&\footnotesize878\tiny$\pm$99&\footnotesize62\tiny$\pm$0&\footnotesize999\tiny$\pm$49&\footnotesize999\tiny$\pm$99\\ 
FAlap&\footnotesize6\tiny$\pm$0&\footnotesize27\tiny$\pm$3&\footnotesize202\tiny$\pm$17&\footnotesize6\tiny$\pm$0&\footnotesize38\tiny$\pm$3&\footnotesize756\tiny$\pm$33&\footnotesize6\tiny$\pm$0&\footnotesize74\tiny$\pm$5&\footnotesize999\tiny$\pm$83&\footnotesize6\tiny$\pm$0&\footnotesize31\tiny$\pm$3&\footnotesize274\tiny$\pm$23&\footnotesize6\tiny$\pm$0&\footnotesize39\tiny$\pm$3&\footnotesize778\tiny$\pm$34\\ 
ICA&\footnotesize3\tiny$\pm$2&\footnotesize0\tiny$\pm$0&\footnotesize ---&\footnotesize3\tiny$\pm$1&\footnotesize0\tiny$\pm$0&\footnotesize ---&\footnotesize3\tiny$\pm$1&\footnotesize0\tiny$\pm$0&\footnotesize ---&\footnotesize3\tiny$\pm$2&\footnotesize0\tiny$\pm$0&\footnotesize ---&\footnotesize3\tiny$\pm$1&\footnotesize0\tiny$\pm$0&\footnotesize ---\\ 
SFA&\footnotesize1\tiny$\pm$0&\footnotesize6\tiny$\pm$0&\footnotesize30\tiny$\pm$5&\footnotesize1\tiny$\pm$0&\footnotesize14\tiny$\pm$0&\footnotesize68\tiny$\pm$3&\footnotesize1\tiny$\pm$0&\footnotesize28\tiny$\pm$1&\footnotesize243\tiny$\pm$8&\footnotesize1\tiny$\pm$0&\footnotesize8\tiny$\pm$0&\footnotesize38\tiny$\pm$5&\footnotesize1\tiny$\pm$0&\footnotesize15\tiny$\pm$0&\footnotesize72\tiny$\pm$3\\ 
FA&\footnotesize1\tiny$\pm$0&\footnotesize6\tiny$\pm$0&\footnotesize18\tiny$\pm$3&\footnotesize1\tiny$\pm$0&\footnotesize14\tiny$\pm$0&\footnotesize50\tiny$\pm$2&\footnotesize1\tiny$\pm$0&\footnotesize28\tiny$\pm$1&\footnotesize182\tiny$\pm$7&\footnotesize1\tiny$\pm$0&\footnotesize8\tiny$\pm$0&\footnotesize24\tiny$\pm$4&\footnotesize1\tiny$\pm$0&\footnotesize15\tiny$\pm$0&\footnotesize52\tiny$\pm$2\\ 
PCA&\footnotesize4\tiny$\pm$0&\footnotesize0\tiny$\pm$0&\footnotesize ---&\footnotesize2\tiny$\pm$0&\footnotesize0\tiny$\pm$0&\footnotesize ---&\footnotesize1\tiny$\pm$0&\footnotesize0\tiny$\pm$0&\footnotesize ---&\footnotesize4\tiny$\pm$0&\footnotesize0\tiny$\pm$0&\footnotesize ---&\footnotesize2\tiny$\pm$0&\footnotesize0\tiny$\pm$0&\footnotesize ---\\ 
\midrule &\multicolumn{3}{c}{D6} 
 &\multicolumn{3}{c}{D7} 
 &\multicolumn{3}{c}{D8} 
 &\multicolumn{3}{c}{D9} 
 &\multicolumn{3}{c}{average } \\ 
 \cmidrule(r){2-4} 
 \cmidrule(lr){5-7} 
 \cmidrule(lr){8-10} 
 \cmidrule(lr){11-13} 
 \cmidrule(l){14-16} 
 & SP & ER & CO 
 & SP & ER & CO 
 & SP & ER & CO 
 & SP & ER & CO 
 & SP & ER & CO\\ 
RFN&\footnotesize82\tiny$\pm$1&\footnotesize120\tiny$\pm$16&\footnotesize60\tiny$\pm$13&\footnotesize80\tiny$\pm$1&\footnotesize18\tiny$\pm$2&\footnotesize1\tiny$\pm$0&\footnotesize82\tiny$\pm$1&\footnotesize61\tiny$\pm$7&\footnotesize15\tiny$\pm$3&\footnotesize82\tiny$\pm$1&\footnotesize122\tiny$\pm$13&\footnotesize60\tiny$\pm$11&\footnotesize81\tiny$\pm$1&\footnotesize68\tiny$\pm$9&\footnotesize26\tiny$\pm$6\\ 
RFNn&\footnotesize80\tiny$\pm$0&\footnotesize329\tiny$\pm$7&\footnotesize137\tiny$\pm$6&\footnotesize78\tiny$\pm$0&\footnotesize49\tiny$\pm$3&\footnotesize4\tiny$\pm$1&\footnotesize80\tiny$\pm$0&\footnotesize165\tiny$\pm$4&\footnotesize34\tiny$\pm$1&\footnotesize80\tiny$\pm$0&\footnotesize325\tiny$\pm$7&\footnotesize134\tiny$\pm$6&\footnotesize79\tiny$\pm$0&\footnotesize185\tiny$\pm$5&\footnotesize59\tiny$\pm$3\\ 
DAE&\footnotesize70\tiny$\pm$0&\footnotesize261\tiny$\pm$2&\footnotesize ---&\footnotesize68\tiny$\pm$0&\footnotesize39\tiny$\pm$2&\footnotesize ---&\footnotesize70\tiny$\pm$0&\footnotesize132\tiny$\pm$1&\footnotesize ---&\footnotesize70\tiny$\pm$0&\footnotesize259\tiny$\pm$2&\footnotesize ---&\footnotesize69\tiny$\pm$0&\footnotesize147\tiny$\pm$2&\footnotesize ---\\ 
RBM&\footnotesize4\tiny$\pm$0&\footnotesize517\tiny$\pm$6&\footnotesize ---&\footnotesize12\tiny$\pm$1&\footnotesize71\tiny$\pm$2&\footnotesize ---&\footnotesize4\tiny$\pm$0&\footnotesize261\tiny$\pm$3&\footnotesize ---&\footnotesize4\tiny$\pm$0&\footnotesize512\tiny$\pm$5&\footnotesize ---&\footnotesize7\tiny$\pm$1&\footnotesize287\tiny$\pm$4&\footnotesize ---\\ 
FAsp&\footnotesize56\tiny$\pm$1&\footnotesize999\tiny$\pm$99&\footnotesize999\tiny$\pm$99&\footnotesize73\tiny$\pm$0&\footnotesize149\tiny$\pm$28&\footnotesize237\tiny$\pm$62&\footnotesize62\tiny$\pm$0&\footnotesize999\tiny$\pm$34&\footnotesize999\tiny$\pm$99&\footnotesize56\tiny$\pm$0&\footnotesize999\tiny$\pm$99&\footnotesize999\tiny$\pm$99&\footnotesize63\tiny$\pm$0&\footnotesize999\tiny$\pm$65&\footnotesize999\tiny$\pm$99\\ 
FAlap&\footnotesize6\tiny$\pm$0&\footnotesize74\tiny$\pm$6&\footnotesize999\tiny$\pm$91&\footnotesize6\tiny$\pm$0&\footnotesize22\tiny$\pm$3&\footnotesize134\tiny$\pm$14&\footnotesize6\tiny$\pm$0&\footnotesize37\tiny$\pm$2&\footnotesize733\tiny$\pm$28&\footnotesize6\tiny$\pm$0&\footnotesize73\tiny$\pm$6&\footnotesize999\tiny$\pm$84&\footnotesize6\tiny$\pm$0&\footnotesize46\tiny$\pm$4&\footnotesize985\tiny$\pm$45\\ 
ICA&\footnotesize3\tiny$\pm$1&\footnotesize0\tiny$\pm$0&\footnotesize ---&\footnotesize3\tiny$\pm$2&\footnotesize0\tiny$\pm$0&\footnotesize ---&\footnotesize3\tiny$\pm$1&\footnotesize0\tiny$\pm$0&\footnotesize ---&\footnotesize3\tiny$\pm$1&\footnotesize0\tiny$\pm$0&\footnotesize ---&\footnotesize3\tiny$\pm$1&\footnotesize0\tiny$\pm$0&\footnotesize ---\\ 
SFA&\footnotesize1\tiny$\pm$0&\footnotesize28\tiny$\pm$1&\footnotesize247\tiny$\pm$8&\footnotesize1\tiny$\pm$0&\footnotesize5\tiny$\pm$0&\footnotesize21\tiny$\pm$5&\footnotesize1\tiny$\pm$0&\footnotesize14\tiny$\pm$0&\footnotesize64\tiny$\pm$2&\footnotesize1\tiny$\pm$0&\footnotesize27\tiny$\pm$1&\footnotesize240\tiny$\pm$7&\footnotesize1\tiny$\pm$0&\footnotesize16\tiny$\pm$1&\footnotesize114\tiny$\pm$5\\ 
FA&\footnotesize1\tiny$\pm$0&\footnotesize28\tiny$\pm$1&\footnotesize184\tiny$\pm$8&\footnotesize1\tiny$\pm$0&\footnotesize5\tiny$\pm$0&\footnotesize11\tiny$\pm$3&\footnotesize1\tiny$\pm$0&\footnotesize14\tiny$\pm$0&\footnotesize47\tiny$\pm$2&\footnotesize1\tiny$\pm$0&\footnotesize27\tiny$\pm$1&\footnotesize179\tiny$\pm$7&\footnotesize1\tiny$\pm$0&\footnotesize16\tiny$\pm$1&\footnotesize83\tiny$\pm$4\\ 
PCA&\footnotesize1\tiny$\pm$0&\footnotesize0\tiny$\pm$0&\footnotesize ---&\footnotesize4\tiny$\pm$0&\footnotesize0\tiny$\pm$0&\footnotesize ---&\footnotesize2\tiny$\pm$0&\footnotesize0\tiny$\pm$0&\footnotesize ---&\footnotesize1\tiny$\pm$0&\footnotesize0\tiny$\pm$0&\footnotesize ---&\footnotesize2\tiny$\pm$0&\footnotesize0\tiny$\pm$0&\footnotesize ---\\ 
\bottomrule
\end{tabular*}%
\end{sidewaystable}

\begin{sidewaystable}
\centering
\caption{ Comparison for 150 factors / hidden units extracted by RFN, RFN without normalization (RFNn), denoising
  autoencoder (DAE), restricted Boltzmann machines (RBM), factor
  analysis with a very sparse prior (FAsp), factor analysis with a Laplace
prior (FAlap), independent component analysis (ICA), sparse factor
analysis (SFA), factor analysis (FA), and principal component analysis
(PCA) on nine data sets. Criteria are: sparseness of the coding units (SP),
reconstruction error (ER), and  the difference between the empirical and the model covariance
matrix (CO). The lower right column block gives the average 
SP (\%), ER and CO.  The lower right column block gives the average 
SP, ER and CO.  Results reported here, are the mean together with the standard deviation of 100 instances.  The maximal value in the table and the maximal standard deviation was set to 999 and to 99, respectively.  
 \label{S_tab:compare_n150_II}%
}
\begin{tabular*}{\textwidth}{l*{3}{>{\columncolor{mColor1} \raggedleft\arraybackslash}p{2.9em}}*{3}{>{\raggedleft\arraybackslash}p{2.9em}}*{3}{>{\columncolor{mColor1} \raggedleft\arraybackslash}p{2.9em}}*{3}{>{\raggedleft\arraybackslash}p{2.9em}}*{3}{>{\columncolor{mColor1} \raggedleft\arraybackslash}p{2.9em}}}
\toprule

&\multicolumn{3}{c}{D1} 
&\multicolumn{3}{c}{D2} 
&\multicolumn{3}{c}{D3} 
&\multicolumn{3}{c}{D4} 
&\multicolumn{3}{c}{D5} \\
\cmidrule(r){2-4} 
\cmidrule(lr){5-7} 
\cmidrule(lr){8-10} 
\cmidrule(lr){11-13} 
\cmidrule(l){14-16} 
& SP & ER & CO
& SP & ER & CO
& SP & ER & CO
& SP & ER & CO
& SP & ER & CO\\
RFN&\footnotesize83\tiny$\pm$1&\footnotesize7\tiny$\pm$2&\footnotesize0\tiny$\pm$1&\footnotesize86\tiny$\pm$0&\footnotesize15\tiny$\pm$1&\footnotesize3\tiny$\pm$1&\footnotesize86\tiny$\pm$2&\footnotesize33\tiny$\pm$20&\footnotesize18\tiny$\pm$23&\footnotesize83\tiny$\pm$1&\footnotesize9\tiny$\pm$2&\footnotesize1\tiny$\pm$0&\footnotesize86\tiny$\pm$1&\footnotesize15\tiny$\pm$3&\footnotesize4\tiny$\pm$1\\ 
RFNn&\footnotesize79\tiny$\pm$0&\footnotesize48\tiny$\pm$3&\footnotesize4\tiny$\pm$1&\footnotesize81\tiny$\pm$0&\footnotesize129\tiny$\pm$3&\footnotesize21\tiny$\pm$1&\footnotesize81\tiny$\pm$0&\footnotesize250\tiny$\pm$7&\footnotesize80\tiny$\pm$4&\footnotesize78\tiny$\pm$0&\footnotesize60\tiny$\pm$4&\footnotesize6\tiny$\pm$1&\footnotesize81\tiny$\pm$0&\footnotesize131\tiny$\pm$3&\footnotesize22\tiny$\pm$1\\ 
DAE&\footnotesize68\tiny$\pm$0&\footnotesize44\tiny$\pm$2&\footnotesize ---&\footnotesize72\tiny$\pm$0&\footnotesize118\tiny$\pm$1&\footnotesize ---&\footnotesize72\tiny$\pm$0&\footnotesize229\tiny$\pm$2&\footnotesize ---&\footnotesize68\tiny$\pm$0&\footnotesize50\tiny$\pm$2&\footnotesize ---&\footnotesize72\tiny$\pm$0&\footnotesize120\tiny$\pm$2&\footnotesize ---\\ 
RBM&\footnotesize10\tiny$\pm$1&\footnotesize81\tiny$\pm$3&\footnotesize ---&\footnotesize3\tiny$\pm$0&\footnotesize265\tiny$\pm$3&\footnotesize ---&\footnotesize3\tiny$\pm$0&\footnotesize514\tiny$\pm$6&\footnotesize ---&\footnotesize10\tiny$\pm$1&\footnotesize88\tiny$\pm$2&\footnotesize ---&\footnotesize3\tiny$\pm$0&\footnotesize270\tiny$\pm$4&\footnotesize ---\\ 
FAsp&\footnotesize83\tiny$\pm$1&\footnotesize233\tiny$\pm$32&\footnotesize340\tiny$\pm$71&\footnotesize79\tiny$\pm$0&\footnotesize999\tiny$\pm$43&\footnotesize999\tiny$\pm$99&\footnotesize77\tiny$\pm$0&\footnotesize999\tiny$\pm$99&\footnotesize999\tiny$\pm$99&\footnotesize81\tiny$\pm$1&\footnotesize320\tiny$\pm$34&\footnotesize574\tiny$\pm$99&\footnotesize79\tiny$\pm$1&\footnotesize999\tiny$\pm$49&\footnotesize999\tiny$\pm$99\\ 
FAlap&\footnotesize4\tiny$\pm$0&\footnotesize27\tiny$\pm$3&\footnotesize295\tiny$\pm$25&\footnotesize4\tiny$\pm$0&\footnotesize38\tiny$\pm$3&\footnotesize791\tiny$\pm$41&\footnotesize3\tiny$\pm$0&\footnotesize74\tiny$\pm$5&\footnotesize999\tiny$\pm$91&\footnotesize4\tiny$\pm$0&\footnotesize31\tiny$\pm$3&\footnotesize394\tiny$\pm$31&\footnotesize4\tiny$\pm$0&\footnotesize39\tiny$\pm$3&\footnotesize817\tiny$\pm$39\\ 
ICA&\footnotesize3\tiny$\pm$2&\footnotesize0\tiny$\pm$0&\footnotesize ---&\footnotesize3\tiny$\pm$1&\footnotesize0\tiny$\pm$0&\footnotesize ---&\footnotesize3\tiny$\pm$1&\footnotesize0\tiny$\pm$0&\footnotesize ---&\footnotesize3\tiny$\pm$2&\footnotesize0\tiny$\pm$0&\footnotesize ---&\footnotesize3\tiny$\pm$1&\footnotesize0\tiny$\pm$0&\footnotesize ---\\ 
SFA&\footnotesize1\tiny$\pm$0&\footnotesize6\tiny$\pm$0&\footnotesize49\tiny$\pm$7&\footnotesize1\tiny$\pm$0&\footnotesize14\tiny$\pm$0&\footnotesize173\tiny$\pm$4&\footnotesize1\tiny$\pm$0&\footnotesize28\tiny$\pm$1&\footnotesize632\tiny$\pm$10&\footnotesize1\tiny$\pm$0&\footnotesize8\tiny$\pm$0&\footnotesize61\tiny$\pm$7&\footnotesize1\tiny$\pm$0&\footnotesize15\tiny$\pm$0&\footnotesize181\tiny$\pm$5\\ 
FA&\footnotesize1\tiny$\pm$0&\footnotesize6\tiny$\pm$0&\footnotesize40\tiny$\pm$5&\footnotesize1\tiny$\pm$0&\footnotesize14\tiny$\pm$0&\footnotesize160\tiny$\pm$4&\footnotesize1\tiny$\pm$0&\footnotesize28\tiny$\pm$1&\footnotesize590\tiny$\pm$10&\footnotesize1\tiny$\pm$0&\footnotesize8\tiny$\pm$0&\footnotesize51\tiny$\pm$6&\footnotesize1\tiny$\pm$0&\footnotesize15\tiny$\pm$0&\footnotesize168\tiny$\pm$4\\ 
PCA&\footnotesize4\tiny$\pm$0&\footnotesize0\tiny$\pm$0&\footnotesize ---&\footnotesize2\tiny$\pm$0&\footnotesize0\tiny$\pm$0&\footnotesize ---&\footnotesize1\tiny$\pm$0&\footnotesize0\tiny$\pm$0&\footnotesize ---&\footnotesize4\tiny$\pm$0&\footnotesize0\tiny$\pm$0&\footnotesize ---&\footnotesize2\tiny$\pm$0&\footnotesize0\tiny$\pm$0&\footnotesize ---\\ 
\midrule &\multicolumn{3}{c}{D6} 
 &\multicolumn{3}{c}{D7} 
 &\multicolumn{3}{c}{D8} 
 &\multicolumn{3}{c}{D9} 
 &\multicolumn{3}{c}{average } \\ 
 \cmidrule(r){2-4} 
 \cmidrule(lr){5-7} 
 \cmidrule(lr){8-10} 
 \cmidrule(lr){11-13} 
 \cmidrule(l){14-16} 
 & SP & ER & CO 
 & SP & ER & CO 
 & SP & ER & CO 
 & SP & ER & CO 
 & SP & ER & CO\\ 
RFN&\footnotesize86\tiny$\pm$1&\footnotesize30\tiny$\pm$13&\footnotesize15\tiny$\pm$16&\footnotesize84\tiny$\pm$2&\footnotesize5\tiny$\pm$3&\footnotesize0\tiny$\pm$1&\footnotesize86\tiny$\pm$0&\footnotesize14\tiny$\pm$1&\footnotesize3\tiny$\pm$1&\footnotesize86\tiny$\pm$1&\footnotesize30\tiny$\pm$8&\footnotesize15\tiny$\pm$9&\footnotesize85\tiny$\pm$1&\footnotesize17\tiny$\pm$6&\footnotesize7\tiny$\pm$6\\ 
RFNn&\footnotesize81\tiny$\pm$0&\footnotesize251\tiny$\pm$6&\footnotesize81\tiny$\pm$3&\footnotesize80\tiny$\pm$0&\footnotesize37\tiny$\pm$3&\footnotesize2\tiny$\pm$0&\footnotesize81\tiny$\pm$0&\footnotesize126\tiny$\pm$3&\footnotesize20\tiny$\pm$1&\footnotesize81\tiny$\pm$0&\footnotesize248\tiny$\pm$6&\footnotesize79\tiny$\pm$3&\footnotesize80\tiny$\pm$0&\footnotesize142\tiny$\pm$4&\footnotesize35\tiny$\pm$2\\ 
DAE&\footnotesize72\tiny$\pm$0&\footnotesize230\tiny$\pm$2&\footnotesize ---&\footnotesize70\tiny$\pm$0&\footnotesize36\tiny$\pm$2&\footnotesize ---&\footnotesize72\tiny$\pm$0&\footnotesize116\tiny$\pm$1&\footnotesize ---&\footnotesize72\tiny$\pm$0&\footnotesize227\tiny$\pm$2&\footnotesize ---&\footnotesize71\tiny$\pm$0&\footnotesize130\tiny$\pm$2&\footnotesize ---\\ 
RBM&\footnotesize3\tiny$\pm$0&\footnotesize516\tiny$\pm$6&\footnotesize ---&\footnotesize8\tiny$\pm$1&\footnotesize71\tiny$\pm$2&\footnotesize ---&\footnotesize3\tiny$\pm$0&\footnotesize260\tiny$\pm$4&\footnotesize ---&\footnotesize3\tiny$\pm$0&\footnotesize511\tiny$\pm$5&\footnotesize ---&\footnotesize5\tiny$\pm$0&\footnotesize286\tiny$\pm$4&\footnotesize ---\\ 
FAsp&\footnotesize77\tiny$\pm$0&\footnotesize999\tiny$\pm$99&\footnotesize999\tiny$\pm$99&\footnotesize84\tiny$\pm$0&\footnotesize149\tiny$\pm$28&\footnotesize168\tiny$\pm$55&\footnotesize80\tiny$\pm$0&\footnotesize999\tiny$\pm$34&\footnotesize999\tiny$\pm$99&\footnotesize77\tiny$\pm$1&\footnotesize999\tiny$\pm$99&\footnotesize999\tiny$\pm$99&\footnotesize80\tiny$\pm$0&\footnotesize999\tiny$\pm$65&\footnotesize999\tiny$\pm$99\\ 
FAlap&\footnotesize3\tiny$\pm$0&\footnotesize74\tiny$\pm$6&\footnotesize999\tiny$\pm$97&\footnotesize4\tiny$\pm$0&\footnotesize22\tiny$\pm$3&\footnotesize198\tiny$\pm$17&\footnotesize4\tiny$\pm$0&\footnotesize37\tiny$\pm$2&\footnotesize768\tiny$\pm$40&\footnotesize3\tiny$\pm$0&\footnotesize73\tiny$\pm$6&\footnotesize999\tiny$\pm$93&\footnotesize4\tiny$\pm$0&\footnotesize46\tiny$\pm$4&\footnotesize976\tiny$\pm$53\\ 
ICA&\footnotesize3\tiny$\pm$1&\footnotesize0\tiny$\pm$0&\footnotesize ---&\footnotesize3\tiny$\pm$2&\footnotesize0\tiny$\pm$0&\footnotesize ---&\footnotesize3\tiny$\pm$1&\footnotesize0\tiny$\pm$0&\footnotesize ---&\footnotesize3\tiny$\pm$1&\footnotesize0\tiny$\pm$0&\footnotesize ---&\footnotesize3\tiny$\pm$1&\footnotesize0\tiny$\pm$0&\footnotesize ---\\ 
SFA&\footnotesize1\tiny$\pm$0&\footnotesize28\tiny$\pm$1&\footnotesize640\tiny$\pm$11&\footnotesize1\tiny$\pm$0&\footnotesize5\tiny$\pm$0&\footnotesize34\tiny$\pm$6&\footnotesize1\tiny$\pm$0&\footnotesize14\tiny$\pm$0&\footnotesize164\tiny$\pm$3&\footnotesize1\tiny$\pm$0&\footnotesize27\tiny$\pm$1&\footnotesize625\tiny$\pm$9&\footnotesize1\tiny$\pm$0&\footnotesize16\tiny$\pm$1&\footnotesize285\tiny$\pm$7\\ 
FA&\footnotesize1\tiny$\pm$0&\footnotesize28\tiny$\pm$1&\footnotesize596\tiny$\pm$10&\footnotesize1\tiny$\pm$0&\footnotesize5\tiny$\pm$0&\footnotesize27\tiny$\pm$5&\footnotesize1\tiny$\pm$0&\footnotesize14\tiny$\pm$0&\footnotesize153\tiny$\pm$3&\footnotesize1\tiny$\pm$0&\footnotesize27\tiny$\pm$1&\footnotesize583\tiny$\pm$9&\footnotesize1\tiny$\pm$0&\footnotesize16\tiny$\pm$1&\footnotesize263\tiny$\pm$6\\ 
PCA&\footnotesize1\tiny$\pm$0&\footnotesize0\tiny$\pm$0&\footnotesize ---&\footnotesize4\tiny$\pm$0&\footnotesize0\tiny$\pm$0&\footnotesize ---&\footnotesize2\tiny$\pm$0&\footnotesize0\tiny$\pm$0&\footnotesize ---&\footnotesize1\tiny$\pm$0&\footnotesize0\tiny$\pm$0&\footnotesize ---&\footnotesize2\tiny$\pm$0&\footnotesize0\tiny$\pm$0&\footnotesize ---\\ 
\bottomrule
\end{tabular*}%
\end{sidewaystable}

\vfill
\clearpage

\section{Data Set II}
\label{S_sec:data2}
This data sets was generate as described in Section \ref{S_sec:data1}, but instead of
drawing the remaining components of the spanning outer product vectors 
from $\Ncal\left(0,0.01\right)$, they were now drawn from $\Ncal\left(0, 0.5\right)$.

\begin{sidewaystable}
\centering
\caption{ Comparison for 50 factors / hidden units extracted by RFN, RFN without normalization (RFNn), denoising
  autoencoder (DAE), restricted Boltzmann machines (RBM), factor
  analysis with a very sparse prior (FAsp), factor analysis with a Laplace
prior (FAlap), independent component analysis (ICA), sparse factor
analysis (SFA), factor analysis (FA), and principal component analysis
(PCA) on nine data sets. Criteria are: sparseness of the coding units (SP),
reconstruction error (ER), and the difference between the empirical and the model covariance
matrix (CO). The lower right column block gives the average 
SP (\%), ER and CO.  Results reported here, are the mean together with the standard deviation of 100 instances. The maximal value in the table and the maximal standard deviation was set to 999 and to 99, respectively. 
 \label{S_tab:compare_n50}%
}
\begin{tabular*}{\textwidth}{l*{3}{>{\columncolor{mColor1} \raggedleft\arraybackslash}p{2.9em}}*{3}{>{\raggedleft\arraybackslash}p{2.9em}}*{3}{>{\columncolor{mColor1} \raggedleft\arraybackslash}p{2.9em}}*{3}{>{\raggedleft\arraybackslash}p{2.9em}}*{3}{>{\columncolor{mColor1} \raggedleft\arraybackslash}p{2.9em}}}
\toprule

&\multicolumn{3}{c}{D1} 
&\multicolumn{3}{c}{D2} 
&\multicolumn{3}{c}{D3} 
&\multicolumn{3}{c}{D4} 
&\multicolumn{3}{c}{D5} \\
\cmidrule(r){2-4} 
\cmidrule(lr){5-7} 
\cmidrule(lr){8-10} 
\cmidrule(lr){11-13} 
\cmidrule(l){14-16} 
& SP & ER & CO
& SP & ER & CO
& SP & ER & CO
& SP & ER & CO
& SP & ER & CO\\
RFN&\footnotesize72\tiny$\pm$1&\footnotesize74\tiny$\pm$2&\footnotesize11\tiny$\pm$1&\footnotesize75\tiny$\pm$0&\footnotesize240\tiny$\pm$3&\footnotesize72\tiny$\pm$2&\footnotesize75\tiny$\pm$0&\footnotesize462\tiny$\pm$5&\footnotesize260\tiny$\pm$6&\footnotesize72\tiny$\pm$1&\footnotesize79\tiny$\pm$2&\footnotesize12\tiny$\pm$1&\footnotesize75\tiny$\pm$0&\footnotesize244\tiny$\pm$3&\footnotesize75\tiny$\pm$2\\ 
RFNn&\footnotesize68\tiny$\pm$1&\footnotesize122\tiny$\pm$5&\footnotesize32\tiny$\pm$4&\footnotesize74\tiny$\pm$0&\footnotesize285\tiny$\pm$4&\footnotesize97\tiny$\pm$3&\footnotesize74\tiny$\pm$0&\footnotesize537\tiny$\pm$7&\footnotesize331\tiny$\pm$8&\footnotesize65\tiny$\pm$1&\footnotesize144\tiny$\pm$6&\footnotesize48\tiny$\pm$6&\footnotesize74\tiny$\pm$0&\footnotesize290\tiny$\pm$4&\footnotesize102\tiny$\pm$4\\ 
DAE&\footnotesize61\tiny$\pm$0&\footnotesize82\tiny$\pm$2&\footnotesize ---&\footnotesize66\tiny$\pm$0&\footnotesize243\tiny$\pm$2&\footnotesize ---&\footnotesize66\tiny$\pm$0&\footnotesize461\tiny$\pm$4&\footnotesize ---&\footnotesize60\tiny$\pm$0&\footnotesize88\tiny$\pm$2&\footnotesize ---&\footnotesize66\tiny$\pm$0&\footnotesize247\tiny$\pm$3&\footnotesize ---\\ 
RBM&\footnotesize22\tiny$\pm$1&\footnotesize106\tiny$\pm$3&\footnotesize ---&\footnotesize11\tiny$\pm$1&\footnotesize301\tiny$\pm$3&\footnotesize ---&\footnotesize10\tiny$\pm$1&\footnotesize566\tiny$\pm$6&\footnotesize ---&\footnotesize22\tiny$\pm$1&\footnotesize113\tiny$\pm$3&\footnotesize ---&\footnotesize11\tiny$\pm$1&\footnotesize308\tiny$\pm$4&\footnotesize ---\\ 
FAsp&\footnotesize37\tiny$\pm$1&\footnotesize469\tiny$\pm$38&\footnotesize999\tiny$\pm$99&\footnotesize40\tiny$\pm$1&\footnotesize999\tiny$\pm$50&\footnotesize999\tiny$\pm$99&\footnotesize40\tiny$\pm$2&\footnotesize999\tiny$\pm$99&\footnotesize999\tiny$\pm$99&\footnotesize37\tiny$\pm$1&\footnotesize610\tiny$\pm$44&\footnotesize999\tiny$\pm$99&\footnotesize40\tiny$\pm$1&\footnotesize999\tiny$\pm$58&\footnotesize999\tiny$\pm$99\\ 
FAlap&\footnotesize4\tiny$\pm$0&\footnotesize50\tiny$\pm$1&\footnotesize392\tiny$\pm$66&\footnotesize4\tiny$\pm$0&\footnotesize228\tiny$\pm$5&\footnotesize135\tiny$\pm$13&\footnotesize5\tiny$\pm$0&\footnotesize443\tiny$\pm$9&\footnotesize406\tiny$\pm$18&\footnotesize4\tiny$\pm$0&\footnotesize51\tiny$\pm$1&\footnotesize477\tiny$\pm$63&\footnotesize4\tiny$\pm$0&\footnotesize230\tiny$\pm$6&\footnotesize147\tiny$\pm$18\\ 
ICA&\footnotesize2\tiny$\pm$0&\footnotesize35\tiny$\pm$0&\footnotesize ---&\footnotesize2\tiny$\pm$0&\footnotesize168\tiny$\pm$2&\footnotesize ---&\footnotesize2\tiny$\pm$0&\footnotesize327\tiny$\pm$4&\footnotesize ---&\footnotesize2\tiny$\pm$0&\footnotesize35\tiny$\pm$0&\footnotesize ---&\footnotesize2\tiny$\pm$0&\footnotesize170\tiny$\pm$2&\footnotesize ---\\ 
SFA&\footnotesize1\tiny$\pm$0&\footnotesize42\tiny$\pm$1&\footnotesize26\tiny$\pm$3&\footnotesize1\tiny$\pm$0&\footnotesize210\tiny$\pm$5&\footnotesize61\tiny$\pm$2&\footnotesize1\tiny$\pm$0&\footnotesize409\tiny$\pm$8&\footnotesize220\tiny$\pm$6&\footnotesize1\tiny$\pm$0&\footnotesize41\tiny$\pm$1&\footnotesize32\tiny$\pm$4&\footnotesize1\tiny$\pm$0&\footnotesize211\tiny$\pm$5&\footnotesize63\tiny$\pm$2\\ 
FA&\footnotesize1\tiny$\pm$0&\footnotesize42\tiny$\pm$1&\footnotesize13\tiny$\pm$2&\footnotesize1\tiny$\pm$0&\footnotesize210\tiny$\pm$4&\footnotesize58\tiny$\pm$2&\footnotesize1\tiny$\pm$0&\footnotesize409\tiny$\pm$8&\footnotesize214\tiny$\pm$6&\footnotesize1\tiny$\pm$0&\footnotesize41\tiny$\pm$1&\footnotesize17\tiny$\pm$2&\footnotesize1\tiny$\pm$0&\footnotesize212\tiny$\pm$5&\footnotesize60\tiny$\pm$2\\ 
PCA&\footnotesize0\tiny$\pm$0&\footnotesize35\tiny$\pm$0&\footnotesize ---&\footnotesize0\tiny$\pm$0&\footnotesize168\tiny$\pm$2&\footnotesize ---&\footnotesize0\tiny$\pm$0&\footnotesize327\tiny$\pm$4&\footnotesize ---&\footnotesize0\tiny$\pm$0&\footnotesize35\tiny$\pm$0&\footnotesize ---&\footnotesize0\tiny$\pm$0&\footnotesize170\tiny$\pm$2&\footnotesize ---\\ 
\midrule &\multicolumn{3}{c}{D6} 
 &\multicolumn{3}{c}{D7} 
 &\multicolumn{3}{c}{D8} 
 &\multicolumn{3}{c}{D9} 
 &\multicolumn{3}{c}{average } \\ 
 \cmidrule(r){2-4} 
 \cmidrule(lr){5-7} 
 \cmidrule(lr){8-10} 
 \cmidrule(lr){11-13} 
 \cmidrule(l){14-16} 
 & SP & ER & CO 
 & SP & ER & CO 
 & SP & ER & CO 
 & SP & ER & CO 
 & SP & ER & CO\\ 
RFN&\footnotesize75\tiny$\pm$0&\footnotesize464\tiny$\pm$5&\footnotesize264\tiny$\pm$6&\footnotesize73\tiny$\pm$0&\footnotesize68\tiny$\pm$2&\footnotesize9\tiny$\pm$1&\footnotesize75\tiny$\pm$0&\footnotesize237\tiny$\pm$3&\footnotesize69\tiny$\pm$1&\footnotesize75\tiny$\pm$0&\footnotesize459\tiny$\pm$5&\footnotesize257\tiny$\pm$6&\footnotesize74\tiny$\pm$0&\footnotesize259\tiny$\pm$3&\footnotesize114\tiny$\pm$3\\ 
RFNn&\footnotesize74\tiny$\pm$0&\footnotesize541\tiny$\pm$6&\footnotesize336\tiny$\pm$8&\footnotesize71\tiny$\pm$1&\footnotesize106\tiny$\pm$4&\footnotesize23\tiny$\pm$3&\footnotesize74\tiny$\pm$0&\footnotesize279\tiny$\pm$3&\footnotesize91\tiny$\pm$2&\footnotesize75\tiny$\pm$0&\footnotesize533\tiny$\pm$6&\footnotesize325\tiny$\pm$8&\footnotesize72\tiny$\pm$1&\footnotesize315\tiny$\pm$5&\footnotesize154\tiny$\pm$5\\ 
DAE&\footnotesize66\tiny$\pm$0&\footnotesize465\tiny$\pm$4&\footnotesize ---&\footnotesize62\tiny$\pm$0&\footnotesize75\tiny$\pm$2&\footnotesize ---&\footnotesize66\tiny$\pm$0&\footnotesize238\tiny$\pm$2&\footnotesize ---&\footnotesize66\tiny$\pm$0&\footnotesize458\tiny$\pm$4&\footnotesize ---&\footnotesize64\tiny$\pm$0&\footnotesize262\tiny$\pm$3&\footnotesize ---\\ 
RBM&\footnotesize10\tiny$\pm$1&\footnotesize570\tiny$\pm$6&\footnotesize ---&\footnotesize20\tiny$\pm$1&\footnotesize97\tiny$\pm$3&\footnotesize ---&\footnotesize11\tiny$\pm$1&\footnotesize294\tiny$\pm$3&\footnotesize ---&\footnotesize10\tiny$\pm$1&\footnotesize562\tiny$\pm$5&\footnotesize ---&\footnotesize14\tiny$\pm$1&\footnotesize324\tiny$\pm$4&\footnotesize ---\\ 
FAsp&\footnotesize41\tiny$\pm$1&\footnotesize999\tiny$\pm$99&\footnotesize999\tiny$\pm$99&\footnotesize38\tiny$\pm$1&\footnotesize335\tiny$\pm$32&\footnotesize999\tiny$\pm$99&\footnotesize41\tiny$\pm$1&\footnotesize999\tiny$\pm$40&\footnotesize999\tiny$\pm$99&\footnotesize41\tiny$\pm$1&\footnotesize999\tiny$\pm$99&\footnotesize999\tiny$\pm$99&\footnotesize39\tiny$\pm$1&\footnotesize999\tiny$\pm$69&\footnotesize999\tiny$\pm$99\\ 
FAlap&\footnotesize5\tiny$\pm$0&\footnotesize447\tiny$\pm$9&\footnotesize413\tiny$\pm$19&\footnotesize4\tiny$\pm$0&\footnotesize49\tiny$\pm$1&\footnotesize292\tiny$\pm$57&\footnotesize4\tiny$\pm$0&\footnotesize227\tiny$\pm$5&\footnotesize123\tiny$\pm$11&\footnotesize5\tiny$\pm$0&\footnotesize443\tiny$\pm$9&\footnotesize401\tiny$\pm$17&\footnotesize4\tiny$\pm$0&\footnotesize241\tiny$\pm$5&\footnotesize310\tiny$\pm$31\\ 
ICA&\footnotesize2\tiny$\pm$0&\footnotesize329\tiny$\pm$4&\footnotesize ---&\footnotesize2\tiny$\pm$0&\footnotesize35\tiny$\pm$0&\footnotesize ---&\footnotesize2\tiny$\pm$0&\footnotesize167\tiny$\pm$2&\footnotesize ---&\footnotesize2\tiny$\pm$0&\footnotesize325\tiny$\pm$4&\footnotesize ---&\footnotesize2\tiny$\pm$0&\footnotesize177\tiny$\pm$2&\footnotesize ---\\ 
SFA&\footnotesize1\tiny$\pm$0&\footnotesize412\tiny$\pm$8&\footnotesize223\tiny$\pm$7&\footnotesize1\tiny$\pm$0&\footnotesize42\tiny$\pm$1&\footnotesize19\tiny$\pm$3&\footnotesize1\tiny$\pm$0&\footnotesize209\tiny$\pm$4&\footnotesize59\tiny$\pm$2&\footnotesize1\tiny$\pm$0&\footnotesize408\tiny$\pm$9&\footnotesize218\tiny$\pm$7&\footnotesize1\tiny$\pm$0&\footnotesize221\tiny$\pm$5&\footnotesize102\tiny$\pm$4\\ 
FA&\footnotesize1\tiny$\pm$0&\footnotesize412\tiny$\pm$8&\footnotesize217\tiny$\pm$7&\footnotesize1\tiny$\pm$0&\footnotesize42\tiny$\pm$1&\footnotesize10\tiny$\pm$1&\footnotesize1\tiny$\pm$0&\footnotesize209\tiny$\pm$4&\footnotesize57\tiny$\pm$2&\footnotesize1\tiny$\pm$0&\footnotesize409\tiny$\pm$9&\footnotesize213\tiny$\pm$7&\footnotesize1\tiny$\pm$0&\footnotesize221\tiny$\pm$5&\footnotesize95\tiny$\pm$3\\ 
PCA&\footnotesize0\tiny$\pm$0&\footnotesize329\tiny$\pm$4&\footnotesize ---&\footnotesize0\tiny$\pm$0&\footnotesize35\tiny$\pm$0&\footnotesize ---&\footnotesize0\tiny$\pm$0&\footnotesize167\tiny$\pm$2&\footnotesize ---&\footnotesize0\tiny$\pm$0&\footnotesize325\tiny$\pm$4&\footnotesize ---&\footnotesize0\tiny$\pm$0&\footnotesize177\tiny$\pm$2&\footnotesize ---\\ 
\bottomrule
\end{tabular*}%
\end{sidewaystable}

\begin{sidewaystable}
\centering
\caption{ Comparison for 100 factors / hidden units extracted by RFN, RFN without normalization (RFNn), denoising
  autoencoder (DAE), restricted Boltzmann machines (RBM), factor
  analysis with a very sparse prior (FAsp), factor analysis with a Laplace
prior (FAlap), independent component analysis (ICA), sparse factor
analysis (SFA), factor analysis (FA), and principal component analysis
(PCA) on nine data sets. Criteria are: sparseness of the coding units (SP),
reconstruction error (ER), and  the difference between the empirical and the model covariance
matrix (CO). The lower right column block gives the average 
SP (\%), ER and CO.  Results reported here, are the mean together with the standard deviation of 100 instances. The maximal value in the table and the maximal standard deviation was set to 999 and to 99, respectively.   
 \label{S_tab:compare_n100}%
}
\begin{tabular*}{\textwidth}{l*{3}{>{\columncolor{mColor1} \raggedleft\arraybackslash}p{2.9em}}*{3}{>{\raggedleft\arraybackslash}p{2.9em}}*{3}{>{\columncolor{mColor1} \raggedleft\arraybackslash}p{2.9em}}*{3}{>{\raggedleft\arraybackslash}p{2.9em}}*{3}{>{\columncolor{mColor1} \raggedleft\arraybackslash}p{2.9em}}}
\toprule

&\multicolumn{3}{c}{D1} 
&\multicolumn{3}{c}{D2} 
&\multicolumn{3}{c}{D3} 
&\multicolumn{3}{c}{D4} 
&\multicolumn{3}{c}{D5} \\
\cmidrule(r){2-4} 
\cmidrule(lr){5-7} 
\cmidrule(lr){8-10} 
\cmidrule(lr){11-13} 
\cmidrule(l){14-16} 
& SP & ER & CO
& SP & ER & CO
& SP & ER & CO
& SP & ER & CO
& SP & ER & CO\\
RFN&\footnotesize76\tiny$\pm$1&\footnotesize34\tiny$\pm$3&\footnotesize4\tiny$\pm$1&\footnotesize82\tiny$\pm$1&\footnotesize67\tiny$\pm$8&\footnotesize18\tiny$\pm$3&\footnotesize82\tiny$\pm$1&\footnotesize124\tiny$\pm$16&\footnotesize63\tiny$\pm$12&\footnotesize75\tiny$\pm$1&\footnotesize38\tiny$\pm$3&\footnotesize5\tiny$\pm$1&\footnotesize82\tiny$\pm$1&\footnotesize69\tiny$\pm$10&\footnotesize19\tiny$\pm$5\\ 
RFNn&\footnotesize71\tiny$\pm$1&\footnotesize110\tiny$\pm$7&\footnotesize25\tiny$\pm$4&\footnotesize79\tiny$\pm$0&\footnotesize180\tiny$\pm$5&\footnotesize42\tiny$\pm$2&\footnotesize80\tiny$\pm$0&\footnotesize331\tiny$\pm$8&\footnotesize139\tiny$\pm$7&\footnotesize65\tiny$\pm$2&\footnotesize143\tiny$\pm$9&\footnotesize47\tiny$\pm$8&\footnotesize79\tiny$\pm$0&\footnotesize185\tiny$\pm$5&\footnotesize45\tiny$\pm$3\\ 
DAE&\footnotesize63\tiny$\pm$0&\footnotesize66\tiny$\pm$2&\footnotesize ---&\footnotesize70\tiny$\pm$0&\footnotesize142\tiny$\pm$2&\footnotesize ---&\footnotesize70\tiny$\pm$0&\footnotesize264\tiny$\pm$3&\footnotesize ---&\footnotesize62\tiny$\pm$0&\footnotesize73\tiny$\pm$2&\footnotesize ---&\footnotesize70\tiny$\pm$0&\footnotesize146\tiny$\pm$2&\footnotesize ---\\ 
RBM&\footnotesize12\tiny$\pm$1&\footnotesize100\tiny$\pm$3&\footnotesize ---&\footnotesize5\tiny$\pm$0&\footnotesize282\tiny$\pm$4&\footnotesize ---&\footnotesize4\tiny$\pm$0&\footnotesize522\tiny$\pm$6&\footnotesize ---&\footnotesize12\tiny$\pm$1&\footnotesize106\tiny$\pm$3&\footnotesize ---&\footnotesize5\tiny$\pm$1&\footnotesize288\tiny$\pm$4&\footnotesize ---\\ 
FAsp&\footnotesize71\tiny$\pm$0&\footnotesize474\tiny$\pm$38&\footnotesize999\tiny$\pm$99&\footnotesize62\tiny$\pm$0&\footnotesize999\tiny$\pm$53&\footnotesize999\tiny$\pm$99&\footnotesize56\tiny$\pm$1&\footnotesize999\tiny$\pm$99&\footnotesize999\tiny$\pm$99&\footnotesize70\tiny$\pm$0&\footnotesize616\tiny$\pm$44&\footnotesize999\tiny$\pm$99&\footnotesize62\tiny$\pm$0&\footnotesize999\tiny$\pm$60&\footnotesize999\tiny$\pm$99\\ 
FAlap&\footnotesize6\tiny$\pm$0&\footnotesize21\tiny$\pm$2&\footnotesize425\tiny$\pm$28&\footnotesize6\tiny$\pm$0&\footnotesize40\tiny$\pm$2&\footnotesize827\tiny$\pm$35&\footnotesize6\tiny$\pm$0&\footnotesize75\tiny$\pm$6&\footnotesize999\tiny$\pm$99&\footnotesize6\tiny$\pm$0&\footnotesize23\tiny$\pm$2&\footnotesize523\tiny$\pm$32&\footnotesize6\tiny$\pm$0&\footnotesize42\tiny$\pm$3&\footnotesize865\tiny$\pm$43\\ 
ICA&\footnotesize3\tiny$\pm$2&\footnotesize0\tiny$\pm$0&\footnotesize ---&\footnotesize3\tiny$\pm$1&\footnotesize0\tiny$\pm$0&\footnotesize ---&\footnotesize3\tiny$\pm$1&\footnotesize0\tiny$\pm$0&\footnotesize ---&\footnotesize3\tiny$\pm$2&\footnotesize0\tiny$\pm$0&\footnotesize ---&\footnotesize3\tiny$\pm$1&\footnotesize0\tiny$\pm$0&\footnotesize ---\\ 
SFA&\footnotesize1\tiny$\pm$0&\footnotesize10\tiny$\pm$0&\footnotesize71\tiny$\pm$7&\footnotesize1\tiny$\pm$0&\footnotesize15\tiny$\pm$0&\footnotesize84\tiny$\pm$4&\footnotesize1\tiny$\pm$0&\footnotesize28\tiny$\pm$1&\footnotesize254\tiny$\pm$8&\footnotesize1\tiny$\pm$0&\footnotesize12\tiny$\pm$0&\footnotesize87\tiny$\pm$8&\footnotesize1\tiny$\pm$0&\footnotesize16\tiny$\pm$0&\footnotesize92\tiny$\pm$5\\ 
FA&\footnotesize1\tiny$\pm$0&\footnotesize10\tiny$\pm$0&\footnotesize48\tiny$\pm$5&\footnotesize1\tiny$\pm$0&\footnotesize15\tiny$\pm$0&\footnotesize59\tiny$\pm$3&\footnotesize1\tiny$\pm$0&\footnotesize28\tiny$\pm$1&\footnotesize189\tiny$\pm$7&\footnotesize1\tiny$\pm$0&\footnotesize12\tiny$\pm$1&\footnotesize61\tiny$\pm$6&\footnotesize1\tiny$\pm$0&\footnotesize16\tiny$\pm$0&\footnotesize64\tiny$\pm$3\\ 
PCA&\footnotesize4\tiny$\pm$0&\footnotesize0\tiny$\pm$0&\footnotesize ---&\footnotesize2\tiny$\pm$0&\footnotesize0\tiny$\pm$0&\footnotesize ---&\footnotesize1\tiny$\pm$0&\footnotesize0\tiny$\pm$0&\footnotesize ---&\footnotesize3\tiny$\pm$0&\footnotesize0\tiny$\pm$0&\footnotesize ---&\footnotesize2\tiny$\pm$0&\footnotesize0\tiny$\pm$0&\footnotesize ---\\ 
\midrule &\multicolumn{3}{c}{D6} 
 &\multicolumn{3}{c}{D7} 
 &\multicolumn{3}{c}{D8} 
 &\multicolumn{3}{c}{D9} 
 &\multicolumn{3}{c}{average } \\ 
 \cmidrule(r){2-4} 
 \cmidrule(lr){5-7} 
 \cmidrule(lr){8-10} 
 \cmidrule(lr){11-13} 
 \cmidrule(l){14-16} 
 & SP & ER & CO 
 & SP & ER & CO 
 & SP & ER & CO 
 & SP & ER & CO 
 & SP & ER & CO\\ 
RFN&\footnotesize82\tiny$\pm$1&\footnotesize127\tiny$\pm$17&\footnotesize65\tiny$\pm$14&\footnotesize77\tiny$\pm$1&\footnotesize30\tiny$\pm$3&\footnotesize3\tiny$\pm$1&\footnotesize82\tiny$\pm$1&\footnotesize64\tiny$\pm$8&\footnotesize17\tiny$\pm$4&\footnotesize82\tiny$\pm$1&\footnotesize123\tiny$\pm$15&\footnotesize62\tiny$\pm$13&\footnotesize80\tiny$\pm$1&\footnotesize75\tiny$\pm$9&\footnotesize28\tiny$\pm$6\\ 
RFNn&\footnotesize80\tiny$\pm$0&\footnotesize334\tiny$\pm$8&\footnotesize141\tiny$\pm$7&\footnotesize74\tiny$\pm$1&\footnotesize86\tiny$\pm$4&\footnotesize14\tiny$\pm$2&\footnotesize79\tiny$\pm$0&\footnotesize174\tiny$\pm$4&\footnotesize39\tiny$\pm$2&\footnotesize80\tiny$\pm$0&\footnotesize329\tiny$\pm$7&\footnotesize137\tiny$\pm$6&\footnotesize76\tiny$\pm$1&\footnotesize208\tiny$\pm$6&\footnotesize70\tiny$\pm$5\\ 
DAE&\footnotesize70\tiny$\pm$0&\footnotesize266\tiny$\pm$2&\footnotesize ---&\footnotesize64\tiny$\pm$0&\footnotesize57\tiny$\pm$2&\footnotesize ---&\footnotesize70\tiny$\pm$0&\footnotesize138\tiny$\pm$1&\footnotesize ---&\footnotesize70\tiny$\pm$0&\footnotesize262\tiny$\pm$2&\footnotesize ---&\footnotesize68\tiny$\pm$0&\footnotesize157\tiny$\pm$2&\footnotesize ---\\ 
RBM&\footnotesize4\tiny$\pm$0&\footnotesize527\tiny$\pm$6&\footnotesize ---&\footnotesize11\tiny$\pm$1&\footnotesize92\tiny$\pm$2&\footnotesize ---&\footnotesize4\tiny$\pm$0&\footnotesize274\tiny$\pm$4&\footnotesize ---&\footnotesize4\tiny$\pm$0&\footnotesize518\tiny$\pm$6&\footnotesize ---&\footnotesize7\tiny$\pm$1&\footnotesize301\tiny$\pm$4&\footnotesize ---\\ 
FAsp&\footnotesize56\tiny$\pm$0&\footnotesize999\tiny$\pm$99&\footnotesize999\tiny$\pm$99&\footnotesize71\tiny$\pm$0&\footnotesize338\tiny$\pm$33&\footnotesize999\tiny$\pm$99&\footnotesize62\tiny$\pm$1&\footnotesize999\tiny$\pm$42&\footnotesize999\tiny$\pm$99&\footnotesize56\tiny$\pm$1&\footnotesize999\tiny$\pm$99&\footnotesize999\tiny$\pm$99&\footnotesize63\tiny$\pm$0&\footnotesize999\tiny$\pm$74&\footnotesize999\tiny$\pm$99\\ 
FAlap&\footnotesize6\tiny$\pm$0&\footnotesize75\tiny$\pm$6&\footnotesize999\tiny$\pm$89&\footnotesize6\tiny$\pm$0&\footnotesize18\tiny$\pm$2&\footnotesize337\tiny$\pm$24&\footnotesize6\tiny$\pm$0&\footnotesize40\tiny$\pm$3&\footnotesize793\tiny$\pm$37&\footnotesize6\tiny$\pm$0&\footnotesize74\tiny$\pm$6&\footnotesize999\tiny$\pm$89&\footnotesize6\tiny$\pm$0&\footnotesize45\tiny$\pm$3&\footnotesize999\tiny$\pm$53\\ 
ICA&\footnotesize3\tiny$\pm$1&\footnotesize0\tiny$\pm$0&\footnotesize ---&\footnotesize3\tiny$\pm$1&\footnotesize0\tiny$\pm$0&\footnotesize ---&\footnotesize3\tiny$\pm$1&\footnotesize0\tiny$\pm$0&\footnotesize ---&\footnotesize3\tiny$\pm$1&\footnotesize0\tiny$\pm$0&\footnotesize ---&\footnotesize3\tiny$\pm$1&\footnotesize0\tiny$\pm$0&\footnotesize ---\\ 
SFA&\footnotesize1\tiny$\pm$0&\footnotesize28\tiny$\pm$1&\footnotesize260\tiny$\pm$9&\footnotesize1\tiny$\pm$0&\footnotesize8\tiny$\pm$0&\footnotesize52\tiny$\pm$7&\footnotesize1\tiny$\pm$0&\footnotesize15\tiny$\pm$0&\footnotesize76\tiny$\pm$3&\footnotesize1\tiny$\pm$0&\footnotesize28\tiny$\pm$1&\footnotesize248\tiny$\pm$7&\footnotesize1\tiny$\pm$0&\footnotesize18\tiny$\pm$1&\footnotesize136\tiny$\pm$6\\ 
FA&\footnotesize1\tiny$\pm$0&\footnotesize28\tiny$\pm$1&\footnotesize193\tiny$\pm$8&\footnotesize1\tiny$\pm$0&\footnotesize8\tiny$\pm$0&\footnotesize33\tiny$\pm$5&\footnotesize1\tiny$\pm$0&\footnotesize15\tiny$\pm$0&\footnotesize54\tiny$\pm$2&\footnotesize1\tiny$\pm$0&\footnotesize28\tiny$\pm$1&\footnotesize185\tiny$\pm$6&\footnotesize1\tiny$\pm$0&\footnotesize18\tiny$\pm$1&\footnotesize99\tiny$\pm$5\\ 
PCA&\footnotesize1\tiny$\pm$0&\footnotesize0\tiny$\pm$0&\footnotesize ---&\footnotesize4\tiny$\pm$0&\footnotesize0\tiny$\pm$0&\footnotesize ---&\footnotesize2\tiny$\pm$0&\footnotesize0\tiny$\pm$0&\footnotesize ---&\footnotesize1\tiny$\pm$0&\footnotesize0\tiny$\pm$0&\footnotesize ---&\footnotesize2\tiny$\pm$0&\footnotesize0\tiny$\pm$0&\footnotesize ---\\ 
\bottomrule
\end{tabular*}%
\end{sidewaystable}

\begin{sidewaystable}
\centering
\caption{ Comparison for 150 factors / hidden units extracted by RFN, RFN without normalization (RFNn), denoising
  autoencoder (DAE), restricted Boltzmann machines (RBM), factor
  analysis with a very sparse prior (FAsp), factor analysis with a Laplace
prior (FAlap), independent component analysis (ICA), sparse factor
analysis (SFA), factor analysis (FA), and principal component analysis
(PCA) on nine data sets. Criteria are: sparseness of the factors (SP) reported in \%,
reconstruction error (ER), and  the difference between the empirical and the model covariance
matrix (CO). The lower right column block gives the average 
SP (\%), ER and CO. Results reported here, are the mean together with the standard deviation of 100 instances. The maximal value in the table and the maximal standard deviation was set to 999 and to 99, respectively.  
 \label{S_tab:compare_n150}%
}
\begin{tabular*}{\textwidth}{l*{3}{>{\columncolor{mColor1} \raggedleft\arraybackslash}p{2.9em}}*{3}{>{\raggedleft\arraybackslash}p{2.9em}}*{3}{>{\columncolor{mColor1} \raggedleft\arraybackslash}p{2.9em}}*{3}{>{\raggedleft\arraybackslash}p{2.9em}}*{3}{>{\columncolor{mColor1} \raggedleft\arraybackslash}p{2.9em}}}
\toprule

&\multicolumn{3}{c}{D1} 
&\multicolumn{3}{c}{D2} 
&\multicolumn{3}{c}{D3} 
&\multicolumn{3}{c}{D4} 
&\multicolumn{3}{c}{D5} \\
\cmidrule(r){2-4} 
\cmidrule(lr){5-7} 
\cmidrule(lr){8-10} 
\cmidrule(lr){11-13} 
\cmidrule(l){14-16} 
& SP & ER & CO
& SP & ER & CO
& SP & ER & CO
& SP & ER & CO
& SP & ER & CO\\
RFN&\footnotesize81\tiny$\pm$1&\footnotesize12\tiny$\pm$2&\footnotesize1\tiny$\pm$1&\footnotesize86\tiny$\pm$0&\footnotesize16\tiny$\pm$1&\footnotesize4\tiny$\pm$1&\footnotesize86\tiny$\pm$0&\footnotesize29\tiny$\pm$4&\footnotesize15\tiny$\pm$5&\footnotesize80\tiny$\pm$1&\footnotesize15\tiny$\pm$5&\footnotesize2\tiny$\pm$2&\footnotesize86\tiny$\pm$1&\footnotesize17\tiny$\pm$5&\footnotesize5\tiny$\pm$3\\ 
RFNn&\footnotesize72\tiny$\pm$1&\footnotesize100\tiny$\pm$8&\footnotesize19\tiny$\pm$4&\footnotesize80\tiny$\pm$0&\footnotesize137\tiny$\pm$4&\footnotesize24\tiny$\pm$1&\footnotesize81\tiny$\pm$0&\footnotesize254\tiny$\pm$6&\footnotesize83\tiny$\pm$4&\footnotesize66\tiny$\pm$0&\footnotesize113\tiny$\pm$3&\footnotesize52\tiny$\pm$5&\footnotesize80\tiny$\pm$0&\footnotesize141\tiny$\pm$4&\footnotesize26\tiny$\pm$2\\ 
DAE&\footnotesize64\tiny$\pm$0&\footnotesize62\tiny$\pm$2&\footnotesize ---&\footnotesize71\tiny$\pm$0&\footnotesize125\tiny$\pm$2&\footnotesize ---&\footnotesize72\tiny$\pm$0&\footnotesize232\tiny$\pm$2&\footnotesize ---&\footnotesize63\tiny$\pm$0&\footnotesize69\tiny$\pm$2&\footnotesize ---&\footnotesize71\tiny$\pm$0&\footnotesize129\tiny$\pm$2&\footnotesize ---\\ 
RBM&\footnotesize8\tiny$\pm$0&\footnotesize101\tiny$\pm$3&\footnotesize ---&\footnotesize4\tiny$\pm$0&\footnotesize282\tiny$\pm$4&\footnotesize ---&\footnotesize3\tiny$\pm$0&\footnotesize521\tiny$\pm$6&\footnotesize ---&\footnotesize8\tiny$\pm$0&\footnotesize106\tiny$\pm$3&\footnotesize ---&\footnotesize4\tiny$\pm$0&\footnotesize289\tiny$\pm$4&\footnotesize ---\\ 
FAsp&\footnotesize81\tiny$\pm$1&\footnotesize474\tiny$\pm$38&\footnotesize999\tiny$\pm$99&\footnotesize79\tiny$\pm$0&\footnotesize999\tiny$\pm$53&\footnotesize999\tiny$\pm$99&\footnotesize77\tiny$\pm$1&\footnotesize999\tiny$\pm$99&\footnotesize999\tiny$\pm$99&\footnotesize80\tiny$\pm$1&\footnotesize616\tiny$\pm$44&\footnotesize999\tiny$\pm$99&\footnotesize79\tiny$\pm$1&\footnotesize999\tiny$\pm$60&\footnotesize999\tiny$\pm$99\\ 
FAlap&\footnotesize4\tiny$\pm$0&\footnotesize21\tiny$\pm$2&\footnotesize607\tiny$\pm$34&\footnotesize4\tiny$\pm$0&\footnotesize40\tiny$\pm$2&\footnotesize879\tiny$\pm$40&\footnotesize3\tiny$\pm$0&\footnotesize75\tiny$\pm$6&\footnotesize999\tiny$\pm$96&\footnotesize4\tiny$\pm$0&\footnotesize23\tiny$\pm$2&\footnotesize749\tiny$\pm$42&\footnotesize4\tiny$\pm$0&\footnotesize42\tiny$\pm$3&\footnotesize926\tiny$\pm$45\\ 
ICA&\footnotesize3\tiny$\pm$2&\footnotesize0\tiny$\pm$0&\footnotesize ---&\footnotesize3\tiny$\pm$1&\footnotesize0\tiny$\pm$0&\footnotesize ---&\footnotesize3\tiny$\pm$1&\footnotesize0\tiny$\pm$0&\footnotesize ---&\footnotesize3\tiny$\pm$2&\footnotesize0\tiny$\pm$0&\footnotesize ---&\footnotesize3\tiny$\pm$1&\footnotesize0\tiny$\pm$0&\footnotesize ---\\ 
SFA&\footnotesize1\tiny$\pm$0&\footnotesize10\tiny$\pm$0&\footnotesize103\tiny$\pm$9&\footnotesize1\tiny$\pm$0&\footnotesize15\tiny$\pm$0&\footnotesize204\tiny$\pm$7&\footnotesize1\tiny$\pm$0&\footnotesize28\tiny$\pm$1&\footnotesize656\tiny$\pm$12&\footnotesize1\tiny$\pm$0&\footnotesize12\tiny$\pm$0&\footnotesize126\tiny$\pm$10&\footnotesize1\tiny$\pm$0&\footnotesize16\tiny$\pm$0&\footnotesize220\tiny$\pm$8\\ 
FA&\footnotesize1\tiny$\pm$0&\footnotesize10\tiny$\pm$0&\footnotesize87\tiny$\pm$8&\footnotesize1\tiny$\pm$0&\footnotesize15\tiny$\pm$0&\footnotesize187\tiny$\pm$5&\footnotesize1\tiny$\pm$0&\footnotesize28\tiny$\pm$1&\footnotesize611\tiny$\pm$11&\footnotesize1\tiny$\pm$0&\footnotesize12\tiny$\pm$1&\footnotesize108\tiny$\pm$9&\footnotesize1\tiny$\pm$0&\footnotesize16\tiny$\pm$0&\footnotesize200\tiny$\pm$6\\ 
PCA&\footnotesize4\tiny$\pm$0&\footnotesize0\tiny$\pm$0&\footnotesize ---&\footnotesize2\tiny$\pm$0&\footnotesize0\tiny$\pm$0&\footnotesize ---&\footnotesize1\tiny$\pm$0&\footnotesize0\tiny$\pm$0&\footnotesize ---&\footnotesize3\tiny$\pm$0&\footnotesize0\tiny$\pm$0&\footnotesize ---&\footnotesize2\tiny$\pm$0&\footnotesize0\tiny$\pm$0&\footnotesize ---\\ 
\midrule &\multicolumn{3}{c}{D6} 
 &\multicolumn{3}{c}{D7} 
 &\multicolumn{3}{c}{D8} 
 &\multicolumn{3}{c}{D9} 
 &\multicolumn{3}{c}{average } \\ 
 \cmidrule(r){2-4} 
 \cmidrule(lr){5-7} 
 \cmidrule(lr){8-10} 
 \cmidrule(lr){11-13} 
 \cmidrule(l){14-16} 
 & SP & ER & CO 
 & SP & ER & CO 
 & SP & ER & CO 
 & SP & ER & CO 
 & SP & ER & CO\\ 
RFN&\footnotesize86\tiny$\pm$1&\footnotesize29\tiny$\pm$7&\footnotesize15\tiny$\pm$6&\footnotesize82\tiny$\pm$1&\footnotesize10\tiny$\pm$3&\footnotesize1\tiny$\pm$1&\footnotesize86\tiny$\pm$1&\footnotesize17\tiny$\pm$10&\footnotesize5\tiny$\pm$9&\footnotesize86\tiny$\pm$1&\footnotesize31\tiny$\pm$19&\footnotesize16\tiny$\pm$13&\footnotesize84\tiny$\pm$1&\footnotesize20\tiny$\pm$6&\footnotesize7\tiny$\pm$4\\ 
RFNn&\footnotesize81\tiny$\pm$0&\footnotesize255\tiny$\pm$6&\footnotesize84\tiny$\pm$3&\footnotesize76\tiny$\pm$1&\footnotesize74\tiny$\pm$5&\footnotesize9\tiny$\pm$2&\footnotesize81\tiny$\pm$0&\footnotesize133\tiny$\pm$3&\footnotesize23\tiny$\pm$1&\footnotesize81\tiny$\pm$0&\footnotesize250\tiny$\pm$7&\footnotesize81\tiny$\pm$4&\footnotesize77\tiny$\pm$0&\footnotesize162\tiny$\pm$5&\footnotesize45\tiny$\pm$3\\ 
DAE&\footnotesize72\tiny$\pm$0&\footnotesize234\tiny$\pm$2&\footnotesize ---&\footnotesize65\tiny$\pm$0&\footnotesize53\tiny$\pm$2&\footnotesize ---&\footnotesize72\tiny$\pm$0&\footnotesize122\tiny$\pm$1&\footnotesize ---&\footnotesize72\tiny$\pm$0&\footnotesize230\tiny$\pm$2&\footnotesize ---&\footnotesize69\tiny$\pm$0&\footnotesize140\tiny$\pm$2&\footnotesize ---\\ 
RBM&\footnotesize3\tiny$\pm$0&\footnotesize525\tiny$\pm$6&\footnotesize ---&\footnotesize8\tiny$\pm$0&\footnotesize93\tiny$\pm$3&\footnotesize ---&\footnotesize3\tiny$\pm$0&\footnotesize273\tiny$\pm$4&\footnotesize ---&\footnotesize3\tiny$\pm$0&\footnotesize517\tiny$\pm$6&\footnotesize ---&\footnotesize5\tiny$\pm$0&\footnotesize301\tiny$\pm$4&\footnotesize ---\\ 
FAsp&\footnotesize77\tiny$\pm$1&\footnotesize999\tiny$\pm$99&\footnotesize999\tiny$\pm$99&\footnotesize81\tiny$\pm$1&\footnotesize338\tiny$\pm$33&\footnotesize673\tiny$\pm$99&\footnotesize79\tiny$\pm$0&\footnotesize999\tiny$\pm$42&\footnotesize999\tiny$\pm$99&\footnotesize77\tiny$\pm$1&\footnotesize999\tiny$\pm$99&\footnotesize999\tiny$\pm$99&\footnotesize79\tiny$\pm$1&\footnotesize999\tiny$\pm$74&\footnotesize999\tiny$\pm$99\\ 
FAlap&\footnotesize3\tiny$\pm$0&\footnotesize75\tiny$\pm$6&\footnotesize999\tiny$\pm$94&\footnotesize4\tiny$\pm$0&\footnotesize18\tiny$\pm$2&\footnotesize479\tiny$\pm$31&\footnotesize4\tiny$\pm$0&\footnotesize40\tiny$\pm$3&\footnotesize831\tiny$\pm$43&\footnotesize3\tiny$\pm$0&\footnotesize74\tiny$\pm$6&\footnotesize999\tiny$\pm$95&\footnotesize4\tiny$\pm$0&\footnotesize45\tiny$\pm$3&\footnotesize999\tiny$\pm$58\\ 
ICA&\footnotesize3\tiny$\pm$1&\footnotesize0\tiny$\pm$0&\footnotesize ---&\footnotesize3\tiny$\pm$1&\footnotesize0\tiny$\pm$0&\footnotesize ---&\footnotesize3\tiny$\pm$1&\footnotesize0\tiny$\pm$0&\footnotesize ---&\footnotesize3\tiny$\pm$1&\footnotesize0\tiny$\pm$0&\footnotesize ---&\footnotesize3\tiny$\pm$1&\footnotesize0\tiny$\pm$0&\footnotesize ---\\ 
SFA&\footnotesize1\tiny$\pm$0&\footnotesize28\tiny$\pm$1&\footnotesize668\tiny$\pm$12&\footnotesize1\tiny$\pm$0&\footnotesize8\tiny$\pm$0&\footnotesize78\tiny$\pm$8&\footnotesize1\tiny$\pm$0&\footnotesize15\tiny$\pm$0&\footnotesize188\tiny$\pm$5&\footnotesize1\tiny$\pm$0&\footnotesize28\tiny$\pm$1&\footnotesize644\tiny$\pm$9&\footnotesize1\tiny$\pm$0&\footnotesize18\tiny$\pm$1&\footnotesize321\tiny$\pm$9\\ 
FA&\footnotesize1\tiny$\pm$0&\footnotesize28\tiny$\pm$1&\footnotesize622\tiny$\pm$11&\footnotesize1\tiny$\pm$0&\footnotesize8\tiny$\pm$0&\footnotesize64\tiny$\pm$7&\footnotesize1\tiny$\pm$0&\footnotesize15\tiny$\pm$0&\footnotesize173\tiny$\pm$4&\footnotesize1\tiny$\pm$0&\footnotesize28\tiny$\pm$1&\footnotesize599\tiny$\pm$9&\footnotesize1\tiny$\pm$0&\footnotesize18\tiny$\pm$1&\footnotesize294\tiny$\pm$8\\ 
PCA&\footnotesize1\tiny$\pm$0&\footnotesize0\tiny$\pm$0&\footnotesize ---&\footnotesize4\tiny$\pm$0&\footnotesize0\tiny$\pm$0&\footnotesize ---&\footnotesize2\tiny$\pm$0&\footnotesize0\tiny$\pm$0&\footnotesize ---&\footnotesize1\tiny$\pm$0&\footnotesize0\tiny$\pm$0&\footnotesize ---&\footnotesize2\tiny$\pm$0&\footnotesize0\tiny$\pm$0&\footnotesize ---\\ 
\bottomrule
\end{tabular*}%
\end{sidewaystable}

\vfill
\clearpage

\section{RFN Pretraining for Convolution Nets}
\label{S_sec:ConvNets}
We assess the performance of RFN {\em first layer} pretraining
on {\em CIFAR-10} and {\em CIFAR-100} for
three deep convolutional network architectures:
(i) the AlexNet \cite{Krizhevsky:12},
(ii) Deeply Supervised Networks (DSN) \cite{Lee:14}, and
(iii) our 5-Convolution-Network-In-Network (5C-NIN).

Both CIFAR datasets contain 60k 32x32 RGB-color images, which were divided
into 50k train and 10k test sets, split between 10 (CIFAR10) and
100 (CIFAR100) categories.
Both datasets are preprocessed  by global contrast normalization and ZCA whitening \cite{Goodfellow:13}.
Additionally, the datasets were augmented by padding the images with
four zero pixels at all borders. For data augmentation, at the beginning
of every epoch, images in the training set were distorted by random
translation and random flipping in horizontal and vertical directions.
For the AlexNet, we neither preprocessed nor augmented the datasets.

Inspired by the Network In Network approach \cite{Min:13}, we constructed a
5-Convolution-Network-In-Network (5C-NIN) architecture
with five convolutional layers,
each followed by a 2x2 max-pooling
layer (stride 1) and a multilayer perceptron (MLP) convolutional layer.
ReLUs were used for the convolutional layers and dropout for regularization.
For weight initialization, learning rates, and learning policies we used same strategy
as in the AlexNet \cite{Krizhevsky:09}.
The networks were trained
using mini-batches of size 100 and 128 for 5C-NIN and AlexNet, respectively.

For RFN pretraining, we randomly extracted 5x5 patches from the
training data to construct 192 filters for DSN and 5C-NIN while 32
for AlexNet. These filters constitute
the first convolutional layer of each network which is then trained
using default setting.
For assessing the improvement by RFNs, we repeated training with
randomly initialized weights in the first layer.
The results are presented in Tab.~\ref{tab:tab_res2}.
For comparison, the lower panel of the table reports the performance
of the currently top performing networks: Network In Network (NIN, \cite{Min:13}),
Maxout Networks (MN, \cite{Goodfellow:13}) and DeepCNiN \cite{Graham:14}.
{\em In all cases pretraining with RFNs
decreases the test error rate.}

\begin{table}[th!]
\begin{center}
\caption{The upper panel shows results of convolutional deep networks
with first layer pretrained by RFN (``RFN'') and with first layer randomly
initialized (``org'').
The first column gives the network architecture,
namely, AlexNet,
Deeply Supervised Networks (DSN),
and our 5-Convolution-Network-In-Network (5C-NIN).
The test error rates are reported (for CIFAR-100 DSN model was missing).
Currently best
performing networks Network In Network (NIN),
Maxout Networks (MN), and DeepCNiN are reported in the lower panel.
In all cases pretraining with RFNs decreased the test error rate.
}
\label{tab:tab_res2}%
\begin{tabular}{*{1}{>{\raggedright\arraybackslash}p{3.8em}}*{4}{>{\raggedleft\arraybackslash}p{1.5em}}*{1}{>{\raggedleft\arraybackslash}p{3.5em}}}
\toprule[1pt]
\addlinespace[2pt]
Dataset & \multicolumn{2}{c} {CIFAR-10 } & \multicolumn{2}{c} {CIFAR-100} & \\[-0.2ex]
 \cmidrule(lr){2-3}\cmidrule(rl){4-5}
 &  org & RFN & org & RFN & augmented \\[-0.2ex]
\toprule[1pt]
AlexNet  & 18.21 & 18.04 & 46.18 & 45.80&\\[-0.2ex]
DSN  & 7.97 & 7.74   & 34.57 & - & $\surd$   \\[-0.2ex]
5C-NIN &  7.81 & 7.63 & 29.96 & 29.75 & $\surd$ \\[-0.2ex]
\addlinespace[1pt]
\midrule[1pt]
\addlinespace[2pt]
NIN &  8.81 & - & 35.68 & - & $\surd$ \\[-0.2ex]
MN &  9.38 & - & 38.57 & - & $\surd$ \\[-0.2ex]
DeepCNiN &  6.28  & - & 24.30 & - & $\surd$ \\[-0.2ex]
\bottomrule
\end{tabular}
\end{center}
\end{table} 

\vfill
\clearpage
\small{
\bibliographystyle{unsrt}

\begin{thebibliography}{10}

\bibitem{Hinton:06}
G.~E. Hinton and R.~Salakhutdinov.
\newblock Reducing the dimensionality of data with neural networks.
\newblock {\em Science}, 313(5786):504--507, 2006.

\bibitem{Bengio:07_short}
Y.~Bengio, P.~Lamblin, D.~Popovici, and H.~Larochelle.
\newblock Greedy layer-wise training of deep networks.
\newblock In B.~Sch\"{o}lkopf, J.~C. Platt, and T.~Hoffman, editors, {\em
  NIPS}, pages 153--160. MIT Press, 2007.

\bibitem{Schmidhuber:15}
J.~Schmidhuber.
\newblock Deep learning in neural networks: An overview.
\newblock {\em Neural Networks}, 61:85--117, 2015.

\bibitem{LeCun:15}
Y.~LeCun, Y.~Bengio, and G.~Hinton.
\newblock {Deep} learning.
\newblock {\em Nature}, 521(7553):436--444, 2015.

\bibitem{Nair:10_short}
V.~Nair and G.~E. Hinton.
\newblock Rectified linear units improve restricted {Boltzmann} machines.
\newblock In {\em ICML}, pages 807--814. Omnipress 2010, ISBN
  978-1-60558-907-7, 2010.

\bibitem{Glorot:11_short}
X.~Glorot, A.~Bordes, and Y.~Bengio.
\newblock Deep sparse rectifier neural networks.
\newblock In {\em AISTATS}, volume~15, pages 315--323, 2011.

\bibitem{Srivastava:14}
N.~Srivastava, G.~Hinton, A.~Krizhevsky, I.~Sutskever, and R.~Salakhutdinov.
\newblock Dropout: {A} simple way to prevent neural networks from overfitting.
\newblock {\em Journal of Machine Learning Research}, 15:1929--1958, 2014.

\bibitem{Hochreiter:10s}
S.~Hochreiter, U.~Bodenhofer, et~al.
\newblock {FABIA:} factor analysis for bicluster acquisition.
\newblock {\em Bioinformatics}, 26(12):1520--1527, 2010.

\bibitem{Hochreiter:13}
S.~Hochreiter.
\newblock {HapFABIA}: Identification of very short segments of identity by
  descent characterized by rare variants in large sequencing data.
\newblock {\em Nucleic Acids Res.}, 41(22):e202, 2013.

\bibitem{Frey:99}
B.~J. Frey and G.~E. Hinton.
\newblock Variational learning in nonlinear {Gaussian} belief networks.
\newblock {\em Neural Computation}, 11(1):193--214, 1999.

\bibitem{Harva:07}
M.~Harva and A.~Kaban.
\newblock Variational learning for rectified factor analysis.
\newblock {\em Signal Processing}, 87(3):509--527, 2007.

\bibitem{Ganchev:10}
K.~Ganchev, J.~Graca, J.~Gillenwater, and B.~Taskar.
\newblock Posterior regularization for structured latent variable models.
\newblock {\em Journal of Machine Learning Research}, 11:2001--2049, 2010.

\bibitem{Palmer:06_short}
J.~Palmer, D.~Wipf, K.~Kreutz-Delgado, and B.~Rao.
\newblock Variational {EM} algorithms for non-{Gaussian} latent variable
  models.
\newblock In {\em NIPS}, volume~18, pages 1059--1066, 2006.

\bibitem{Bertsekas:76}
D.~P. Bertsekas.
\newblock On the {Goldstein-Levitin-Polyak} gradient projection method.
\newblock {\em IEEE Trans. Automat. Control}, 21:174--184, 1976.

\bibitem{Kelley:99}
C.~T. Kelley.
\newblock {\em Iterative Methods for Optimization}.
\newblock Society for Industrial and Applied Mathematics (SIAM), Philadelphia,
  1999.

\bibitem{Bertsekas:82}
D.~P. Bertsekas.
\newblock Projected {Newton} methods for optimization problems with simple
  constraints.
\newblock {\em SIAM J. Control Optim.}, 20:221--246, 1982.

\bibitem{Abadie:69}
J.~Abadie and J.~Carpentier.
\newblock {\em Optimization}, chapter Generalization of the {Wolfe} Reduced
  Gradient Method to the Case of Nonlinear Constraints.
\newblock Academic Press, 1969.

\bibitem{Rosen:61}
J.~B. Rosen.
\newblock The gradient projection method for nonlinear programming. part ii.
  nonlinear constraints.
\newblock {\em Journal of the Society for Industrial and Applied Mathematics},
  9(4):514--532, 1961.

\bibitem{Haug:79}
E.~J. Haug and J.~S. Arora.
\newblock {\em Applied optimal design}.
\newblock J. Wiley \& Sons, New York, 1979.

\bibitem{Ben-Tal:01}
A.~Ben-Tal and A.~Nemirovski.
\newblock {\em Interior Point Polynomial Time Methods for Linear Programming,
  Conic Quadratic Programming, and Semidefinite Programming}, chapter~6, pages
  377--442.
\newblock Society for Industrial and Applied Mathematics, 2001.

\bibitem{Gunawardana:05}
A.~Gunawardana and W.~Byrne.
\newblock Convergence theorems for generalized alternating minimization
  procedures.
\newblock {\em Journal of Machine Learning Research}, 6:2049--2073, 2005.

\bibitem{Zangwill:69}
W.~I. Zangwill.
\newblock {\em Nonlinear Programming: A Unified Approach}.
\newblock Prentice Hall, Englewood Cliffs, N.J., 1969.

\bibitem{Srebro:04}
N.~Srebro.
\newblock {\em Learning with Matrix Factorizations}.
\newblock PhD thesis, Department of Electrical Engineering and Computer
  Science, Massachusetts Institute of Technology, 2004.

\bibitem{Hyvarinen:99a}
A.~Hyv{\"a}rinen and E.~Oja.
\newblock A fast fixed-point algorithm for independent component analysis.
\newblock {\em Neural Comput.}, 9(7):1483--1492, 1999.

\bibitem{LeCun:04}
Y.~LeCun, F.-J. Huang, and L.~Bottou.
\newblock Learning methods for generic object recognition with invariance to
  pose and lighting.
\newblock In {\em Proceedings of the IEEE Conference on Computer Vision and
  Pattern Recognition (CVPR)}. IEEE Press, 2004.

\bibitem{Vincent:10_short}
P.~Vincent, H.~Larochelle, et~al.
\newblock Stacked denoising autoencoders: {L}earning useful representations in
  a deep network with a local denoising criterion.
\newblock {\em JMLR}, 11:3371--3408, 2010.

\bibitem{Larochelle:07_short}
H.~Larochelle, D.~Erhan, et~al.
\newblock An empirical evaluation of deep architectures on problems with many
  factors of variation.
\newblock In {\em ICML}, pages 473--480, 2007.

\bibitem{Krizhevsky:09}
A.~Krizhevsky.
\newblock Learning multiple layers of features from tiny images.
\newblock Master's thesis, Deptartment of Computer Science, University of
  Toronto, 2009.

\bibitem{Verbist:15_short}
B.~Verbist, G.~Klambauer, et~al.
\newblock Using transcriptomics to guide lead optimization in drug discovery
  projects: Lessons learned from the \{QSTAR\} project.
\newblock {\em Drug Discovery Today}, 20(5):505 -- 513, 2015.

\bibitem{Hochreiter:06}
S.~Hochreiter, D.-A. Clevert, and K.~Obermayer.
\newblock A new summarization method for {Affymetrix} probe level data.
\newblock {\em Bioinformatics}, 22(8):943--949, 2006.

\bibitem{Ciresan:12cvpr}
D.~C. Ciresan, U.~Meier, and J.~Schmidhuber.
\newblock Multi-column deep neural networks for image classification.
\newblock In {\em IEEE Conference on Computer Vision and Pattern Recognition
  CVPR 2012}, 2012.
\newblock Long preprint arXiv:1202.2745v1 [cs.CV].

\bibitem{Krizhevsky:12}
A.~Krizhevsky, I.~Sutskever, and G.~E. Hinton.
\newblock {ImageNet} classification with deep convolutional neural networks.
\newblock In F.~Pereira, C.~J.~C. Burges, L.~Bottou, and K.~Q. Weinberger,
  editors, {\em Advances in Neural Information Processing Systems 25}, pages
  1097--1105. Curran Associates, Inc., 2012.

\bibitem{Lee:14}
C.-Y. {Lee}, S.~{Xie}, P.~{Gallagher}, Z.~{Zhang}, and Z.~{Tu}.
\newblock {Deeply-Supervised Nets}.
\newblock {\em ArXiv e-prints}, 2014.

\bibitem{Kullback:51}
S.~Kullback and R.~A. Leibler.
\newblock On information and sufficiency.
\newblock {\em Annals of Mathematical Statistics}, 22:79--86, 1951.

\bibitem{Dredze:08}
M.~Dredze, K.~Crammer, and F.~Pereira.
\newblock Confidence-weighted linear classification.
\newblock In {\em Proceedings of the 25th international conference on Machine
  learning (ICML08)}, volume~25, pages 264--271. ACM New York, 2008.

\bibitem{Dredze:12}
M.~Dredze, K.~Crammer, and F.~Pereira.
\newblock Confidence-weighted linear classification for text categorization.
\newblock {\em Journal of Machine Learning Research}, 13(1):1891--1926, 2012.

\bibitem{Patel:79}
R.~Patel and M.~Toda.
\newblock Trace inequalities involving hermitian matrices.
\newblock {\em Linear Algebra and its Applications}, 23:13--20, 1979.

\bibitem{Harva:05}
M.~Harva and A.~Kaban.
\newblock A variational bayesian method for rectified factor analysis.
\newblock In {\em Proc. Int. Joint Conf. on Neural Networks (IJCNN'05)}, pages
  185--190, 2005.

\bibitem{Graca:09}
J.~V. Graca, , K.~Ganchev, B.~Taskar, and F.~Pereira.
\newblock Posterior vs.\ parameter sparsity in latent variable models.
\newblock In Y.~Bengio, D.~Schuurmans, J.~D. Lafferty, C.~K.~I. Williams, and
  A.~Culotta, editors, {\em Advances in Neural Information Processing Systems},
  volume~22, pages 664--672, 2009.

\bibitem{Graca:07}
J.~V. Graca, , K.~Ganchev, and B.~Taskar.
\newblock Expectation maximization and posterior constraints.
\newblock In J.C. Platt, D.~Koller, Y.~Singer, and S.T. Roweis, editors, {\em
  Advances in Neural Information Processing Systems}, volume~20, pages
  569--576, 2007.

\bibitem{Neal:98}
R.~Neal and G.~E. Hinton.
\newblock A view of the {EM} algorithm that justifies incremental, sparse, and
  other variants.
\newblock In M.~I. Jordan, editor, {\em Learning in Graphical Models}, pages
  355--368. MIT Press, Cambridge, MA, 1998.

\bibitem{Feynman:72}
R.~P. Feynman.
\newblock {\em Statistical Mechanics}.
\newblock Benjamin, Reading, MA, 1972.

\bibitem{Friston:12}
K.~Friston.
\newblock A free energy principle for biological systems.
\newblock {\em Entropy}, 14:2100--2121, 2012.

\bibitem{Wu:83}
C.~F.~J. Wu.
\newblock On the convergence properties of the {EM} algorithm.
\newblock {\em Annals of Statistics}, 11(1):95--103, 1983.

\bibitem{Birgin:00}
E.~G. Birgin, J.~M. Mart\'{i}nez, and M.~Raydan.
\newblock Nonmonotone spectral projected gradient methods on convex sets.
\newblock {\em Siam Journal on Optimization}, 10(4):1196--1211, 2000.

\bibitem{Serafini:05}
T.~Serafini, G.~Zanghirati, and L.~Zanni.
\newblock Gradient projection methods for quadratic programs and applications
  in training support vector machines.
\newblock {\em Optimization Methods and Software}, 20(2-3):353--378, 2005.

\bibitem{Kim:06}
D.~Kim, S.~Sra, and I.~S. Dhillon.
\newblock A new projected quasi-{Newton} approach for the nonnegative least
  squares problem.
\newblock Technical Report TR-06-54, Department of Computer Sciences,
  University of Texas at Austin, 2006.

\bibitem{Goodfellow:13}
I.~J. Goodfellow, D.~Warde-Farley, M.~Mirza, A.~Courville, and Y.~Bengio.
\newblock Maxout networks.
\newblock {\em ArXiv e-prints}, 2013.

\bibitem{Min:13}
Min Lin, Qiang Chen, and Shuicheng Yan.
\newblock Network in network.
\newblock {\em CoRR}, abs/1312.4400, 2013.

\bibitem{Graham:14}
Benjamin Graham.
\newblock Fractional max-pooling.
\newblock {\em CoRR}, abs/1412.6071, 2014.

\end{thebibliography}

}

\end{document}